\documentclass[sigconf]{acmart}
\usepackage{amsmath}
\usepackage{amsthm}
\usepackage{amsmath,amsfonts,bm}

\def\eqref#1{equation~\ref{#1}}
\def\1{\bm{1}}

\DeclareMathAlphabet{\mathsfit}{\encodingdefault}{\sfdefault}{m}{sl}
\SetMathAlphabet{\mathsfit}{bold}{\encodingdefault}{\sfdefault}{bx}{n}

\DeclareMathOperator*{\argmin}{arg\,min}

\newcommand{\model}{SAT}
\newcommand{\mecha}{SA}
\newcommand{\mechanism}{Selective Attention}
\newcommand{\mechanismlowercase}{selective attention}
\usepackage{subfigure}
\usepackage{diagbox}
\usepackage{multirow}
\usepackage{url}
\usepackage{soul}
\usepackage{algorithm}
\usepackage{algpseudocode}

\setstcolor{blue}
\AtBeginDocument{%
  \providecommand\BibTeX{{%
    \normalfont B\kern-0.5em{\scshape i\kern-0.25em b}\kern-0.8em\TeX}}}

\setcopyright{acmcopyright}
\copyrightyear{2018}
\acmYear{2018}
\acmDOI{XXXXXXX.XXXXXXX}

\acmConference[Preprint]{Make sure to enter the correct
  conference title from your rights confirmation emai}{Under review}
\begin{document}

\title{Not All Neighbors Are Worth Attending to: Graph Selective \\ Attention Networks for Semi-supervised Learning}
%for Semi-supervised Learning}

\author{%
  Tiantian~He$^{1*}$,~~~Haicang~Zhou$^{2,3*}$,~~~Yew-Soon~Ong$^{1,2,3}$,~~~Gao~Cong$^{2,3}$\\
  $^1$Agency for Science, Technology and Research (A*STAR)\\
  $^2$Nanyang Technological University\\
  $^3$Singtel Cognitive and Artificial Intelligence Lab for Enterprises@NTU, Singapore\\
  {$\lbrace$He\_Tiantian@ihpc.,Ong\_Yew\_Soon@hq.$\rbrace$a-star.edu.sg} \\
  {$\lbrace$haicang001@e.,ASYSOng@,gaocong@$\rbrace$ntu.edu.sg}
}
\thanks{*Equal Contribution}
\renewcommand{\shortauthors}{Tiantian He et al.}

\begin{abstract}
Graph attention networks (GATs) are powerful tools for analyzing graph data from various real-world scenarios. To learn representations for downstream tasks, GATs generally attend to all neighbors of the central node when aggregating the features. In this paper, we show that a large portion of the neighbors are irrelevant to the central nodes in many real-world graphs, and can be excluded from neighbor aggregation. Taking the cue, we present Selective Attention (SA) and a series of novel attention mechanisms for graph neural networks (GNNs). SA leverages diverse forms of learnable node-node dissimilarity to acquire the scope of attention for each node, from which irrelevant neighbors are excluded. We further propose Graph selective attention networks (SATs) to learn representations from the highly correlated node features identified and investigated by different SA mechanisms. Lastly, theoretical analysis on the expressive power of the proposed SATs and a comprehensive empirical study of the SATs on challenging real-world datasets against state-of-the-art GNNs are presented to demonstrate the effectiveness of SATs.
\end{abstract}


%% To change. `Graph` should be incorporated.



\maketitle

\section{Introduction}
% Graph is universal for representing the complex relations among real-world data samples, such as social ties among social network users, and collaborations among academic researchers.
% Due to the significant roles played in various real applications, such as social network analysis \cite{he2021learning,DBLP:conf/nips/HamiltonYL17}, biological module detection \cite{DBLP:conf/iclr/VelickovicCCRLB18,DBLP:conf/iclr/KipfW17,DBLP:conf/iclr/0001ZWZ20}, and collaboration analysis \cite{DBLP:conf/iclr/VelickovicFHLBH19,he2021learning,bianchi2021graph}, graph learning has drawn much interest in the recent.

Graph neural networks (GNNs) \cite{DBLP:conf/iclr/KipfW17, DBLP:conf/cvpr/MontiBMRSB17, DBLP:conf/nips/HamiltonYL17, DBLP:conf/iclr/VelickovicCCRLB18, DBLP:conf/icml/XuLTSKJ18, DBLP:conf/icml/ChenZS18, DBLP:conf/iclr/KlicperaBG19, DBLP:conf/iclr/XuHLJ19, DBLP:conf/icml/WuSZFYW19, abu2019mixhop, DBLP:conf/icml/YouYL19, DBLP:conf/nips/KlicperaWG19, DBLP:conf/nips/YouYL20, DBLP:conf/kdd/ZhangHZZ20} have achieved great success in semi-supervised learning tasks in graph-structured data.
Among various types of GNNs, graph attention networks have gained popularity with multifarious real-world applications, especially arising from social and collaboration graphs \cite{he2021learning,DBLP:conf/iclr/VelickovicCCRLB18,DBLP:conf/iclr/0001ZWZ20, DBLP:conf/kdd/GaoJ19, DBLP:conf/www/WangJSWYCY19, ying2021transformers, brody2022how, min2022transformer}.
In each graph attention layer \citep{DBLP:conf/iclr/VelickovicCCRLB18}, the node representation is generally learned following a two-step procedure.
Attention scores (attention coefficients) between each node and all its neighbors are firstly computed by some attention mechanism.
The node representation for downstream tasks is then computed as a weighted aggregation of all neighbor features.

%Although most GATs typically include all neighbors in the scope of attention,
Existing GNNs typically include all neighbors in the scope (i.e., the receptive field) for feature aggregation \cite{zeng2021decoupling}.
%Although existing work on GNNs typically include all neighbors in feature aggregation,
% Although attending to all neighbors in feature aggregation is widely accepted,
%Although including all neighbors in the scope of attention is widely accepted, 
However, it might not be always best to do so in attention-based GNNs (Fig. \ref{fig:motivation}).
Recent studies \citep{Rong2020DropEdge, DBLP:conf/icml/ZhengZCSNYC020} have shown that 
incorporating all neighbors in the scope
%considering all neighbors 
for feature aggregation can possibly lead to deterioration in the predictive performances of GNNs on various graph learning tasks.
%the learning capability of GNNs.
Notably, our study also indicates that most neighbors in widely used social and collaboration graph datasets, such as Cora and Cite \citep{DBLP:journals/aim/SenNBGGE08}, are found to be far apart
% a simple distance measure and share at most one common neighbor
%simple distance (or similarity) measures 
(see Appendix \ref{conflict-example} for more details).
This reveals that most neighbors in real-world graphs are highly dissimilar to the central node and are hence likely irrelevant for feature aggregation.
%This reveals that most neighbors in these graphs are highly irrelevant.
% and can be excluded from the neighbor aggregation.
Our idea is that adapting the scope of attention appropriately (i.e., the receptive field of graph attention for feature aggregation)
%appropriately modulating the scope of attention \newtext{(As defined in \citep{zeng2021decoupling}, scope means the receptive field. Thus, the scope of attention means the receptive field of graph attention.)}
can enable attention-based GNNs to learn better representations by attending more to highly relevant neighbors while ignoring the irrelevant ones.
We conjecture that with better representations, the performance of GNNs on semi-supervised learning tasks would improve.

% Aiming to fine-tune the scope of neighbors, some strategies, e.g., top-$k$ ranking \citep{DBLP:conf/kdd/GaoJ19} have been proposed to constrain attention-based GNNs to learn representations by aggregating the features from a fixed subset of neighbors.
% Such arbitrary approaches cannot identify that the neighbor scope for neighbor aggregation is possibly different from one node to another.
% Consequently, inappropriate neighbors might be attended to in the feature aggregation, probably leading to a deteriorated performance on semi-supervised learning tasks.
% Although it is of great importance to develop more effective attention-based GNNs, how to adaptively determine the scope of neighbors is still nascent to date.

\begin{figure*}[ht]
    \centering
    \subfigure[Neighbor aggregation in GAT]{
        \includegraphics[width=0.23\textwidth]{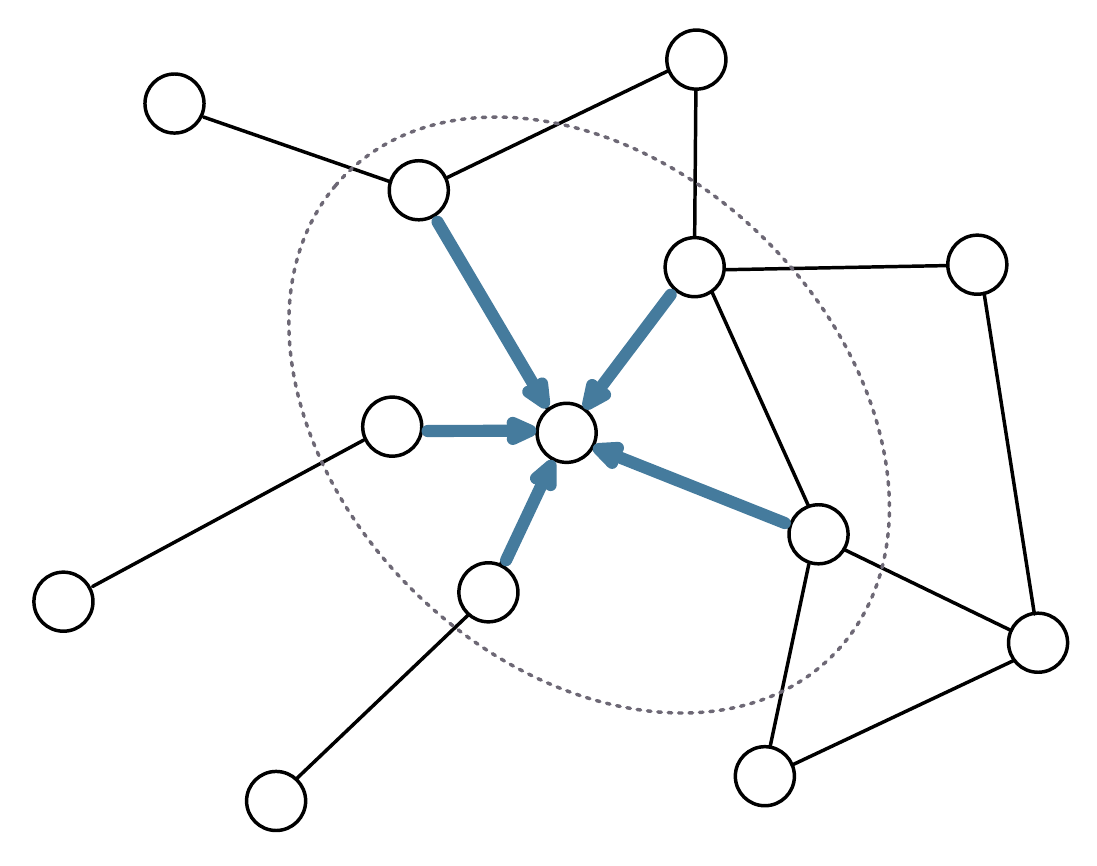}
    }
    \subfigure[Neighbor aggregation in SAT]{
        \includegraphics[width=0.23\textwidth]{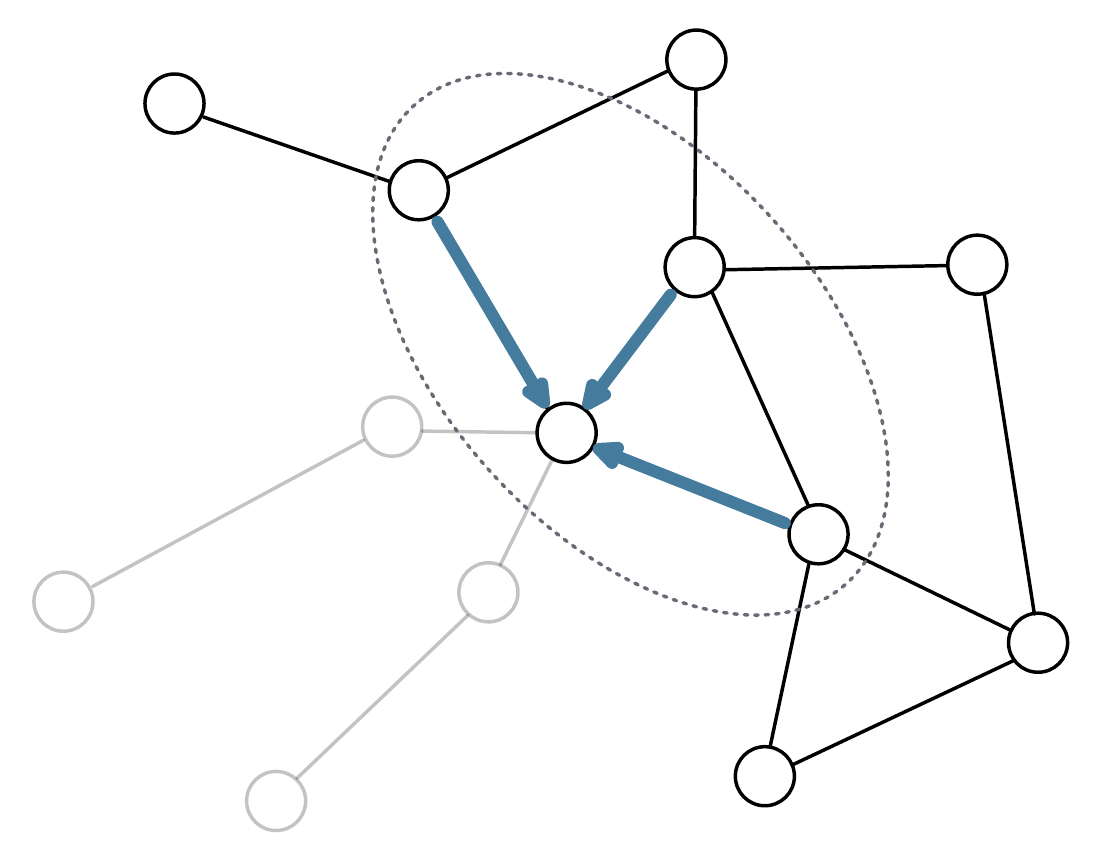}
    }
    \subfigure[A case study of attention scores]{
        \includegraphics[width=0.27\textwidth]{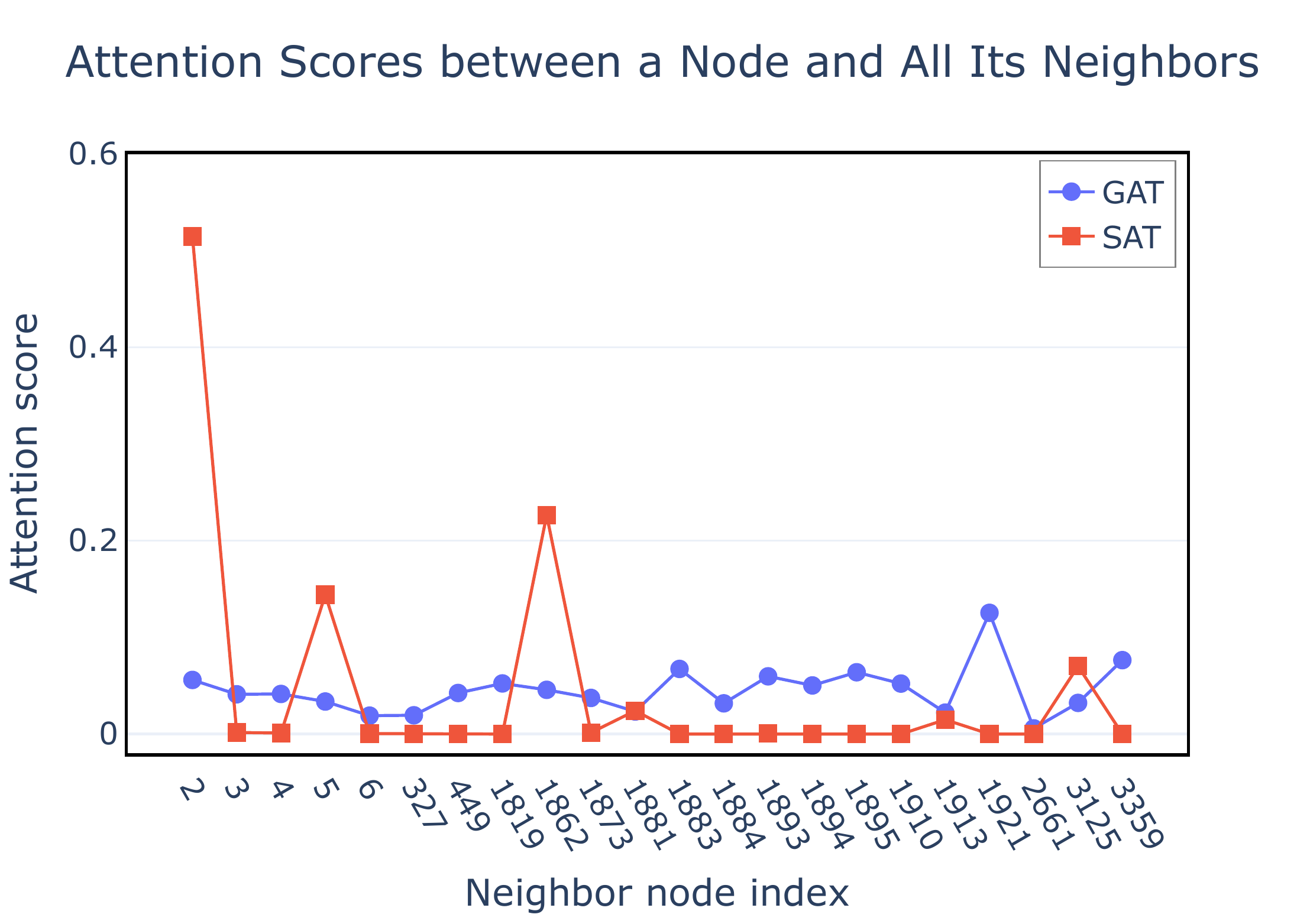}
    }
    \caption{Schematic illustration of the difference between SAT and GAT. 
    SAT (subfigure (b)) can learn the scope of attention for each node while GAT cannot (subfigure (a)).
    %SAT (subfigure (b)) can ignore irrelevant neighbors while GAT cannot (subfigure (a)).
    In subfigure (c), we exemplify the attention scores from a node and all its neighbors in a real-world dataset. Most of the learned attention scores learned by \model~are very close to 0, which means the corresponding neighbors are excluded from the feature aggregation. (Please refer to Section \ref{sec:attention_analysis} for more details.)}
    \label{fig:motivation}
\end{figure*}

% In contrast, human beings are born with the capacity of adaptively perceiving which subset of objects are worthy attending to.
% Previous studies have shown 
Our present idea is inspired and motivated by studies \cite{chi1975span,johnston1986selective,fu2020can} in cognitive science. 
Humans are known to be capable of determining the number of stimuli to respond to well, such as spoken words and presented images within their apprehension span \citep{chi1975span}.
The high quality of cognition is maintained by paying \mechanism~\cite{johnston1986selective,fu2020can} to the few stimuli identified in the apprehension span as most relevant to the cognitive goal, and ignoring the irrelevant ones.
Although the idea of apprehension span has been tentatively adapted to solve some learning-based tasks in computer vision \cite{huang2021leveraging} and natural language processing \cite{DBLP:conf/naacl/MarufMH19, han2022sancl}, how to adapt the scope of attention in attention-based GNNs for  representation learning of graph data remains under-explored to date.
%\newtext{Meanwhile, we notice that selective attention has been successfully adapted for computer vision (CV) \cite{huang2021leveraging} and natural language processing (NLP) \cite{DBLP:conf/naacl/MarufMH19, han2022sancl}, but it has not been explored on graph data. This also motivates us to design its adaptation for semi-supervised learning on graphs.}

%Taking an analogy between the apprehension span in cognition and the scope of attention for the feature aggregation, we proposed here a study of such capabilities in GNN.  
%, which can identify and exclude the neighbors highly irrelevant to feature aggregation.
In this paper, we present an investigation on Graph selective attention networks that takes an analogy between the scope of attention for feature aggregation in GNNs and the apprehension span in human cognition.
%with this analogy of human cognitive capacity  and how it may enhance the performance of attention-based GNNs in real-world analytical tasks.
A series of \mechanism~(SA) mechanisms for graph neural networks is also proposed, and
% , which generalizes a class of novel attention mechanisms.
diverse forms of node-node dissimilarities for learning the node-wise scope of attention are investigated.
Neighbors that are dissimilar to the central node based on the scope of attention are deemed as highly irrelevant, and hence excluded in the feature aggregation process.
%SA leverages diverse forms of node-node dissimilarity to learn the node-wise scope of attention, which excludes highly irrelevant neighbors from the feature aggregation.
% \newtext{Borrowing the descriptions in cognitive science, relevant (irrelevant) neighbors for a node are those relevant (informative) to the feature aggregation .}
Our proposed \mecha~mechanisms also return attention coefficients that are differentiably pruned by the learned scope of attention, allowing highly irrelevant nodes to be ignored in feature aggregation.
% To effectively perform the analytical tasks in social and collaboration graphs, 
Subsequently, we propose and construct Graph selective attention networks (\model s) capable of learning effective representations that favor highly relevant nodes while ignoring the irrelevant ones (identified using the SA mechanisms).
The main contributions of the paper are summarized as follows:
\begin{itemize}
    \item We propose Selective Attention (\mecha), which comprises a class of novel attention mechanisms for GNNs.
    SA leverages node-node dissimilarity to learn the node-wise scope of attention, which can exclude irrelevant neighbors from the feature aggregation.
    % SA thus endows GNNs with capabilities to learn representations that concentrate and aggregate features from highly relevant nodes while ignoring those \warningtext{deemed} as irrelevant.
    SA thus endows GNNs with capabilities to learn representations that concentrate and aggregate features from highly relevant neighbors while ignoring irrelevant neighbors.
    \item The expressive power of the proposed \mecha~layers is analyzed. The theoretical analysis verifies that the expressive power of the proposed \mecha~layers can reach the upper bound of all message-passing GNNs, indicating that \mecha~layers are more powerful than those layers used in existing attention-based GNNs.
    \item We further use the proposed \mecha~layers to construct Graph \mechanismlowercase~networks (\model s) for various downstream learning tasks arising from social and collaboration graphs.
    SATs are comprehensively tested on several well-established benchmarking datasets and compared to a number of state-of-the-art GNNs for the task of semi-supervised node classification and clustering.
    The results demonstrate that \model s can outperform other strong baselines.
\end{itemize}

\section{Related work}\label{related}

\paragraph{Attention in GNNs} Graph attentions \citep{DBLP:conf/iclr/VelickovicCCRLB18,gulcehre2018hyperbolic,DBLP:conf/kdd/GaoJ19,kreuzer2021rethinking,wang2021multi,he2021learning,wu2021representing,brody2022how} take advantage of effective attention mechanisms \citep{xu2015show,cheng2016long} to dynamically learn the normalized node-node correlations regarding node features (attention scores/coefficients), which determine the neighbor importance for the subsequent feature aggregation. 
To further enhance the learning performance, there is a trend to inject graph structures into the computation of attention coefficients \citep{DBLP:conf/iclr/0001ZWZ20,dwivedi2020generalization, hussain2021edge, ying2021transformers,mialon2021graphit,min2022transformer}.
Compared with existing attention-based GNNs, our method is fundamentally different.
Our method can ignore irrelevant neighbors by considering diverse forms of dissimilarity between node pairs.
Dissimilarity pertaining to graph structure is one possible choice of our approach.
Besides, our experiments, which analyze the distribution of attention scores, show that injecting the graph structure into the computation of graph attention without considering the node-node dissimilarity cannot exclude those irrelevant neighbors from the feature aggregation.
%does not take effect by ignoring unimportant neighbors.

%\paragraph{Modulating the scope of neighborhood in GNNs}
\paragraph{Adjusting the scope of GNNs}
In general, the scope (the receptive field) for each node in one GNN layer \cite{zeng2021decoupling} is its first-order neighbors (including itself), and after $k$ layers, it can capture the information from $k$-hop neighbors \citep{DBLP:conf/iclr/KipfW17, DBLP:conf/iclr/VelickovicCCRLB18, DBLP:conf/icml/WuSZFYW19, brody2022how}. 
There is a rich literature on carefully designing the scope for GNNs.
This includes approaches (e.g., \citep{DBLP:conf/nips/HamiltonYL17, DBLP:conf/icml/ChenZS18, Rong2020DropEdge, DBLP:conf/iclr/ZengZSKP20, DBLP:conf/nips/ZhangWGS21, DBLP:conf/icml/LiWLCX22}) that apply sampling strategies to improve the scalability of GNNs. Different from these methods, our work is not based on any sampling strategy. 
There are also some approaches \citep{DBLP:conf/kdd/GaoJ19, DBLP:conf/icml/ZhengZCSNYC020} adopting the top-$k$ strategy to select the $k$ most relevant neighbors for aggregation.  Unlike the top-$k$ strategy, our method can learn how many neighbors to ignore for each node adaptively.
% Unlike these methods of sampling neighbors with some predefined strategy, our approach can adaptively modulate the scope by ignoring irrelevant neighbors according to the learning objective through end-to-end training.
In parallel, there have been several studies (e.g., \citep{DBLP:conf/icml/XuLTSKJ18, DBLP:conf/aaai/LiuCLZLSQ19, DBLP:conf/iclr/0001ZWZ20, DBLP:conf/iclr/HouZCMMCY20, DBLP:conf/www/MaWCS21, DBLP:conf/nips/ZhaoDDKT21}) on designing adaptive receptive fields, either for each node \cite{DBLP:conf/www/MaWCS21} or for different parts of models \cite{DBLP:conf/nips/ZhaoDDKT21}. To the best of our knowledge, none of these methods generate receptive fields by identifying and ignoring irrelevant neighbors.
\section{Graph selective attention networks}\label{method}
In this section, we introduce the proposed Graph \mechanismlowercase~networks (\model s).
%In this section, we elaborate the proposed Graph \mechanismlowercase~networks (\model s).
%The learnable node-node dissimilarity considered by the proposed \mechanism~is firstly introduced.
The \mechanism~layers adopting different \mechanism~mechanisms are firstly elaborated.
How to use the proposed \mechanism~layers to construct \model s is then introduced.
The computational complexity of \model s is finally analyzed.

\subsection{Notations}
% In this paper, we denote a graph as $G = \lbrace V, E \rbrace$, 
Let $G = \lbrace V, E \rbrace$ denote a graph,
where $V$ and $E$ represent the node and edge set. In $G$, there are $N$ nodes, $|E|$ edges, and $C$ classes ($C\ll N$) which the nodes possibly belong to.
The adjacency matrix of $G$ and the input node feature matrix are denoted as $\mathbf A \in \lbrace 0, 1\rbrace^{N \times N}$ and $\mathbf X \in \mathbb R^{N \times D}$.
Node $i$ and its one-hop neighbors are denoted as $\mathcal N_i$.
$\mathbf W^l$ and $\lbrace \mathbf h^l_i \rbrace_{i = 1, ... N}$ denote the learnable weight matrix and output representation of node $i$ at $l$-th layer of \model s, respectively, and $\mathbf h^0$ is set to be the input feature $\mathbf{X}$.
%giving the multiplicity of each $s$ in $S_i$.

\subsection{Selective Attention layers}
In this subsection, we present the \mechanism~(\mecha) layer, which is the core for building the Graph \mechanismlowercase~networks (\model s).
The present graph attention mechanisms attempt to compute attention coefficients to all neighbors of each node.
However, our preliminary study has shown that most neighbors in widely used graph datasets are quite irrelevant when evaluated by simple criteria (Appendix \ref{conflict-example}).
Thus, these irrelevant neighbors can be excluded from the neighbor aggregation in an appropriate way.
% How to utilize diverse forms of node-node dissimilarity to adaptively modulate the scope of neighborhood involved in the feature aggregation is still nascent.
Meanwhile, humans flexibly adjust the number of stimuli to respond in their apprehension span \citep{chi1975span}.
% The quality of cognition can therefore be maintained by selectively attending to few stimuli that are most relevant to the cognitive goal, but ignoring those irrelevant ones.
Specifically, in order to maintain the quality of cognition, human selectively attends to a few stimuli that are most relevant to the cognitive goal but ignores those irrelevant ones.
Inspired by these, in this paper, we propose \mechanism, which utilizes diverse forms of node-node dissimilarity to learn
the scope of attention for each node.
% the node-wise apprehension span.
Irrelevant neighbors can be ignored in the feature aggregation stage.
\mecha~therefore endows \model s with the capability of learning representations concentrating on highly relevant neighbors.
% The predictive performance on various real-world applications can be enhanced.

Given a set of node features $\lbrace \mathbf h^l_i\rbrace_{i = 1, ... N}$, $\mathbf h^l_i \in R^{D^l}$, the \mechanism~layer maps them to $D^{l+1}$ dimensional vectors $\lbrace \mathbf h^{l+1}_i\rbrace_{i = 1, ... N}$.
The mapping requires weighted aggregation, with weights computed from both correlations of node features and diverse forms of node-node dissimilarity.
%according to both correlations of node features and the aforementioned \textit{Multi-dissimilarity}.
The feature correlation between two connected nodes is firstly obtained.
%The contextual correlation between two connected nodes, say $v_i$ and $v_j$ is firstly obtained.
To generate the correlations of node features between connected nodes, we adopt the method in \citep{DBLP:conf/iclr/VelickovicCCRLB18}:
\begin{equation}\label{f-att}
	f_{ij}=\frac{\exp(\text{LeakyReLU}(\mathbf {\vec{a}}^T(\mathbf {W}^l\mathbf h^l_i\parallel \mathbf {W}^l\mathbf h^l_j )))}{\sum_{k\in \mathcal{N}_i} \exp(\text{LeakyReLU}(\mathbf {\vec{a}}^T(\mathbf {W}^l\mathbf h^l_i\parallel \mathbf {W}^l\mathbf h^l_k )))},
\end{equation}
where $\mathbf {\vec{a}}\in \mathbb{R}^{2D^{l+1}}$ is a vector of attention parameters, $\parallel$ stands for the concatenation function for two vectors, and $\mathbf W^l$ is a $D^{l+1}\times D^l$ parameter matrix for feature mapping.
Given $f_{ij}$ computed from Eq. (\ref{f-att}), the proposed \mecha~layer can capture the feature correlations between connected nodes.
%Given Eq. (\ref{f-att}), the proposed SA layer captures the feature correlations between connected nodes (first-order neighbors) by computing the similarities w.r.t. node features mapped to next layer.

As mentioned above, existing graph attention mechanisms \cite{DBLP:conf/iclr/VelickovicCCRLB18,he2021learning,brody2022how,DBLP:conf/iclr/0001ZWZ20} do not consider adjusting the scope of attention,
%do not consider diverse forms of node-node dissimilarity that can be readily available, 
resulting in somewhat less informative attention scores.
% To address this, the proposed \mechanism~further considers utilizing diverse forms of node-node dissimilarity, which will be introduced later in this subsection, to learn the scope of attention for each node in the graph,
% which can quantify the extent that a neighbor can be ignored.
To address this, the proposed \mechanism~further considers utilizing diverse forms of node-node dissimilarity to learn the scope of attention for each node in the graph. The node-node dissimilarity, which will be introduced later in this subsection, can quantify the extent that a neighbor can be ignored.
% Such apprehension span can quantify the extent that a neighbor can be ignored.
Leveraging the aforementioned feature correlations (Eq. (\ref{f-att})) and node-node dissimilarity, we propose two different strategies derived from the \mecha~concept to compute attention coefficients.

The first strategy is named as \textit{Contractive apprehension span}, which is able to exponentially contract the normalized feature correlations (Eq.(\ref{f-att})).
The attention score obtained via the \textit{Contractive apprehension span} is defined as follows:
\begin{equation}\label{c-strategy}
\begin{aligned}
&\alpha_{ij} = \frac{f_{ij}\cdot \exp{(-\beta \mathbf S_{ij})}}{\sum_{k \in \mathcal N_i} f_{ik}\cdot \exp{(-\beta \mathbf S_{ik})}},\\
\end{aligned}
\end{equation}
where $\mathbf S_{ij}$ represents the node-node dissimilarity, and $\beta \in (0, 1]$ is a positive hyperparameter used to control the significance of $\mathbf S_{ij}$.
As shown in Eq. (\ref{c-strategy}), the feature-based attention scores are reduced as $\mathbf S_{ij}$ goes higher.
The scope of attention for a given node is contracted as the attention coefficients obtained by Eq. (\ref{c-strategy}) might be very close to zero.
Those irrelevant neighbors can be less attended to in the feature aggregation according to $\mathbf S_{ij}$,
and the \mecha~layer can therefore pay higher attention to those neighbors that are more relevant to the central node. 

The second strategy is called \textit{Subtractive apprehension span}.
As its name implies, this strategy attempts to adjust the scope of attention
%ignore those irrelevant nodes 
by directly subtracting the effect brought by the dissimilarity between node pairs.
The attention coefficient computed by the \textit{Subtractive apprehension span} is defined as follows:
\begin{equation}\label{s-strategy}
\begin{aligned}
    &\quad \alpha_{ij} = \frac{f_{ij} (1 - \beta \mathbf T_{ij})}{\sum_{k \in \mathcal N_i} f_{ik} (1 - \beta \mathbf T_{ik})},\mathbf T_{ij} = \frac{\exp{(\mathbf S_{ij})}}{\sum_{k \in \mathcal N_i}\exp{(\mathbf S_{ik})}},\\
\end{aligned}
\end{equation}
where $\beta \in (0, 1]$ is a positive hyperparameter controlling the effect of $\mathbf T_{ij}$. 
Compared with the \textit{Contractive apprehension span}, the above \textit{Subtractive apprehension span} enables the \mecha~layers to 
adjust the scope of the attention for feature aggregation
%learn to ignore 
in a more radical way, as some attention coefficients between connected nodes can be reduced to zero when $\mathbf T_{ij}$ is sufficiently high.
\textit{Subtractive apprehension span} allows \mecha~layers to concentrate only on those relevant nodes when aggregating neighboring features for message passing in the GNN.

With the \mechanism~coefficients, the \mecha~layer now can aggregate features associated with each node and its neighbors to generate layer outputs, which will be either propagated to the higher layer, or be used as the representations for downstream tasks.
The described aggregation phase can be formulated as follows:
\begin{equation}\label{att-aggregation}
	\mathbf h^{l+1}_i = (\alpha_{ii}+\epsilon\cdot\frac{1}{\vert \mathcal N_i \vert})\mathbf {W}^l\mathbf h^l_i+ \sum_{j \in \mathcal N_i, j \ne i} \alpha_{ij} \mathbf {W}^l\mathbf h^l_j,
\end{equation}
where $\epsilon \in (0, 1)$ is a learnable parameter to improve the expressive power of the \mecha~layer.

\paragraph{Node-node dissimilarity for \mechanism}

%\subsection{Node-node dissimilarity for \mechanism}\label{conflicting-factor}
The proposed \mechanism~(\mecha) allows diverse node-node dissimilarity to be used to compute attention coefficients.
%This synergy of multiple sources of dissimilarity between connected nodes can bring about new positive joint effect on the computation of \mecha~coefficients.
In this paper, we use the following method to compute dissimilarity ($\mathbf S_{ij}$) between each pair of connected nodes in the graph:
%In this paper, we propose \textit{Multi-dissimilarity} to better quantify such joint effect between each pair of connected nodes in the graph ($\mathbf S_{ij}$):
\begin{equation}\label{f-conflict}
    \begin{aligned}
        &\mathbf{S}_{ij} = \sum_k r_k\cdot \Psi_k(\mathbf c^k_i, \mathbf c^k_j), \text{ subject to}\text{ }\sum_k r_k = 1,
    \end{aligned}
\end{equation}
where $\Psi(\cdot, \cdot)$ is a distance metric, $k$ is the type index of dissimilarity, $\mathbf c^k_i$ is a vector characterizing node $i$ from type $k$, and $r_k$ is a learnable parameter used to balance the relative significance of different types of node properties, i.e., $\mathbf c^k_i$.
In this paper, we focus on two types of dissimilarity, which are node features (inputs) at each layer of the GNN and the structure of the nodes in the graph. By replacing $\Psi(\cdot, \cdot)$ as Euclidean distance function, Eq. (\ref{f-conflict}) can be rewritten as:
\begin{equation}\label{bi-conflict}
\mathbf{S}_{ij} \!=\! r_f\!\cdot\! \Vert \mathbf{Wh}_i \!-\! \mathbf{Wh}_j\Vert^2 \!+\! r_p\! \cdot\! \Vert \mathbf{p}_i \!-\! \mathbf{p}_j\Vert^2,\text{subject to}\text{ }r_f \!+\! r_p \!=\! 1,
\end{equation}
where $\mathbf{Wh}_i$ is the features of node $i$ in a GNN layer, and $\mathbf{p}_i$ is defined as
% a 1-by-$C$ latent space %the 1-by-$C$ latent positions 
% of node $i$ learned from graph structure.
a 1-by-$C$ vector of node $i$ in the latent space learned from graph structure.
%that describe the learnable structural properties of node $i$ in the graph.
As $\mathbf{p}_i$ is assumed to be learnable, 
%here we propose a generic and simple method to enable \mecha~to effectively acquire diverse characteristics hidden in the graph:
%\begin{equation}\label{Lp}
%    \begin{aligned}
%        L(\mathbf{P}) = \argmin_{\phi(\mathbf{P})} \Phi(\phi(\mathbf{P}), \mathbf{Y}),
%    \end{aligned}
%\end{equation}
%where $\mathbf{Y}$ is a prior matrix containing the properties from which $\mathbf{p}_i$ attempts to learn, $\mathbf {P}$ is an $N$-by-$C$ matrix containing all $\mathbf{p}_i$s, $\phi(\cdot)$ is a proper function that transforms $\mathbf {P}$ to the dimension of $\mathbf{Y}$, and $\Phi(\cdot)$ is an appropriate operator between $\phi(\mathbf{P})$ and $\mathbf {Y}$.
%Obviously,
different learning approaches enable $\mathbf{p}_i$ to capture different properties in the graph structure.
In this paper, we mainly consider the learnable properties hidden in the graph adjacency.
%Taking $\mathbf{A}$ as the prior matrix, $\phi(\cdot)$ as the operator of matrix outer product, and $\Phi(\cdot)$ as Euclidean distance function, hence we have:
Thus, we have:
\begin{equation}\label{lp-mf}
    \begin{aligned}
        L(\mathbf{P}) = \argmin_{\mathbf{P}\mathbf{P}^T_{ij}} \sum_{i,j}(\mathbf{A}_{ij} - \lbrack\mathbf{P}\mathbf{P}^T\rbrack_{ij})^2,
    \end{aligned}
\end{equation}
%\newtext{Haicang: Better to use $(PP^{T})_{ij}$.}
where $\mathbf {P}$ is an $N$-by-$C$ matrix containing all $\mathbf{p}_i$s.
As shown in Eq. (\ref{lp-mf}), $\mathbf P$ is learned by matrix factorization, which assumes that the edge adjacency matrix $\mathbf{A}$ can be reconstructed by $\mathbf{P}\mathbf{P}^T$.
Through Eq. (\ref{lp-mf}), 
similar nodes in the graph will induce a low distance in Eq. (\ref{bi-conflict}), and vice versa.

\subsection{The architecture of Graph selective attention networks}

Now, we can build Graph \mechanismlowercase~networks (\model s) with the proposed \mecha~layers.
``To stabilize the learning process'', we follow \citep{DBLP:conf/nips/VaswaniSPUJGKP17,DBLP:conf/iclr/VelickovicCCRLB18,he2021learning} to use the multi-head attention strategy when constructing \model s.
\model s either concatenate the node representations generated by multiple attention heads as the input of the next layers, or compute the mean of node features obtained by multiple attention heads as the output representations.
As we in this paper additionally define learnable latent spaces for each node in the graph, the overall loss function for \model s is slightly different from classical graph attention networks.
%the overall loss function for \model s is slightly different from empirical GNNs.
It is conceptually written as follows:
\begin{equation}\label{loss-function}
L = L_{task} + L(\mathbf P),
\end{equation}
where $L(\mathbf P)$ is the MF method shown in Eq. (\ref{lp-mf}), and $L_{task}$ is the task-specific loss.

\subsection{Computational complexity of \mechanism~layers}
As each layer in \model s additionally requires 
node-node dissimilarity
%\textit{Multi-dissimilarity} 
to compute Selective Attention coefficients, the computational complexity is slightly different from classical attention-based GNNs.
Let $D^l$ (or $D^{l + 1}$) denote the dimension of the input (or output) vector of the $l$-th layer. The complexity of feature aggregation in the $l$-th layer is $O(ND^lD^{l+1}+(\vert E\vert+e)D^{l+1})$ for one attention head, where $e$ represents the average degree of each node and $N$ is the number of nodes in the input graph. 
This is the same as that of classical graph attention networks \cite{DBLP:conf/iclr/VelickovicCCRLB18}.
% For the computation of the feature aggregation in one attention head, its complexity is $O(ND^lD^{l+1}+(\vert E\vert+e)D^{l+1})$, where $e$ represents the average degree of each node in the graph and such complexity is the same as that of classical graph attention networks.
When there are $K$ attention heads, the complexity is $O(KND^lD^{l+1}+K(\vert E\vert+e)D^{l+1})$.
Additional computation in \model s is demanded as \model s have to capture node-node dissimilarity
%\textit{Multi-dissimilarity} 
for the computation of \mechanism~ coefficients.
The complexity for learning the node-node dissimilarity
%\textit{Multi-dissimilarity}
in each attention head is $O(e(D^{l+1} + 2C))$, where $C$ represents the dimension of each $\mathbf p$ in Eq. (\ref{bi-conflict}).
\section{Theoretical analysis}  \label{sec:analysis}
In this section, the expressive power of the proposed \model s is analyzed.
The expressive power evaluates whether a GNN can discriminate distinct (sub)structures wherein nodes have different features.
Thus, it can theoretically reveal whether a GNN is sufficiently powerful for various downstream tasks.  
%Thus, it can theoretically reveal the learning capabilities of GNNs.
Recent studies \citep{DBLP:conf/iclr/XuHLJ19, zhang2020improving,corso2020principal} have shown that the feature aggregations in all message-passing GNNs are similar to the injective 1-dimensional Weisfeiler-Lehman test (1-WL test) \citep{weisfeiler1968reduction}.
Theoretically, the expressive power of all message-passing GNNs is as most as the 1-WL test \citep{DBLP:conf/iclr/XuHLJ19}.

As the proposed \model s belong to message-passing GNNs, their expressive power can be verified by showing the injectivity of the feature aggregation in \model s.
To do so, we first prove that either \textit {Contractive} (Eq. (\ref{c-strategy})), or \textit {Subtractive apprehension span} (Eq. (\ref{s-strategy})) is still unable to distinguish some structures satisfying some conditions, without the improving term shown in Eq. (\ref{att-aggregation}), i.e., $\epsilon\cdot\frac{1}{\vert \mathcal N_i \vert}\mathbf {W}^l\mathbf h^l_i$.
Then, we prove that all the \mecha~layers can discriminate all different structures when aggregating neighboring features utilizing either of the two proposed strategies integrated with $\epsilon\cdot\frac{1}{\vert \mathcal N_i \vert}\mathbf {W}^l\mathbf h^l_i$ (Eq. (\ref{att-aggregation})).
Before the proof, we follow \citep{DBLP:conf/iclr/XuHLJ19,he2021learning} to give the notations for multisets.
For the nodes in $\mathcal N_i$, their feature vectors form a multiset $X_i = (\mathsf{M}_i, \mu_i)$, where $\mathsf{M}_i = \lbrace s_1, ... s_n\rbrace$ is the underlying set of $X_i$ containing its \emph{distinct elements}, and $\mu_i : \mathsf{M}_i \rightarrow \mathbb N^\star$ gives the multiplicity of each distinct element in $\mathsf{M}_i$.

For the neighborhood aggregating function that is solely based on the \textit {Contractive apprehension span} (Eq. (\ref{c-strategy})), the following theorem shows that it still cannot distinguish some structures.
%we have the following theorem showing that it still cannot distinguish some structures.
%\begin{theorem}\label{theorem-c}
%Assume $\mathcal X$ represents the countable feature space, and the aggregation function utilizing the attention coefficients obtained by Eq. (\ref{c-strategy}) is denoted as $h(c, X) = \sum_{x\in X} \alpha_{cx} g(x)$,
%where $c$ represents the central node feature, $X \in \mathcal X$ represents a multiset comprising the features from nodes in $\mathcal N_i$,
%$\alpha_{cx}$ is the attention score between $g(c)$ and $g(x)$, and $g(\cdot)$ is a function for mapping input feature $X$.
%For all $g$ and the Contractive apprehension span in Eq. (\ref{c-strategy}), $h(c_1, X_1) = h(c_2, X_2)$ if and only if $c_1 = c_2$, $X_1 = ( \mathsf{M}, \mu_1 )$, $X_2 = ( \mathsf{M}, \mu_2 )$, and $q\cdot\sum_{y=x, y\in X_1} \psi(-\beta\mathbf S_{c_1y}) = \sum_{y=x, y\in X_2} \psi(-\beta\mathbf S_{c_2y})$, for $q > 0$ and $x \in \mathsf{M}$, where $\psi(\cdot)$ is a function for mapping values to $\mathbb R^+$.
%\end{theorem}

\begin{theorem}\label{theorem-c}
Let $c_i$ denote the feature vector of node $i$, and $X_i = \{\mathsf{M}_i, \mu_i\} \in \mathcal X$ denote a multiset comprising the features from nodes in $\mathcal N_i$, where $\mathcal X$ represents the countable feature space. The aggregation function using the attention scores computed by Eq. (\ref{c-strategy})
%in Eq. (\ref{c-strategy}) 
is denoted as $h(c_i, X_i) = \sum_{x\in X_i} \alpha_{c_i x} g(x)$, where $g(\cdot)$ is a function defined on $X_i$ and $\alpha_{c_i x}$ is the attention score between $g(c_i)$ and $g(x)$. For all $g$, any two nodes $1$ and $2$ and the Contractive apprehension span in Eq. (\ref{c-strategy}), $h(c_1, X_1) = h(c_2, X_2)$ holds if and only if $c_1 = c_2, \mathsf{M}_1 = \mathsf{M}_2 = \mathsf{M}$, and $q\cdot\sum_{y=x, y\in X_1} \psi(-\beta\mathbf S_{c_1y}) = \sum_{y=x, y\in X_2} \psi(-\beta\mathbf S_{c_2y})$, for $q > 0$ and $x \in \mathsf{M}$, where $\psi(\cdot)$ is a function for mapping values to $\mathbb R^+$.
\end{theorem}

\begin{proof}
Due to space limitations, here we briefly illustrate the method for completing the proof of Theorem \ref{theorem-c}.
The full proof can be checked in supplementary materials.
The proof of Theorem \ref{theorem-c} can be divided into two parts, i.e., the proof of the sufficiency and necessity of the iff conditions \cite{he2021learning,DBLP:conf/iclr/XuHLJ19,zhang2020improving}.
Given $c_1=c_2$, $\mathsf{M}_1=\mathsf{M}_2$, and $q \cdot \sum_{y=x, y\in X_1} \psi(-\beta\mathbf S_{c_1y}) = \sum_{y=x, y\in X_2} \psi(-\beta\mathbf S_{c_2y})$, $h(c_1, X_1) = h(c_2, X_2)$ can be easily verified.
Thus, the sufficiency of the iff conditions stated in Theorem \ref{theorem-c} is proved.
Given $h(c_1, X_1) = h(c_2, X_2)$, the necessity of the iff conditions can be proved by showing possible contradictions when $\mathsf{M}_1 \ne \mathsf{M}_2$, $c_1 \ne c_2$, or $q \cdot \sum_{y=x, y\in X_1} \psi(-\beta\mathbf S_{c_1y}) \ne \sum_{y=x, y\in X_2} \psi(-\beta\mathbf S_{c_2y})$.
\end{proof}

Theorem \ref{theorem-c} shows that the function for feature aggregation ($h$) using the attention scores obtained by Eq. (\ref{c-strategy}) may map different multisets into the same embedding if and only if these multisets share the same central node feature and the same node features whose 
node-node dissimilarity
%\textit{Multi-dissimilarity} 
is proportional.

For the neighborhood aggregating function that is solely based on the \textit {Subtractive apprehension span} (Eq. (\ref{s-strategy})), the following theorem shows that it cannot distinguish some structures.
% \begin{theorem}\label{theorem-s}
% Given the same assumptions shown in Theorem \ref{theorem-c},
% for all $g$ and the Subtractive apprehension span in Eq. (\ref{s-strategy}), $h(c_1, X_1) = h(c_2, X_2)$ if and only if $c_1 = c_2$, $X_1 = \lbrace \mathsf{M}, \mu_1 \rbrace$, $X_2 = \lbrace \mathsf{M}, \mu_2 \rbrace$, and $q\sum_{y=x, y\in X_1}[\sum_{x\in X_1}\psi(\mathbf S_{c_1x}) - \beta\psi(\mathbf S_{c_1y})]= \sum_{y=x, y\in X_2}[\sum_{x\in X_2}\psi($ $\mathbf S_{c_2x}) - \beta\psi(\mathbf S_{c_2y})]$, for $q > 0$ and $x \in \mathsf{M}$, where $\psi(\cdot)$ is a function for mapping values to $\mathbb R^+$.
% \end{theorem}

\begin{theorem}\label{theorem-s}
Given the same assumptions shown in Theorem \ref{theorem-c} and the aggregation function using the attention scores computed by Eq. (\ref{s-strategy}) is denoted as $h(c_i, X_i) = \sum_{x\in X_i} \alpha_{c_i x} g(x)$,
for all $g$, any two nodes $1$ and $2$ and the Subtractive apprehension span in Eq. (\ref{s-strategy}), $h(c_1, X_1) = h(c_2, X_2)$ holds if and only if $c_1 = c_2$, $\mathsf{M}_1 = \mathsf{M}_2 = \mathsf{M}$, and $q\sum_{y=x, y\in X_1}[\sum_{x\in X_1}\psi(\mathbf S_{c_1x}) - \beta\psi(\mathbf S_{c_1y})]= \sum_{y=x, y\in X_2}[\sum_{x\in X_2}\psi($ $\mathbf S_{c_2x}) - \beta\psi(\mathbf S_{c_2y})]$, for $q > 0$ and $x \in \mathsf{M}$, where $\psi(\cdot)$ is a function for mapping values to $\mathbb R^+$.
\end{theorem}

\begin{proof}
The full proof of Theorem \ref{theorem-s} can also be checked in supplementary materials.
\end{proof}

Theorem \ref{theorem-s} shows that $h$ solely based on Eq. (\ref{s-strategy}) may map different multisets into the same embedding if and only if the multisets share the same central node feature, and the same node features whose adjusted negative 
dissimilarity
%\textit {Multi-dissimilarity}
is proportional.
Theorems \ref{theorem-c} and \ref{theorem-s} show that the expressive power of \mecha~layers solely utilizing \textit {Contractive} (Eq. (\ref{c-strategy})) or \textit {Subtractive apprehension span} (Eq. (\ref{s-strategy})) is stronger than that of classical graph attention networks \citep{DBLP:conf/iclr/VelickovicCCRLB18},
although they cannot discriminate some structures.  
As shown in the two presented theorems, the conditions giving rise to the failure of attention layers solely utilizing the \textit {Contractive} or \textit {Subtractive apprehension span} in distinguishing all structures are dependent on both node features and 
%\textit{Multi-dissimilarity}
node-node dissimilarity.
As node features and node-node dissimilarity are generally heterogeneous, it is infrequent for the conditions stated in Theorems \ref{theorem-c} and \ref{theorem-s} to be simultaneously satisfied.
Such observation may well explain those attention-based GNNs additionally considering properties other than node features, e.g., injecting structural node embeddings \citep{DBLP:conf/wsdm/QiuDMLWT18} into the computation of attention coefficients can outperform classical graph attention networks.

However, the expressive power of \model s can be immediately improved to be equivalent to the 1-WL test through the slight modification as Eq. (\ref{att-aggregation}) shows.
We next prove the proposed \mechanism~mechanisms (Eqs. (\ref{c-strategy}) to (\ref{att-aggregation})) can reach the upper bound of the expressive power of all message-passing GNNs by verifying that the aggregation function based on Eq. (\ref{att-aggregation}) can successfully distinguish the structures whose properties meet the conditions stated in Theorems \ref{theorem-c} and \ref{theorem-s}.
\begin{corollary}\label{coro-att}
	Assume $\mathcal T$ is the attention-based aggregator shown in Eq. (\ref{att-aggregation}) and utilizes either Contractive (Eq. (\ref{c-strategy})) or Subtractive apprehension span (Eq. (\ref{s-strategy})), $\mathcal H$ is a mapping of countable feature space $\mathcal X$.
	$\mathcal T$ operates on a multiset $H \in \mathcal H$.
	A $\mathcal H$ exists so that with the attention-based aggregator in Eq. (\ref{att-aggregation}), $\mathcal T$ can distinguish all different multisets that it previously cannot distinguish.
\end{corollary}

\begin{proof}
Corollary \ref{coro-att} can be proved by following the procedure presented in \cite{DBLP:conf/iclr/XuHLJ19,he2021learning}.
According to Theorem \ref{theorem-c}, we assume $X_1 = (\mathsf{M}, \mu_1)$, $X_2 = (\mathsf{M}, \mu_2)$, $c \in \mathsf{M}$, and $q\cdot\sum_{y=x, y\in X_1} \psi(-\beta\mathbf S_{c_1y}) = \sum_{y=x, y\in X_2} \psi(-\beta\mathbf S_{c_2y})$, for $q > 0$.
When $\mathcal T$ uses the attention scores solely according to Eq. (\ref{c-strategy}) to aggregate node features, we have $\sum_{x\in X_1} \alpha_{cx} g(x) = \sum_{x\in X_2} \alpha_{cx} g(x)$.
This means $\mathcal T$ fails to discriminate the structures satisfying the conditions stated in Theorem \ref{theorem-c}.
When $\mathcal T$ uses Eq. (\ref{att-aggregation}) where the attention coefficients are obtained by the \textit{Contractive apprehension span} (Eq. (\ref{c-strategy})) to aggregate node features, we have $\sum_{x\in X_1} \alpha_{cx} g(x) - \sum_{x\in X_2} \alpha_{cx} g(x) = \epsilon (\frac{1}{|X_1|}-\frac{1}{|X_2|})\alpha_{cc} g(c)$, where $|X_1| = |\mathcal N_1|$, and $|X_2| = |\mathcal N_2|$.
Since $|X_1| \ne |X_2|$, $\sum_{x\in X_1} \alpha_{cx} g(x) - \sum_{x\in X_2} \alpha_{cx} g(x) \ne 0$, which means $\mathcal T$ based on Eqs. (\ref{c-strategy}) and (\ref{att-aggregation}) is able to discriminate all the structures that $\mathcal T$ solely based on Eq. (\ref{c-strategy}) fails to distinguish.
Following the similar procedure, when the \mechanism~layer (Eq. (\ref{att-aggregation})) utilizes the \textit{Subtractive apprehension span} (Eq. (\ref{s-strategy})), we are able to prove that the corresponding aggregation function also can distinguish those distinct structures that the aggregation function only using \textit{Subtractive apprehension span} fails to discriminate.
\end{proof}

\paragraph{Remarks} Based on the conducted theoretical analysis, 
%Based on the above theoretical analysis, 
the proposed \model s are the most powerful message-passing GNNs and are consequently more powerful than popular attention-based GNNs, e.g., GATs \citep{DBLP:conf/iclr/VelickovicCCRLB18} and HardGAT \citep{DBLP:conf/kdd/GaoJ19}.
In this paper, we mainly verify that the proposed Graph \mechanismlowercase~networks are the most powerful message-passing GNNs under the condition that the feature space is countable \cite{he2021learning,DBLP:conf/iclr/XuHLJ19,zhang2020improving}.
Recent studies have shown that a simplex operator for feature aggregations in some GNN layer is injective, i.e., the 1-WL test equivalent in the countable feature space \cite{corso2020principal}.
But such injectivity might not hold when the simplex operator operates in the uncountable feature space.
To ensure injectivity when a GNN deals with uncountable features, diverse forms of operators, e.g., mean, max, and min operators are required to collaboratively aggregate neighbor features.
Thus, the expressive power of the proposed \model s can be equivalent to the 1-WL test in uncountable feature space by appropriately integrating with other effective operators for feature aggregation.
\begin{table}[htbp]
    \caption{Dataset statistics}
    \label{tab:dataset}
    \small
    \centering
    \resizebox{\linewidth}{!}{
    \begin{tabular}{lcrrrr}
        \toprule
        \textbf{Dataset} & \textbf{Type} & \textbf{Nodes} & \textbf{Edges} & \textbf{Features} & \textbf{Classes} \\
        \midrule
        Cora     & Citation &  2,708 &   5,429 &  1,433 & 7  \\
        Cite & Citation &  3,327 &   4,552 &  3,703 & 6  \\
        Pubmed   & Citation & 19,717 &  44,324 &    500 & 3  \\
        Wiki     & Web      &  2,405 &  17,981 &  4,973 & 17 \\
        Uai      & Web      &  3,363 &  33,300 &  4,971 & 19 \\
        CoauthorCS& Collaboration & 18,333 & 327,576 &  6,805 & 15 \\
        \bottomrule
    \end{tabular}}
\end{table}

\section{Experiment and analysis}\label{experiment}
In this section, we evaluate the effectiveness of the proposed \mechanism~by comparing \model s with state-of-the-art approaches for two semi-supervised learning tasks on well-established benchmarking datasets.
To understand the performance of SATs, we also perform ablation studies and visualize the distribution of learned attention coefficients. We also analyze the effect of $\beta$ and the space consumption of SATs.

\begin{table*}[ht]
    \caption{Performance comparison on semi-supervised node classification. The results in bold show that SAT outperforms all the baselines and the best baselines are underlined.}
    \label{tab:clf}
    \centering
    \resizebox{0.75\textwidth}{!}{
    \begin{tabular}{lcccccc}
        \toprule
        \textbf{Models} & \textbf{Cora} & \textbf{Cite} & \textbf{Pubmed} & \textbf{Wiki} & \textbf{Uai} & \textbf{CoauthorCS} \\
        \midrule
        MoNet & 80.93 $\pm$ 0.41 & 67.54 $\pm$ 0.69 & 80.06 $\pm$ 0.55 & 68.36 $\pm$ 0.27 & 53.58 $\pm$ 0.26 & 87.86 $\pm$ 0.35  \\
        GCN & 81.79 $\pm$ 0.21 & 71.83 $\pm$ 0.72 & 79.80 $\pm$ 0.31 & 66.99 $\pm$ 0.15 & \underline{60.72} $\pm$ 0.23 & 88.50 $\pm$ 1.06  \\
        JKNet & 78.70 $\pm$ 0.45 & 66.30 $\pm$ 0.16 & 79.73 $\pm$ 0.04 & 63.86 $\pm$ 0.19 & 54.08 $\pm$ 0.84 & 87.66 $\pm$ 1.02  \\
        APPNP & 82.50 $\pm$ 0.42 & 72.04 $\pm$ 0.48 & \underline{82.76} $\pm$ 0.37 & 66.94 $\pm$ 0.20 & 52.74 $\pm$ 0.34 & 88.46 $\pm$ 0.21 \\
        ARMA & 80.47 $\pm$ 0.71 & 69.45 $\pm$ 0.57 & 76.63 $\pm$ 0.56 & 66.94 $\pm$ 0.17 & 54.56 $\pm$ 0.15 & 82.18 $\pm$ 0.37  \\
        %GIN & 81.58 $\pm$ 0.62 & 66.90 $\pm$ 0.16 & 80.76 $\pm$ 0.33 & 73.34 $\pm$ 0.21 & 57.26 $\pm$ 0.16 & 82.94 $\pm$ 0.21  \\\midrule
        GIN & 81.27 $\pm$ 0.57 & 66.30 $\pm$ 0.59 & 80.37 $\pm$ 0.69 & \underline{71.94} $\pm$ 0.27 & 57.26 $\pm$ 0.16 & 82.94 $\pm$ 0.21  \\
        \midrule
        GAT & 83.72 $\pm$ 0.78 & 70.48 $\pm$ 0.39 & 81.17 $\pm$ 0.49 & 66.12 $\pm$ 0.14 & 54.02 $\pm$ 0.09 & 87.44 $\pm$ 0.55  \\
        GATv2 & 84.06 $\pm$ 0.83 & 71.44 $\pm$ 0.45 & 81.94 $\pm$ 0.29 & 66.82 $\pm$ 0.41 & 56.02 $\pm$ 0.21 & 87.84 $\pm$ 0.11 \\
        CAT & \underline{84.38} $\pm$ 0.25 & 72.08 $\pm$ 0.70 & 82.60 $\pm$ 0.27 & 71.73 $\pm$ 0.38 & 54.48 $\pm$ 0.52 & \underline{90.71} $\pm$ 0.21  \\
        \midrule
        %GAT-$k$-Lap & 84.10 $\pm$ 0.24 & 71.18 $\pm$ 0.52 & 82.56 $\pm$ 0.30 & 64.74 $\pm$ 0.19 & 53.68 $\pm$ 0.17 & 88.84 $\pm$ 0.10 \\
        HardGAT & 79.44 $\pm$ 0.52 & \underline{72.14} $\pm$ 0.44 & 81.00 $\pm$ 0.29 & 60.22 $\pm$ 0.21 & 46.82 $\pm$ 0.15 & 86.30 $\pm$ 0.62  \\
        GraphSage & 81.36 $\pm$ 0.44 & 70.50 $\pm$ 0.66 & 79.28 $\pm$ 0.26 & 68.04 $\pm$ 0.21 & 57.44 $\pm$ 0.19 & 89.96 $\pm$ 0.97  \\
        Neural Sparse & 83.33 $\pm$ 0.18 & 71.50 $\pm$ 0.22 & 81.07 $\pm$ 0.21 & 68.77 $\pm$ 0.11 & 46.40 $\pm$ 0.22 & 88.07 $\pm$ 0.74 \\
        \midrule
        \model-C & \textbf{85.02} $\pm$ 0.35 & \textbf{72.40} $\pm$ 0.28 & \textbf{83.16} $\pm$ 0.10 & \textbf{73.96} $\pm$ 0.35 & \textbf{61.26} $\pm$ 0.71 & \textbf{91.34} $\pm$ 0.26  \\
        \model-S & \textbf{84.66} $\pm$ 0.23 & \textbf{72.66} $\pm$ 0.23 & \textbf{83.22} $\pm$ 0.15 & \textbf{72.96} $\pm$ 0.38 & 58.88 $\pm$ 0.66 & 90.54 $\pm$ 0.39  \\
        \bottomrule
    \end{tabular}}
\end{table*}

\begin{table*}[ht]
    \caption{Performance comparison on semi-supervised node clustering. The results in bold show that SAT outperforms all the baselines and the best baselines are underlined.}
    \label{tab:cls}
    \centering
    \resizebox{0.75\textwidth}{!}{
    \begin{tabular}{lcccccc}
        \toprule
        \textbf{Models} & \textbf{Cora} & \textbf{Cite} & \textbf{Pubmed} & \textbf{Wiki} & \textbf{Uai} & \textbf{CoauthorCS} \\
        \midrule
        MoNet & 79.06 $\pm$ 0.61 & 63.35 $\pm$ 0.29 & 79.78 $\pm$ 0.40 & 71.68 $\pm$ 0.14 & 55.21 $\pm$ 0.31 & 87.52 $\pm$ 0.92 \\
        GCN & 74.61 $\pm$ 0.24 & 63.61 $\pm$ 0.80 & 77.56 $\pm$ 0.56 & 72.78 $\pm$ 0.51 & 60.14 $\pm$ 0.18 & 88.72 $\pm$ 0.73 \\
        JKNet & 76.94 $\pm$ 0.48 & 64.33 $\pm$ 0.02 & 79.53 $\pm$ 0.47 & 67.90 $\pm$ 0.17 & 57.15 $\pm$ 0.09 & 87.89 $\pm$ 0.48  \\
        APPNP & 80.44 $\pm$ 0.94 & 70.33 $\pm$ 0.62 & \underline{82.53} $\pm$ 0.45 & 69.71 $\pm$ 0.19 & 53.63 $\pm$ 0.40 & 88.68 $\pm$ 0.39  \\
        ARMA & 77.99 $\pm$ 0.91 & 68.07 $\pm$ 0.89 & 78.01 $\pm$ 0.98 & 71.28 $\pm$ 0.12 & 58.44 $\pm$ 0.13 & 82.43 $\pm$ 0.33 \\
        GIN & 79.59 $\pm$ 0.36 & 66.04 $\pm$ 0.83 & 78.39 $\pm$ 0.88 & \underline{74.18} $\pm$ 0.15 & \underline{60.39} $\pm$ 0.11 & 85.34 $\pm$ 0.29  \\\midrule
        GAT & \underline{81.46} $\pm$ 0.26 & 68.49 $\pm$ 0.26 & 81.30 $\pm$ 0.17 & 70.52 $\pm$ 0.13 & 56.45 $\pm$ 0.64 & 88.00 $\pm$ 0.36 \\
        GATv2 & 80.79 $\pm$ 0.54 & 70.03 $\pm$ 0.36 & 81.43 $\pm$ 0.17 & 69.10 $\pm$ 0.17 & 57.47 $\pm$ 0.22 & 88.13 $\pm$ 0.56  \\
        CAT & 81.13 $\pm$ 0.24 & \underline{70.35} $\pm$ 0.33 & 82.07 $\pm$ 0.16 & 73.29 $\pm$ 0.37 & 56.11 $\pm$ 0.41 & \underline{90.48} $\pm$ 0.22  \\
        \midrule
        HardGAT & 76.53 $\pm$ 0.36 & 69.68 $\pm$ 0.44 & 79.89 $\pm$ 0.18 & 63.19 $\pm$ 0.18 & 47.65 $\pm$ 0.15 & 87.32 $\pm$ 0.55  \\
        GraphSage & 78.50 $\pm$ 0.67 & 68.92 $\pm$ 0.42 & 79.74 $\pm$ 0.49 & 70.72 $\pm$ 0.70 & 59.77 $\pm$ 0.11 & 89.97 $\pm$ 0.43 \\
        Neural Sparse & 80.69 $\pm$ 0.24 & 69.97 $\pm$ 0.19 & 80.79 $\pm$ 0.04 & 71.19 $\pm$ 0.35 & 50.98 $\pm$ 0.20 & 88.38 $\pm$ 0.44  \\
        \midrule
        \model-C & \textbf{81.85} $\pm$ 0.27 & \textbf{70.82} $\pm$ 0.19 & 82.28 $\pm$ 0.09 & \textbf{75.17} $\pm$ 0.31 & \textbf{61.48} $\pm$ 0.61 & \textbf{90.84} $\pm$ 0.21  \\
        \model-S & \textbf{81.47} $\pm$ 0.21 & \textbf{70.94} $\pm$ 0.09 & 82.22 $\pm$ 0.07 & 74.05 $\pm$ 0.38 & 59.31 $\pm$ 0.55 & 90.29 $\pm$ 0.09  \\
        \bottomrule
    \end{tabular}}
\end{table*}

\subsection{Experimental set-up}

\paragraph{Baselines}
\model s are compared with twelve strong baselines, including MoNet \citep{DBLP:conf/cvpr/MontiBMRSB17}, GCN \citep{DBLP:conf/iclr/KipfW17}, GraphSage \citep{DBLP:conf/nips/HamiltonYL17}, JKNet \citep{DBLP:conf/icml/XuLTSKJ18}, APPNP \citep{DBLP:conf/iclr/KlicperaBG19}, ARMA \citep{bianchi2021graph}, GIN \citep{DBLP:conf/iclr/XuHLJ19}, Neural Sparse \cite{DBLP:conf/icml/ZhengZCSNYC020}, GAT \citep{DBLP:conf/iclr/VelickovicCCRLB18}, GATv2 \citep{brody2022how}, CAT \citep{he2021learning}, 
%GAT-$k$-Lap \citep{DBLP:conf/wsdm/QiuDMLWT18}, 
and HardGAT \citep{DBLP:conf/kdd/GaoJ19}.
MoNet, GCN, JKnet, APPNP, and ARMA are five representative GNNs that leverage graph convolutional operators to learn representations.
GIN and CAT are two state-of-the-art GNNs whose expressive power is equivalent to the 1-WL test.
GAT and GATv2 are two powerful attention-based GNNs that learn representations by attending to all neighbors of each node in the graph.
GraphSage, Neural Sparse, and HardGAT are three state-of-the-art GNNs that perform graph learning tasks by aggregating the features from sampled neighbors.
By comparing with diverse types of GNNs which consider different scopes of neighbors for feature aggregation, the effectiveness of the proposed \model s can be better validated.

\paragraph{Datasets} 
Six widely used datasets, including Cora, Cite, Pubmed \citep{DBLP:conf/icml/LuG03, DBLP:journals/aim/SenNBGGE08}, Wiki \citep{DBLP:conf/icml/LuG03}, Uai \citep{DBLP:conf/icml/LuG03}, and CoauthorCS \citep{DBLP:journals/corr/abs-1811-05868} are used for evaluation.
Cora, Cite, and Pubmed are citation networks widely used for the evaluation of GNNs. However, recent studies show that they may be insufficient to evaluate the performance of GNNs due to their limited data size \citep{DBLP:conf/nips/HuFZDRLCL20, DBLP:journals/corr/abs-1811-05868}.
By following previous work, we additionally include Wiki, Uai, and CoauthorCS in our experiments. The data statistics are summarized in Table \ref{tab:dataset}, and more details can be found in Appendix \ref{app:data_description}.

\paragraph{Evaluation and experimental settings} Following previous work \citep{DBLP:conf/iclr/VelickovicCCRLB18,he2021learning,DBLP:conf/nips/KlicperaWG19,DBLP:conf/iclr/KipfW17}, we mainly consider two learning tasks to test the effectiveness of different approaches, i.e., semi-supervised node classification and semi-supervised node clustering.
For the training and testing paradigms, we closely follow the established settings reported in the previous studies \citep{DBLP:conf/iclr/KipfW17,DBLP:conf/iclr/VelickovicCCRLB18,DBLP:conf/iclr/KlicperaBG19,yang2016revisiting,he2021learning}.
The classification and clustering performances of all approaches are evaluated by \textit{Accuracy}.
In the training phase, all the GNNs are implemented with the two-layer network structure, i.e., one hidden layer followed by the output layer.
All approaches are run 10 times on each testing dataset to obtain the average performance.
We report the detailed experimental settings in Appendix \ref{detail-setting}.

\subsection{Results of semi-supervised learning}
Semi-supervised node classification and clustering are closely related to several important applications from real-world graph data, such as social community detection and document classification.
In our experiments, we use the proposed SATs to perform these two learning tasks in the aforementioned real graph datasets and compare their performance with that of state-of-the-art GNNs.
The corresponding results are summarized in Tables \ref{tab:clf} and \ref{tab:cls}.

The performance comparisons of semi-supervised node classification are presented in Table \ref{tab:clf}.
%, which shows that our \model s perform on par or outperform the baselines.
As the table shows, \model s consistently outperform all the baselines on all the datasets. Specifically, \model~archives 0.76\%, 0.72\%, 0.56\%, 2.81\%, 0.89\%, and 0.69\% improvement over the best baselines (which are underlined in Table \ref{tab:clf}) in terms of \textit{Accuracy}.
It is worth noting that \model~achieves significant performance gain over all attention-based GNNs. Specifically, \model~archives 0.76\%, 0.72\%, 0.75\%, 3.11\%, 9.35\% and 0.69\% performance gain over the best attention-based GNNs.
The better performance of SATs could be attributed to its ability of identifying irrelevant neighbors and only attending to those highly relevant ones.
This is why the performance improvement of \model s over attention-based GNNs is much more significant on Wiki, Uai, and CoauthorCS, where the density of edges is higher and more irrelevant neighbors exist.
We will visualize the distribution of attention scores learned by \model s in Section \ref{sec:attention_analysis} to demonstrate their ability to identify and ignore irrelevant neighbors.
%We will also visualize the distribution of attention scores in Section \ref{sec:attention_analysis}.

% Then we consider a more challenging task, semi-supervised clustering, to validate the effectiveness of \model s.
The results of semi-supervised clustering are summarized in Table \ref{tab:cls}. \model~archives the best scores on five datasets, and it outperforms all the attention-based methods on all six datasets. Specifically, \model~archives 0.48\%, 0.84\%, 0.26\%, 2.57\%, 6.98\%, 0.40\% performance gain over the best attention-based GNN. Similar to the results on the first task, the \textit{Accuracy} increases reported by \model~over attention-based GNNs are more significant on dense graphs, including Wiki, Uai, and CoauthorCS.

\subsection {Ablation studies} \label{sec:ablation_study}

To further investigate the effectiveness of our model, we conduct ablation studies on the two types of dissimilarity (Eq. (\ref{bi-conflict})) of \model.
% As diverse forms of node-node dissimilarity (Eq. (\ref{bi-conflict})) are used to compute novel attention coefficients (Eq. (\ref{c-strategy}) or (\ref{s-strategy})), 
We consider vanilla GAT, \model~with dissimilarity only regarding node features (C/S-F) (i.e. $r_f = 1, r_p=0$ in Eq. (\ref{bi-conflict})), \model~with the dissimilarity of graph structure only (C/S-P) (i.e. $r_f=0, r_p=1$ in Eq. (\ref{bi-conflict})), and the complete \model~model. 
The results of semi-supervised node classification and clustering tasks are summarized in Tables \ref{tab:ablation_maintext} and \ref{tab:ablation_cls}.
%For brevity, we only list the results of \model~in the semi-supervised node classification task in Table \ref{tab:ablation_maintext}, while we list ablation results in semi-supervised clustering tasks in Appendix \ref{app:ablation}.
From the tables, we observe that either using node features or node structural properties to compute node-node dissimilarity can improve the performances of graph attention, 
which demonstrates that ignoring irrelevant neighbors in graph attention indeed improves the performance.
Meanwhile, considering both dissimilarities 
% to compute \mecha~coefficients makes \model s 
produces better and more stable results.

\begin{table}[ht]
    \caption{Ablation study in semi-supervised node classification}
    \label{tab:ablation_maintext}
    \centering
    \resizebox{\linewidth}{!}{
    \begin{tabular}{cc|l|cccccc}
        \toprule
        \multicolumn{9}{c}{\bf \model~ (\textit{Contractive apprehension span})} \\
        \midrule
        \multicolumn{2}{c|}{\bf Dissimilarity} & 
        \multirow{2}*{\bf Models} & \multirow{2}*{\bf Cora} & \multirow{2}*{\bf Cite} & \multirow{2}*{\bf Pubmed} & \multirow{2}*{\bf Wiki} & \multirow{2}*{\bf Uai} & \multirow{2}*{\bf CoauthorCS} \\ {\bf node feature} & {\bf node structure} & & & & & & & \\
        \midrule
         &  & GAT  &  83.84  &  70.36  &  81.50  &  66.12  &  54.02  &  87.44 \\
        $\surd$ &  & C-F  &  83.87  &  72.00  &  82.17  &  71.70  &  60.17  &  90.77 \\
         & $\surd$ & C-P  &  84.37  &  72.60  &  82.70  &  73.50  &  59.13  &  89.93 \\
         $\surd$ & $\surd$ & \model-C  &  85.02  &  72.40  &  83.16  &  73.96  &  61.26  &  91.34 \\
        \midrule
        \multicolumn{9}{c}{\bf \model~ (\textit{Subtractive apprehension span})} \\
        \midrule
        \multicolumn{2}{c|}{\bf Dissimilarity} & 
        \multirow{2}*{\bf Models} & \multirow{2}*{\bf Cora} & \multirow{2}*{\bf Cite} & \multirow{2}*{\bf Pubmed} & \multirow{2}*{\bf Wiki} & \multirow{2}*{\bf Uai} & \multirow{2}*{\bf CoauthorCS} \\ {\bf node feature} & {\bf node structure} & & & & & & & \\
        \midrule
        &  &  GAT  &  83.84  &  70.36  &  81.50  &  66.12  &  54.02  &  87.44  \\
        $\surd$ &  & S-F  &  84.30  &  72.40  &  82.07  &  72.70  &  59.13  &  90.20  \\
        & $\surd$ & S-P  &  84.13  &  72.60  &  82.50  &  73.17  &  56.97  &  90.17  \\
        $\surd$ & $\surd$ & \model-S  &  84.66  &  72.66  &  83.22  &  72.96  &  58.88  &  90.54  \\
        \bottomrule
    \end{tabular}}
\end{table}

\begin{table}[htbp]
    \vspace{-2px}
    \caption{Ablation study in semi-supervised node clustering}
    \vspace{-4px}
    \label{tab:ablation_cls}
    \centering
    \resizebox{\linewidth}{!}{
    \begin{tabular}{cc|l|cccccc}
        \toprule
        \multicolumn{9}{c}{\bf \model~ (\textit{Contractive apprehension span})} \\
        \midrule
        \multicolumn{2}{c|}{\bf Dissimilarity} & 
        \multirow{2}*{\bf Models} & \multirow{2}*{\bf Cora} & \multirow{2}*{\bf Cite} & \multirow{2}*{\bf Pubmed} & \multirow{2}*{\bf Wiki} & \multirow{2}*{\bf Uai} & \multirow{2}*{\bf CoauthorCS} \\ {\bf node feature} & {\bf node structure} & & & & & & & \\
        \midrule
        &  &  GAT  &  81.39   &  69.20   &  80.88   &  70.52   &  56.45   &  88.00   \\
        % \midrule
        $\surd$ &  & C-F  &  80.07   &  70.61   &  82.01   &  72.69   &  59.95   &  90.48   \\
        & $\surd$ & C-P  &  81.18   &  70.80   &  82.25   &  74.39   &  59.37   &  89.78   \\
        $\surd$ & $\surd$ & \model-C  &  81.85   &  70.82   &  82.28   &  75.17   &  61.48   &  90.84   \\
        \midrule
        \multicolumn{9}{c}{\bf \model~ (\textit{Subtractive apprehension span})} \\
        \midrule
        \multicolumn{2}{c|}{\bf Dissimilarity} & 
        \multirow{2}*{\bf Models} & \multirow{2}*{\bf Cora} & \multirow{2}*{\bf Cite} & \multirow{2}*{\bf Pubmed} & \multirow{2}*{\bf Wiki} & \multirow{2}*{\bf Uai} & \multirow{2}*{\bf CoauthorCS} \\ {\bf node feature} & {\bf node structure} & & & & & & & \\
        \midrule
        &  &  GAT  &  81.39  &  69.20  &  80.88  &  70.52  &  56.45  &  88.00  \\
        % \midrule
        $\surd$ &  & S-F  &  81.08  &  70.99  &  81.98  &  73.24  &  59.36  &  90.14  \\
        & $\surd$ & S-P  &  81.21  &  70.73  &  82.23  &  74.01  &  56.72  &  89.98  \\
        $\surd$ & $\surd$ & \model-S  &  81.47  &  70.94  &  82.22  &  74.05  &  59.31  &  90.29  \\
        \bottomrule
    \end{tabular}}
\end{table}

\subsection{Visualization of attention scores} \label{sec:attention_analysis}
% To demonstrate how \mecha~influences the representation learning
To demonstrate the ability of SA to ignore irrelevant neighbors, we compare the attention scores obtained by the output layer of GAT, CAT, and \model s. 
Here we show the results on Uai and the results on the other datasets are given in Appendix \ref{app:attn_coef}.
In Fig. \ref{fig:attn_dist}, we conduct a case study on a node to show how \mecha~mechanisms influence the values of attention coefficients of neighbors, and then we show the histogram of all attention scores. 
% It is observed from the case study
We observe that some attention scores obtained by \model~are close to zero, showcasing that \model~indeed can learn the scope of attention to exclude the irrelevant neighbors from the feature aggregation.
Such observation generalizes to the whole graph, as the histogram of attention scores (Fig. \ref{fig:attn_dist}) shows that more attention scores acquired by \model~are close to zero compared with GAT and SAT.
We also conduct the case study on CAT \citep{he2021learning}, which utilizes both node features and graph structure to compute attention coefficients for graph analysis, and draw the distribution of its attention scores. We do not observe a significant difference in the histogram of attention scores between CAT and GAT, showing that CAT is not able to ignore irrelevant neighbors.
% prevalent attention-based GNNs cannot learn to adjust the scope of neighborhood for feature aggregation.
%which shows that CAT does not take effect by ignoring unimportant neighbors.

\begin{figure}[ht]
    \centering
    \subfigure{
        \includegraphics[width=0.65\linewidth]{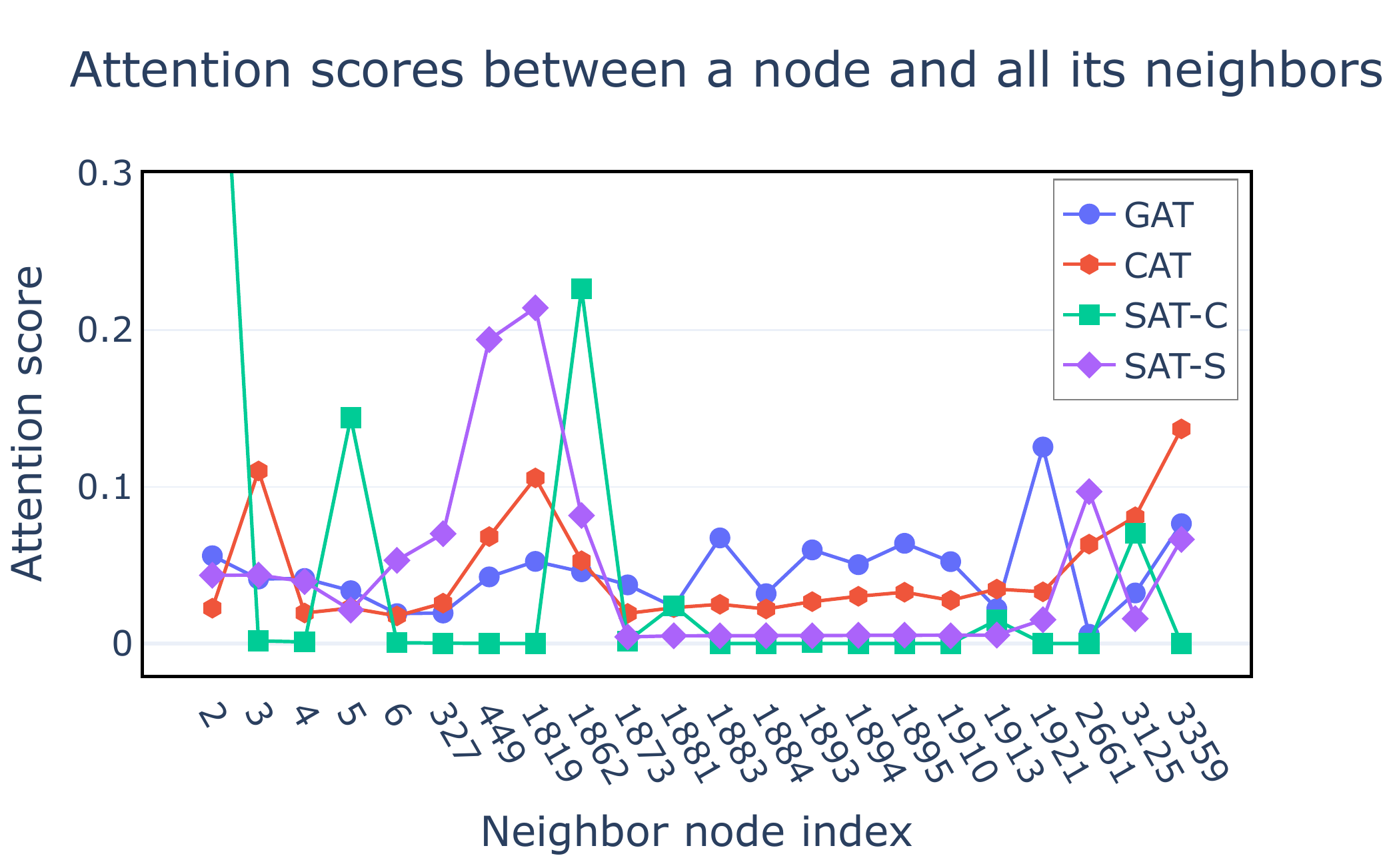}  \label{fig:attn_dist_case}
    }
    \subfigure{
        \includegraphics[width=0.65\linewidth]{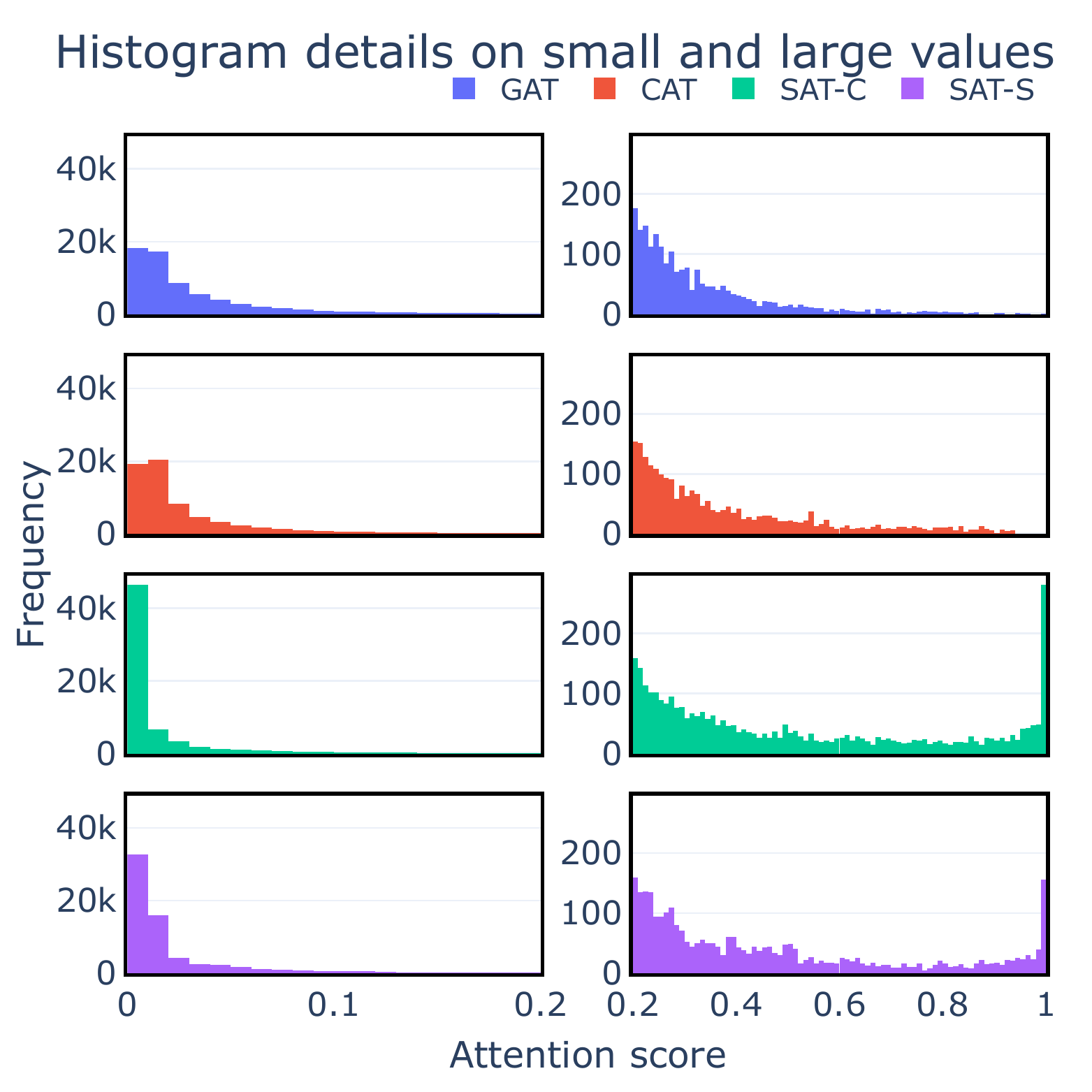}
    }
    % \subfigure{
    %     \includegraphics[width=0.35\linewidth]{images/uai_hist_small_m.pdf} \label{fig:attn_dist_hist}
    % }
    % \subfigure{
    %     \includegraphics[width=0.35\linewidth]{images/uai_hist_large_m.pdf}
    % }
    \caption{Attention scores from GAT, CAT and \model s on Uai}
    \label{fig:attn_dist}
\end{figure}

\subsection{The effect of $\beta$}
\model s use $\beta$ to control the influence from the node-node dissimilarity 
% brought by the node-node dissimilarity 
when computing attention coefficients.
The effect of the node-node dissimilarity becomes larger when $\beta$ goes high.
Potentially, \model s can learn smaller 
% apprehension spans
scopes of attention, which exclude more neighbors, or compute very low attention coefficients for more neighbors when $\beta$ is set to high values.
%Potentially, \model s might ignore more neighbors, or compute very low attention coefficients for more neighbors when $\beta$ is set to high values.
To show this, in Table \ref{neighbors}, we list the number of edges with small attention scores (<0.05) when $\beta$ ranges in [0.1, 0.5, 0.75, 1.0] on all testing datasets. 
As seen in the table, more neighbors are assigned with very small attention scores as $\beta$ goes higher.

% We further draw the performance of our model with different settings of $\beta$ in Fig. \ref{fig:sensitivity}. 
We show the \textit {Accuracy} of our models with varying values of $\beta$ in the range of $[0.01, 1.0]$ in Fig. \ref{fig:sensitivity}.
Based on the results in Table \ref{neighbors} and Fig. \ref{fig:sensitivity}, one possible way to find a better setting of $\beta$ is to configure it according to the density of the graph. Generally, SATs can perform better when $\beta$ is set to a higher value in dense graphs, such as Wiki, Uai, and CoauthorCS. In these graphs, SATs can identify a large number of dissimilar neighbors as irrelevant to the feature aggregation, so that they can learn representations by concentrating on the features of a few relevant nodes. However, in sparse graphs such as Cora, Cite, and Pubmed, SATs can perform better when a relatively small $\beta$ is used. In these graphs, the number of neighbors connecting each central node is small. Thus, SATs do not need to identify many neighbors as irrelevant ones.

\begin{table}[htbp]
\caption{The number of neighbors with very low attention coefficients ($\leq 0.05$) and proportions}
\small
\label{neighbors}
\resizebox{\linewidth}{!}{
\begin{tabular}{|c|c|c|c|c|c|c}
\cline{1-6}
\bf Dataset    & \begin{tabular}[c]{@{}c@{}}\bf C: Contractive\\ \bf S: Subtractive\end{tabular} & $\beta = 0.1$ & $\beta = 0.5$  & $\beta = 0.75$ & $\beta = 1$   \\ \cline{1-6}
\bf Cora       & \bf C                                                                       & 2273 (0.171)  & 2672 (0.201)   & 2726 (0.206)   & 2800 (0.211)  \\ \cline{2-6}
           & \bf S                                                                       & 2147 (0.162)  & 2164 (0.164)   & 2214 (0.167)   & 2415 (0.182)  \\ \cline{1-6}
\bf Cite       & \bf C                                                                       & 879 (0.071)   & 1018 (0.082)   & 1081 (0.087)   & 1120 (0.090)  \\ \cline{2-6}
           & \bf S                                                                       & 878 (0.071)   & 897 (0.072)    & 902 (0.073)    & 1047 (0.084)  \\ \cline{1-6}
\bf Pubmed     & \bf C                                                                       & 32416 (0.299) & 34078 (0.314)  & 34590 (0.319)  & 35079 (0.324) \\ \cline{2-6}
           & \bf S                                                                       & 32185 (0.297) & 32268 (0.298)  & 32822 (0.303)  & 34809 (0.321) \\ \cline{1-6}
\bf Wiki       & \bf C                                                                       & 15870 (0.619) & 16153 (0.6310) & 16190 (0.632)  & 16663 (0.651) \\ \cline{2-6}
           & \bf S                                                                       & 15944 (0.622) & 15997 (0.625)  & 16006 (0.625)  & 16308 (0.637) \\ \cline{1-6}
\bf Uai        & \bf C                                                                       & 58699 (0.839) & 58733 (0.839)  & 59726 (0.853)  & 62155 (0.888) \\ \cline{2-6}
           & \bf S                                                                       & 62705 (0.896) & 62741 (0.896)  & 62834 (0.898)  & 63688 (0.910) \\ \cline{1-6}
\bf CoauthorCS & \bf C                                                                       & 79384 (0.436) & 79838 (0.438)  & 79898 (0.439)  & 87610 (0.481) \\ \cline{2-6}
           & \bf S                                                                       & 78909 (0.433) & 78992 (0.434)  & 79857 (0.438)  & 80952 (0.444) \\ \cline{1-6}
\end{tabular}}
\end{table}

\begin{figure}[htbp]
    \centering
    \includegraphics[width=0.32\linewidth]{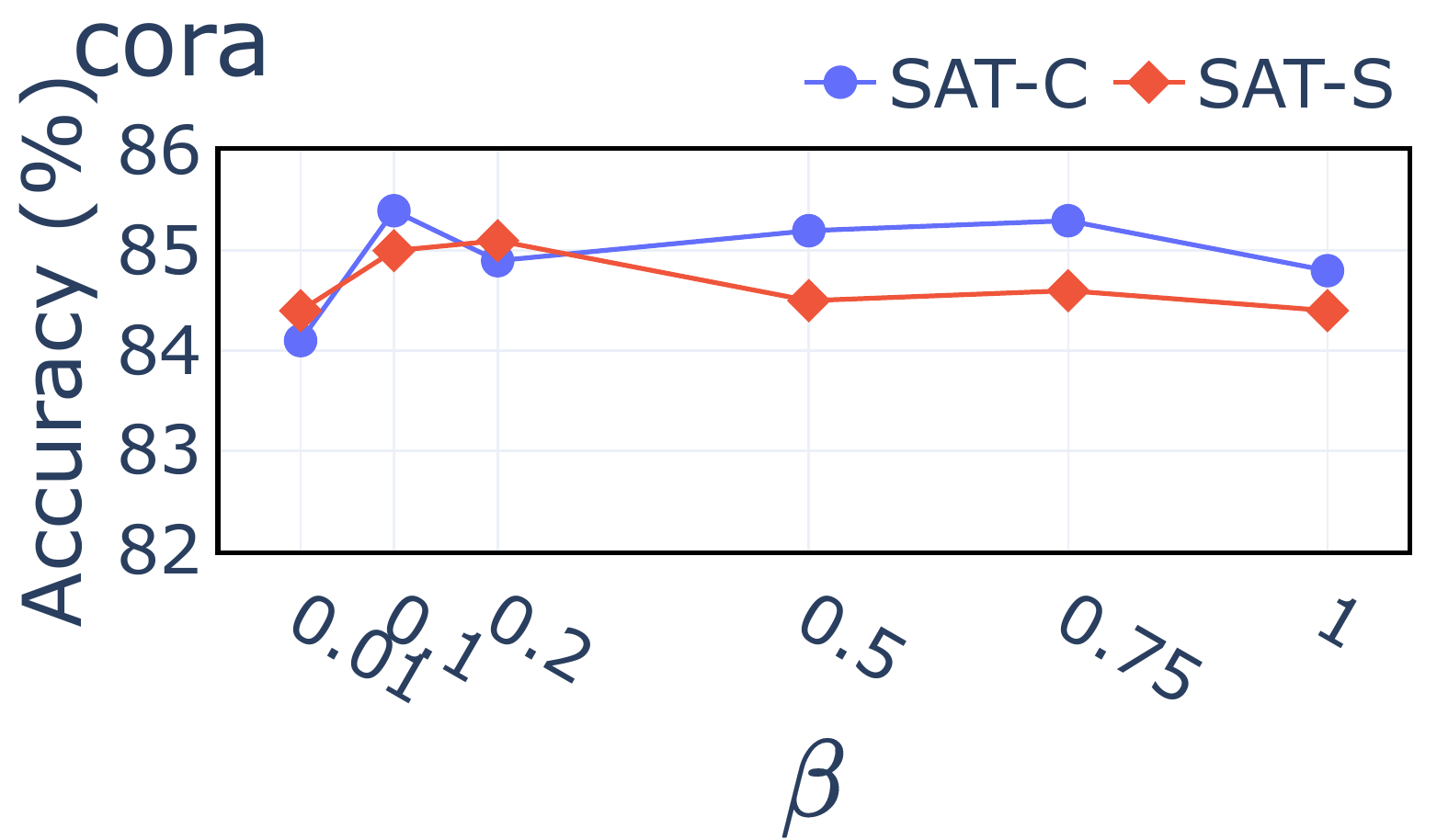} 
    \includegraphics[width=0.32\linewidth]{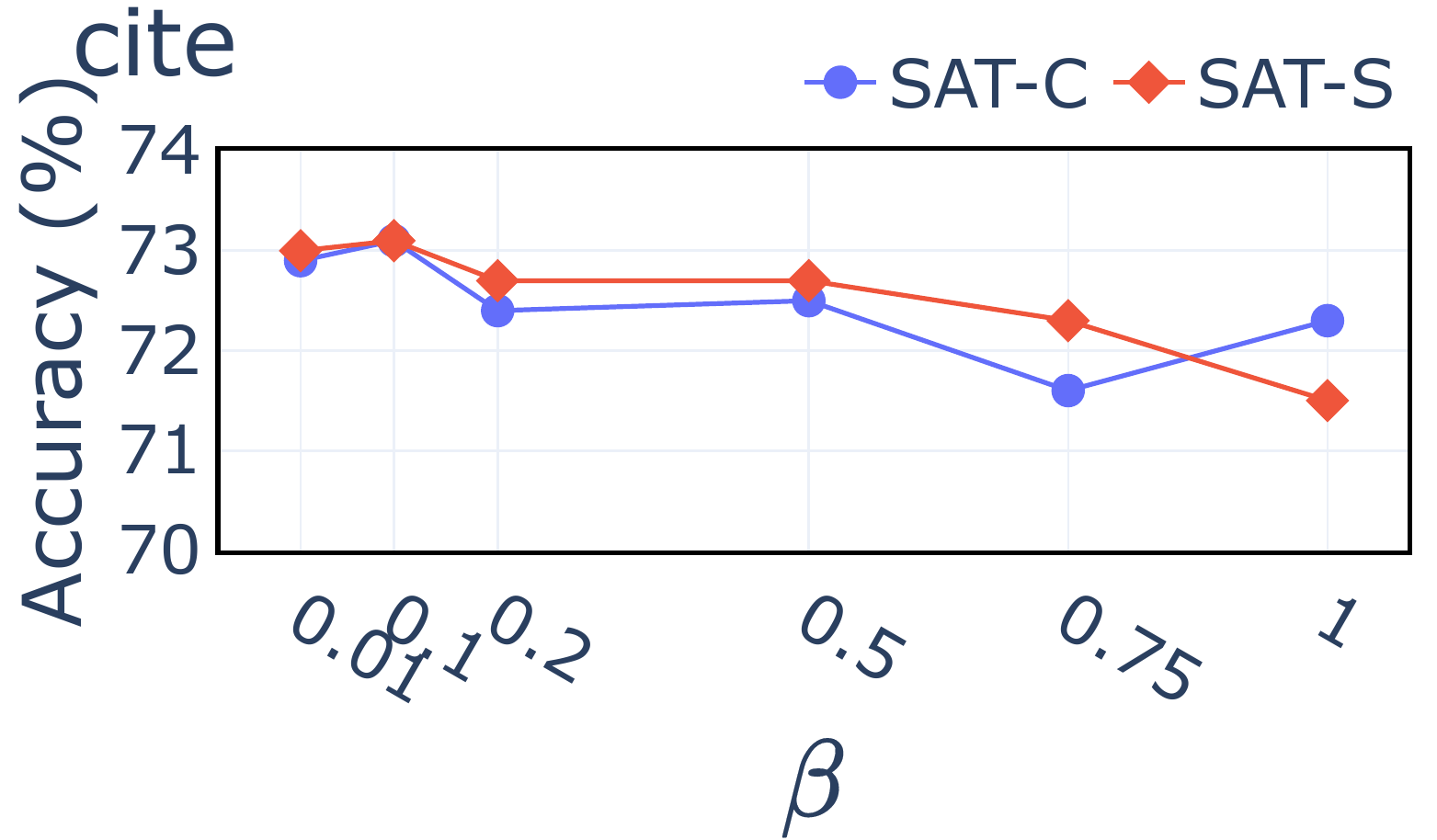} 
    \includegraphics[width=0.32\linewidth]{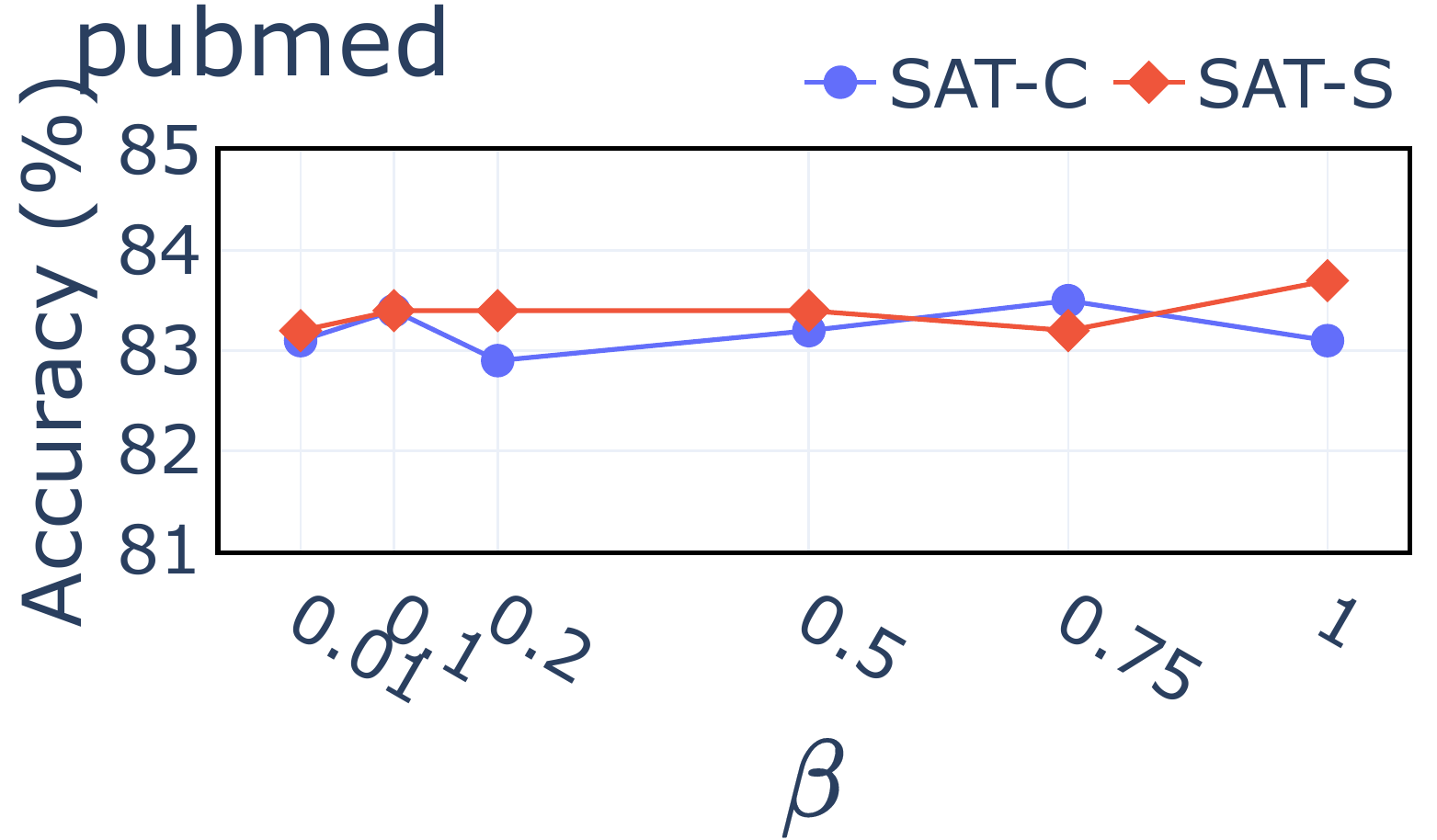}
    \includegraphics[width=0.32\linewidth]{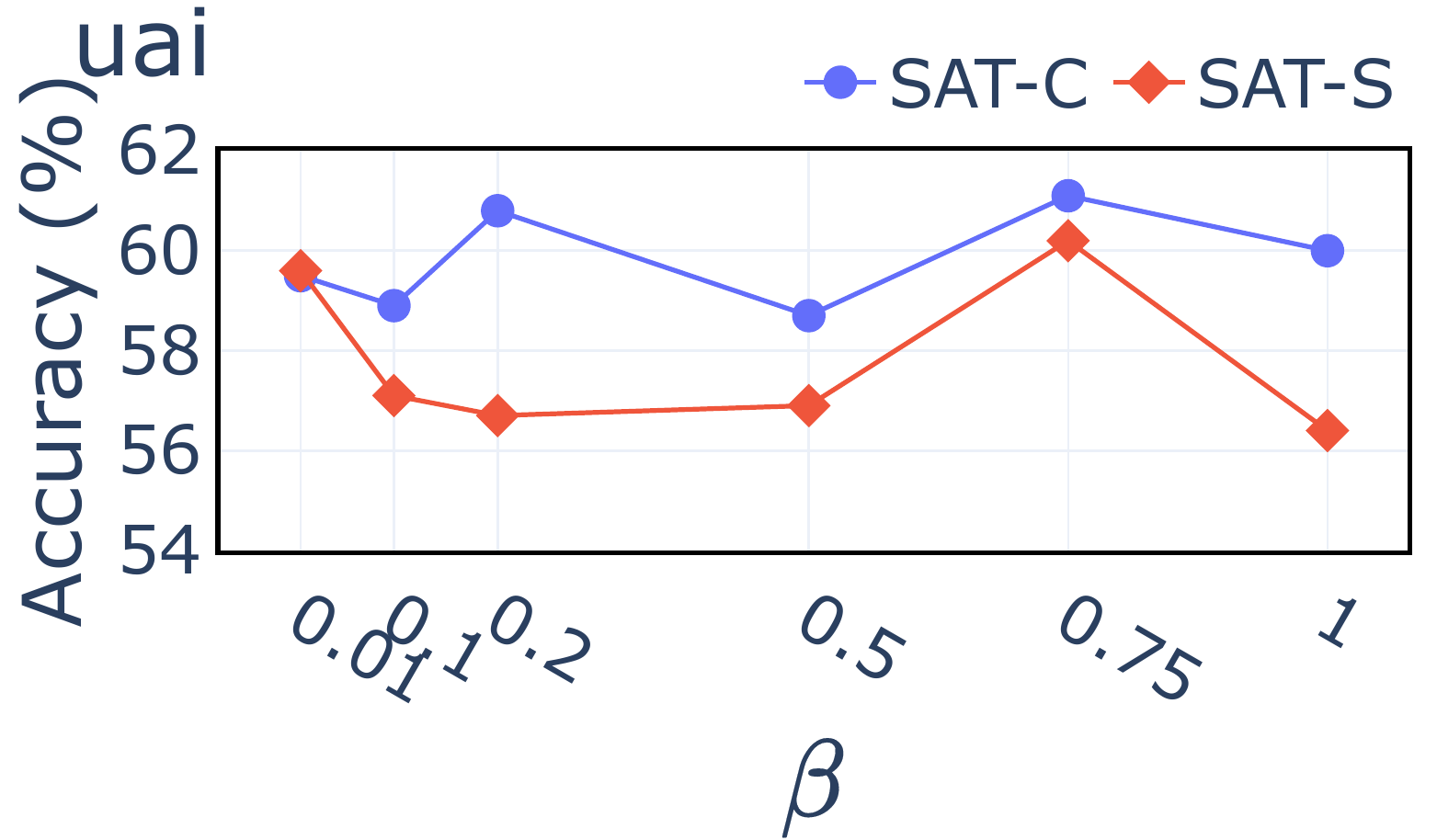} 
    \includegraphics[width=0.32\linewidth]{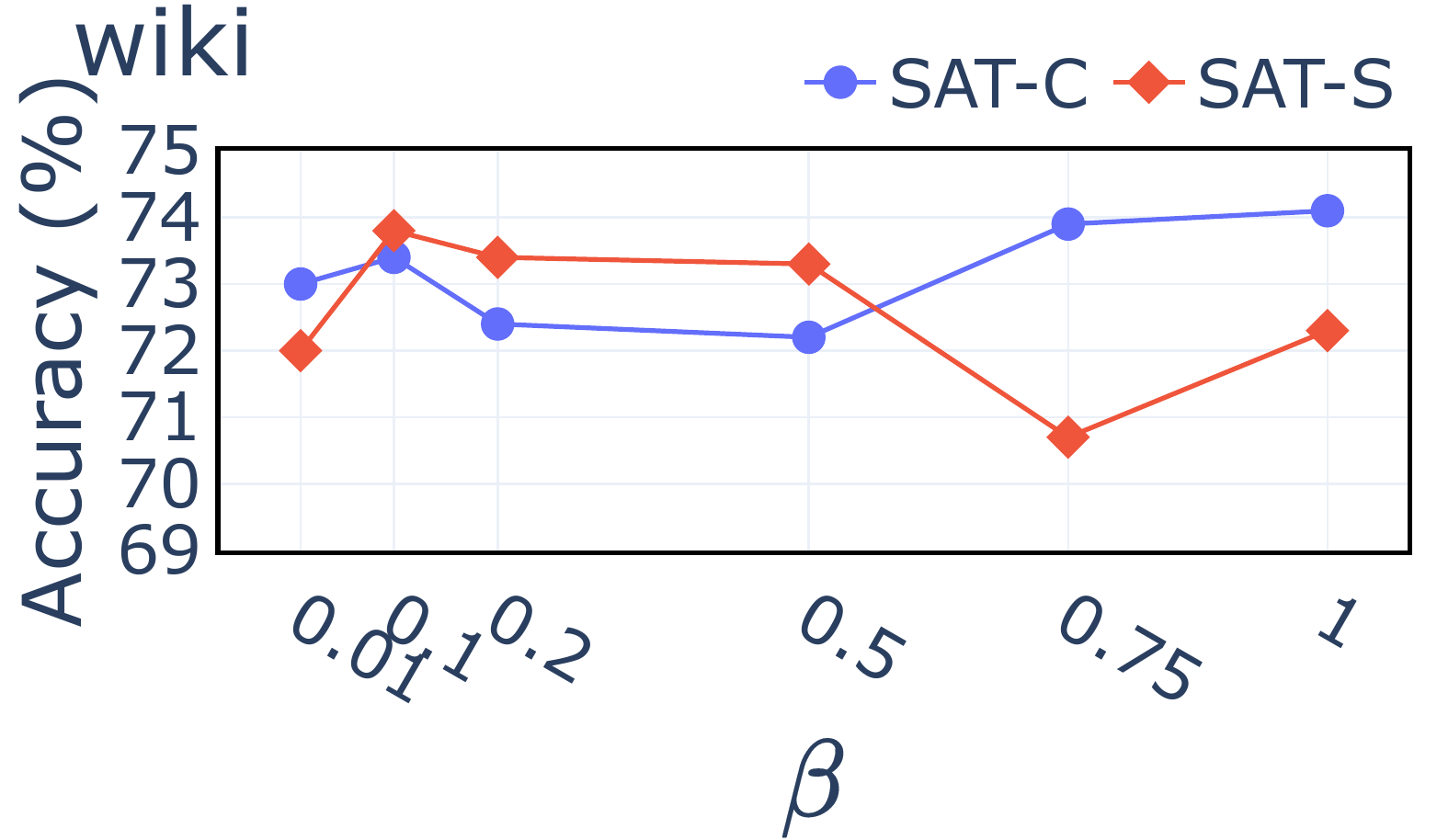} 
    \includegraphics[width=0.32\linewidth]{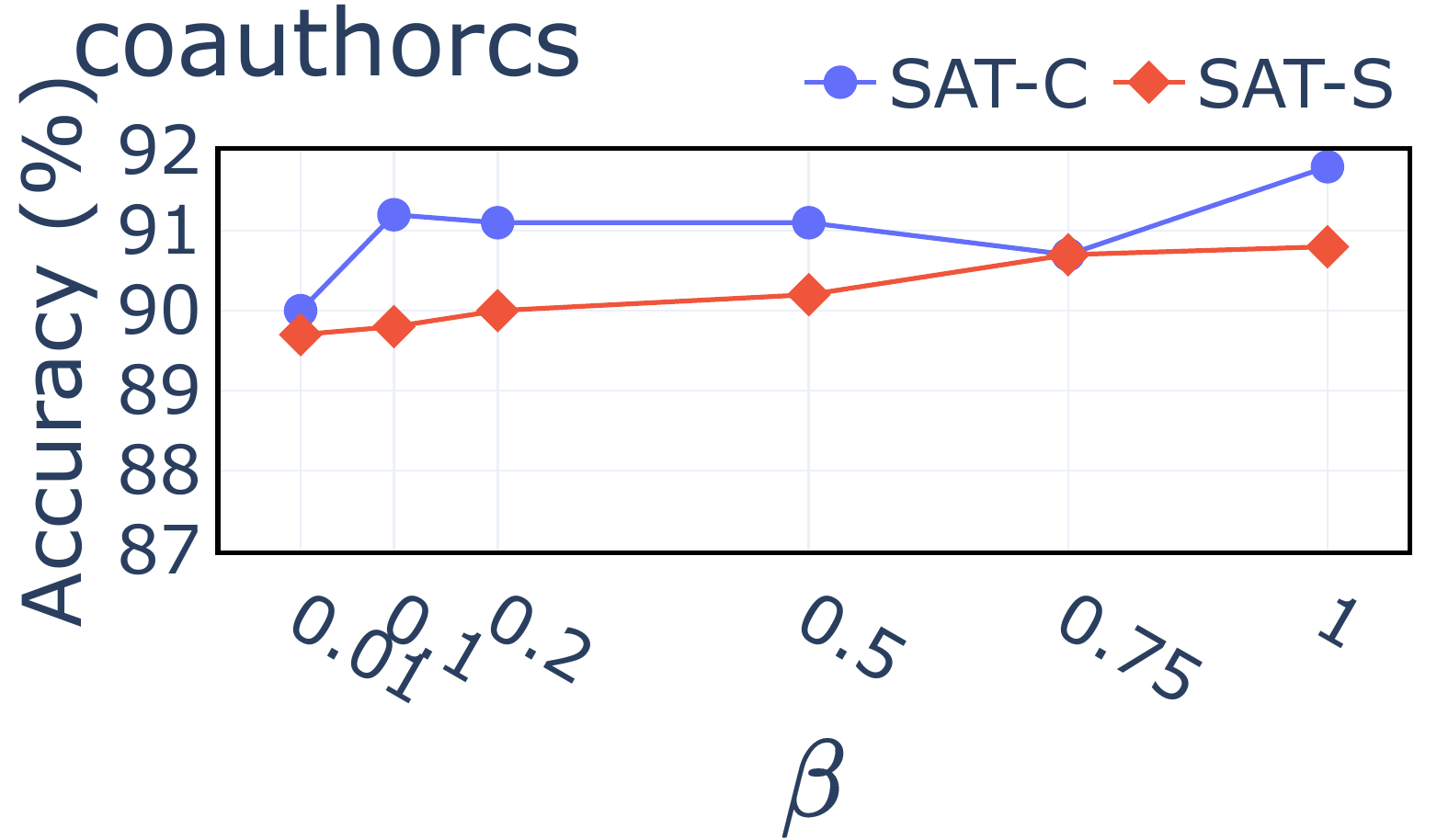} 
    \caption{Sensitivity test for $\beta$. 
    %The x-axis is $\beta$, while the y-axis is \textit{Accuracy}. 
    }
    \label{fig:sensitivity}
\end{figure}

\subsection{Comparisons on model parameters and space consumption between GAT and \model s}  \label{app:parameter_com}
To show the difference in architecture between GAT and \model s, we compare the number of parameters and memory usage between them.
The results have been summarized in Table \ref{complexity}.
SATs uses slightly more memory since more parameters, which are mainly due to the learning of latent spaces (i.e., $\mathbf p$), are used to perform the task of representation learning.
% As the table shows, more parameters, which are mainly due to the learning of latent spaces (i.e., $\mathbf p$) are used by variants of \model s to perform the task of representation learning.
Thus, \model s use slightly more memory in the training stage.
% Given the significant improvement made by \model s, such slightly more memory consumption is acceptable. 

\begin{table}[]
	\small
	\centering
	\caption{Parameter comparison between GAT and \model.}
	\label{complexity}
	\resizebox{\linewidth}{!}{
	\begin{tabular}{c|c|cccccc}
		\hline
		&                   & \textbf{Cora} & \textbf{Cite} & \textbf{Pubmed} & \textbf{Wiki} & \textbf{Uai}& \textbf{CoauthorCS}\\\hline
		\multirow{2}{*}{\textbf{GAT}}      & \# Parameters     & 92430         & 237644        & 32454           & 2682658             & 2682150 &  3620126           \\
		& Space consumption & 1.1GB         & 1.2GB         & 1.2GB           & 1.2GB               & 1.3GB   & 1.8GB           \\\hline
		\multirow{2}{*}{\textbf{\model-C}} & \# Parameters     & 111285        & 257505        & 91504           & 2723058             & 2745562  &  3894636          \\
		& Space consumption & 1.2GB         & 1.2GB         & 1.4GB           & 1.4GB               & 1.5GB  &2.5GB           \\\hline
		\multirow{2}{*}{\textbf{\model-S}} & \# Parameters     & 111285        & 257505        & 91504           & 2723058             & 2745562 &  3894636           \\
		& Space consumption & 1.2GB         & 1.2GB         & 1.4GB           & 1.4GB               & 1.6GB  &2.7GB            \\\hline
	\end{tabular}}
\end{table}

\section{Conclusion}
In this paper, we have proposed \mechanism~(\mecha), which generalizes a class of novel attention mechanisms for GNNs.
Motivated by the analogy between the apprehension span of human cognition and the scope of attention for feature aggregation, 
\mecha~leverages diverse forms of node-node dissimilarity to adapt the node-wise scope of attention, which can flexibly exclude those irrelevant neighbors from the feature aggregation stage.
% \mecha~therefore enables GNNs to learn representations by favoring the features of highly relevant nodes and ignoring those \warningtext{deemed} as irrelevant.
\mecha~therefore enables GNNs to learn representations by favoring the features of highly relevant neighbors and ignoring irrelevant neighbors.
Given different \mecha~mechanisms, we build Graph \mechanismlowercase~networks (\model s) to learn representations for various tasks arising from real-world graph data.
\model s have been tested on widely used benchmarking datasets and compared to several strong baselines.
The obtained notable results can validate the effectiveness of the proposed \mechanism.
In the future, the proposed \mecha~will be further improved by exploring more forms of node-node dissimilarity that can be used for computing \mecha~coefficients and developing \mecha~mechanisms that are compatible with multi-view graphs.

%\newpage

\bibliographystyle{ACM-Reference-Format}
\bibliography{references}

\clearpage

\section*{Appendices}
\appendix
%\section{Remarks on the expressive power of \mechanism~layers} %\label{app:remarks_of_theory}
%In this paper, we mainly verify that the proposed Graph \mechanismlowercase~networks are most powerful message-passing GNNs under the condition that the feature space is countable \cite{he2021learning,DBLP:conf/iclr/XuHLJ19,zhang2020improving}.
%Recent studies have shown that a simplex operator for feature aggregations in some GNN layer is injective, i.e., the 1-WL test equivalent in the countable feature space \cite{corso2020principal}.
%But such injectivity might not hold for the simplex operator when it operates in the uncountable feature space.
%To ensure the injectivity when a GNN deals with uncountable features, diverse forms of operators, e.g., mean, max, and min operators are required to collaboratively aggregate neighbor features.
%Thus, the expressive power of the proposed \model s can be equivalent to the 1-WL test in uncountable feature space by appropriately integrating with other effective operators for feature aggregation.

\section{Irrelevance between connected nodes}\label{conflict-example}

In this section, we show a large amount of neighbors are found to highly differ regarding either node features or graph structure. The corresponding results obtained from Cora and Cite are exemplified here to demonstrate such phenomenon (Fig. \ref{fig:conflict_feature} and \ref{fig:conflict_struct}).

\begin{figure}[htbp]
    \centering
    \vspace{-4px}
    \includegraphics[width=0.45\linewidth]{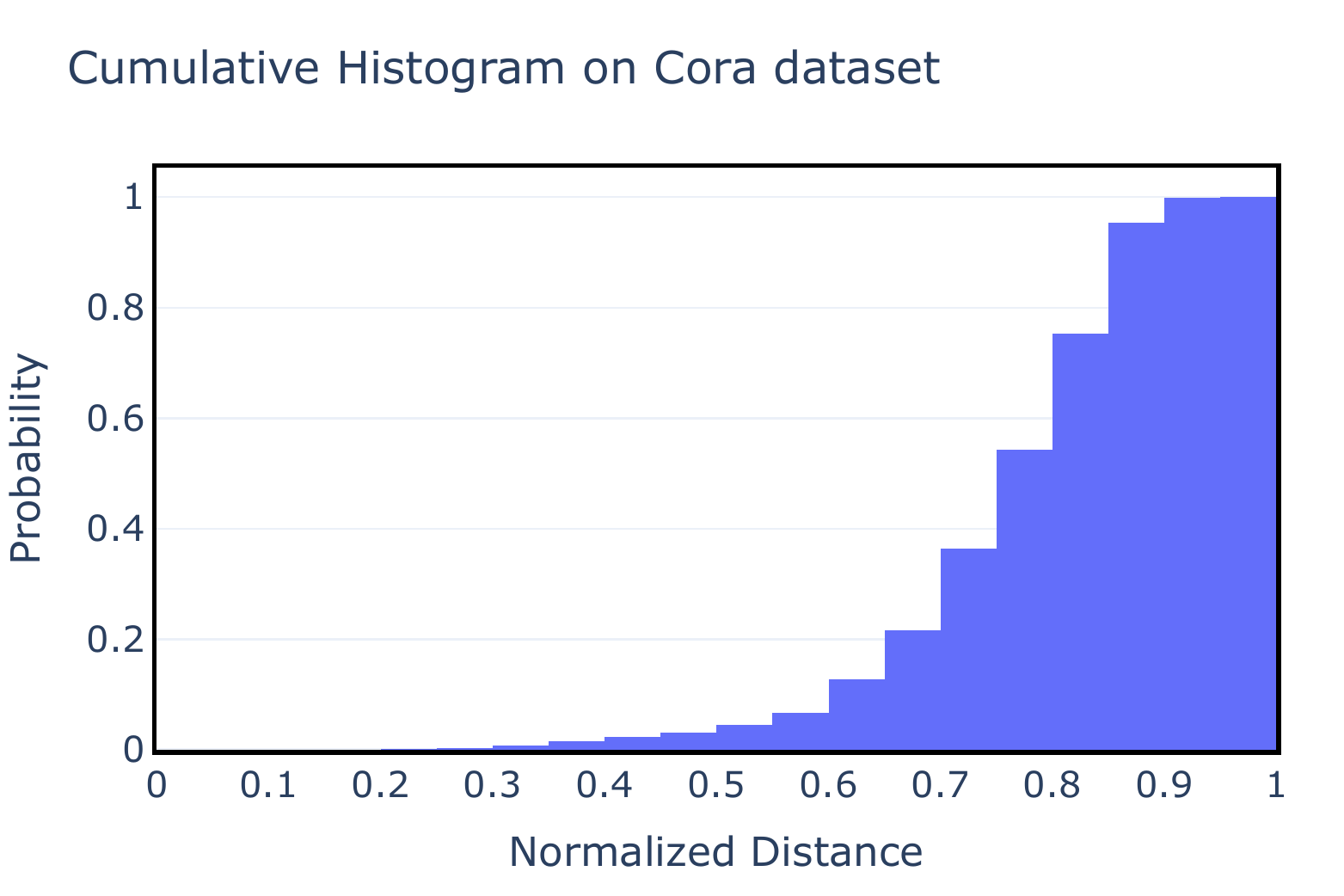}
    \includegraphics[width=0.45\linewidth]{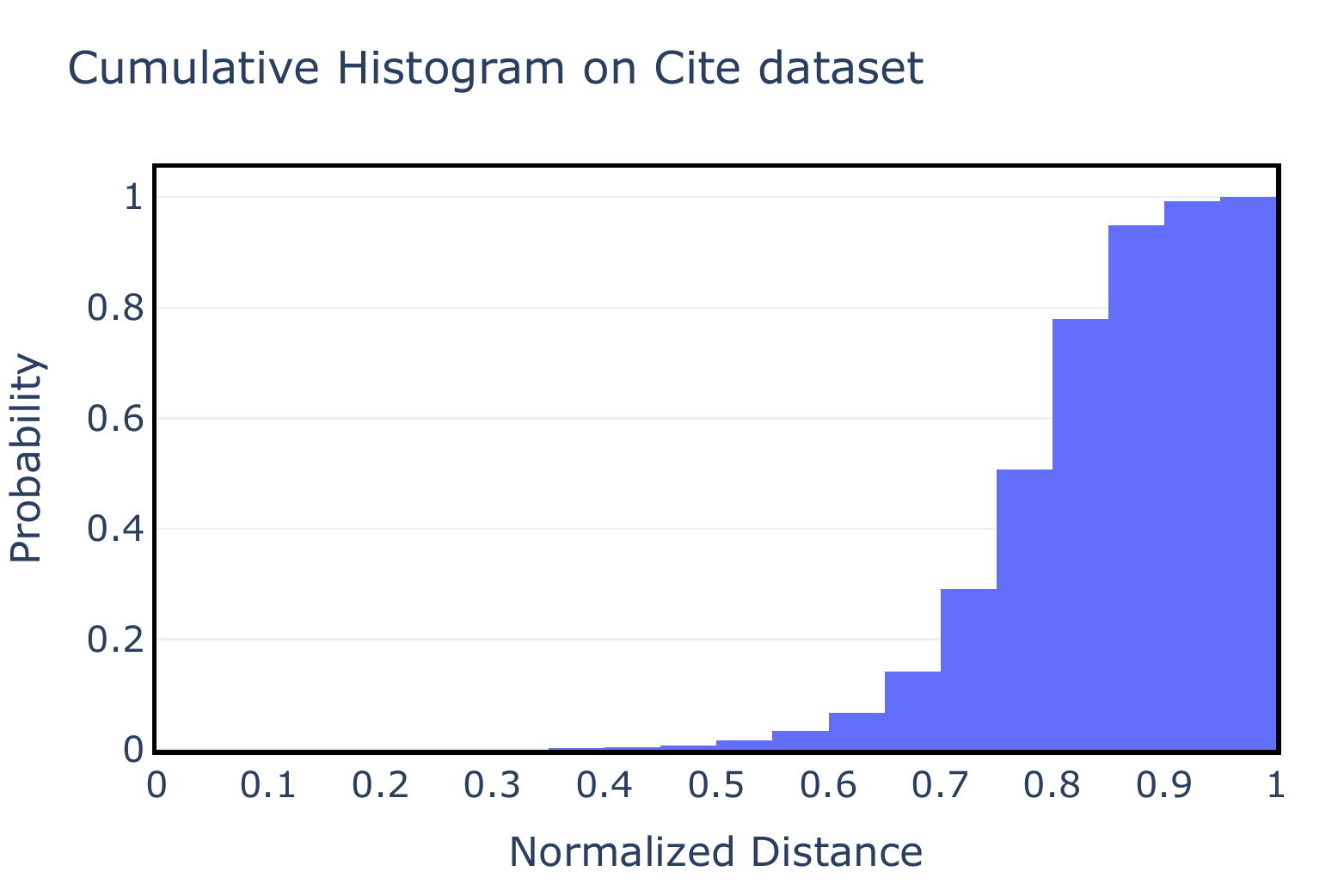}
    % \includegraphics[width=0.45\textwidth]{images/conflict_pubmed_feat.pdf}
    % \includegraphics[width=0.45\textwidth]{images/conflict_wiki_feat.pdf}
    % \includegraphics[width=0.45\textwidth]{images/conflict_uai_feat.pdf}
    % \includegraphics[width=0.45\textwidth]{images/conflict_coauthorcs_feat.pdf}
    % \caption{Histograms and cumulative histograms of normalized Euclidean distance regarding node features. \textbf{Large} values mean connected node pairs have very different features. Here most distances are large.}
    \vspace{-6px}
    \caption{Cumulative histograms of normalized Euclidean distance regarding node features. \textbf{Large} values mean connected node pairs have very different features. Here most distances are large.}
    \vspace{-8px}
    \label{fig:conflict_feature}
\end{figure}

\begin{figure}[htbp]
    \centering
    \vspace{-4px}
    \includegraphics[width=0.45\linewidth]{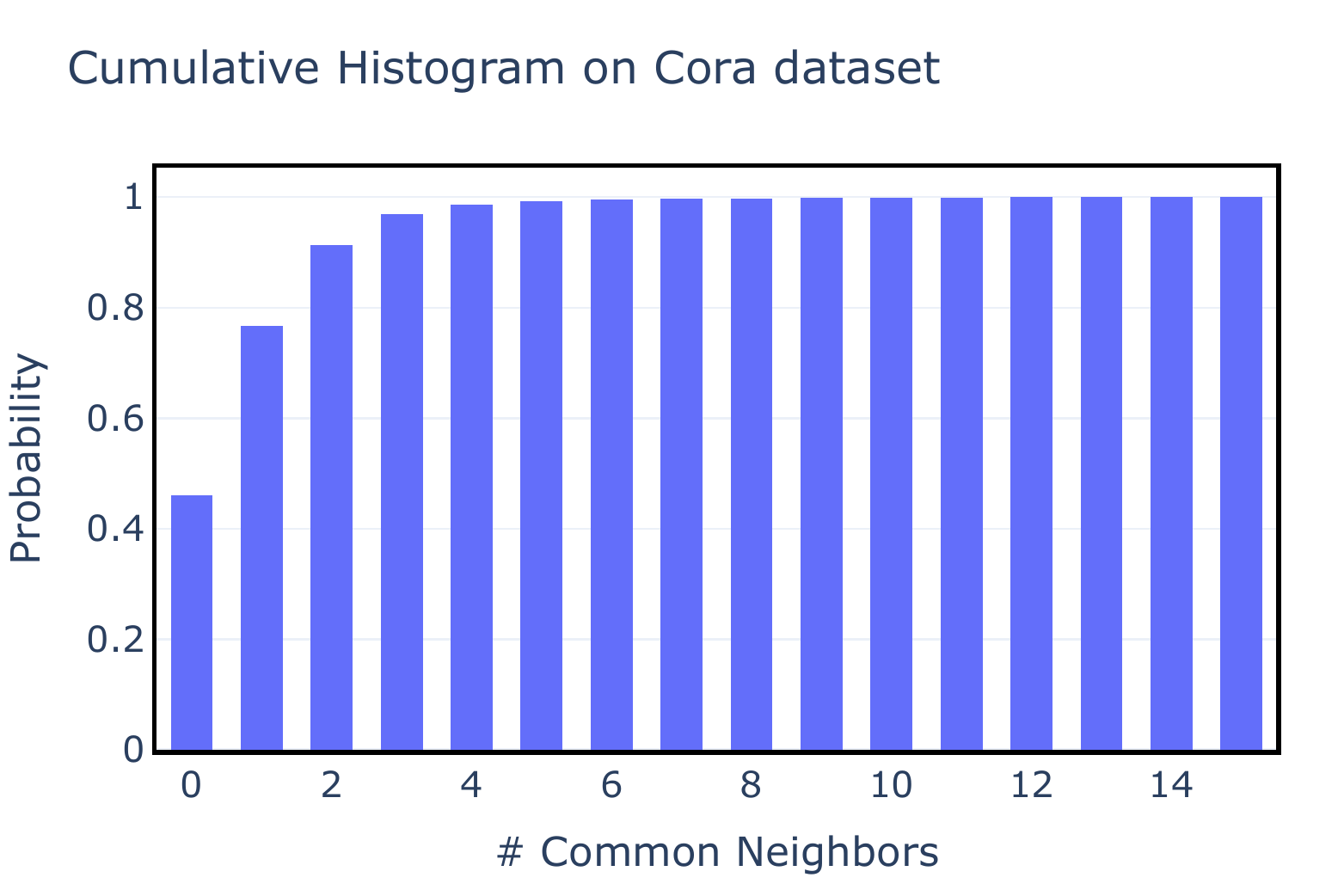}
    \includegraphics[width=0.45\linewidth]{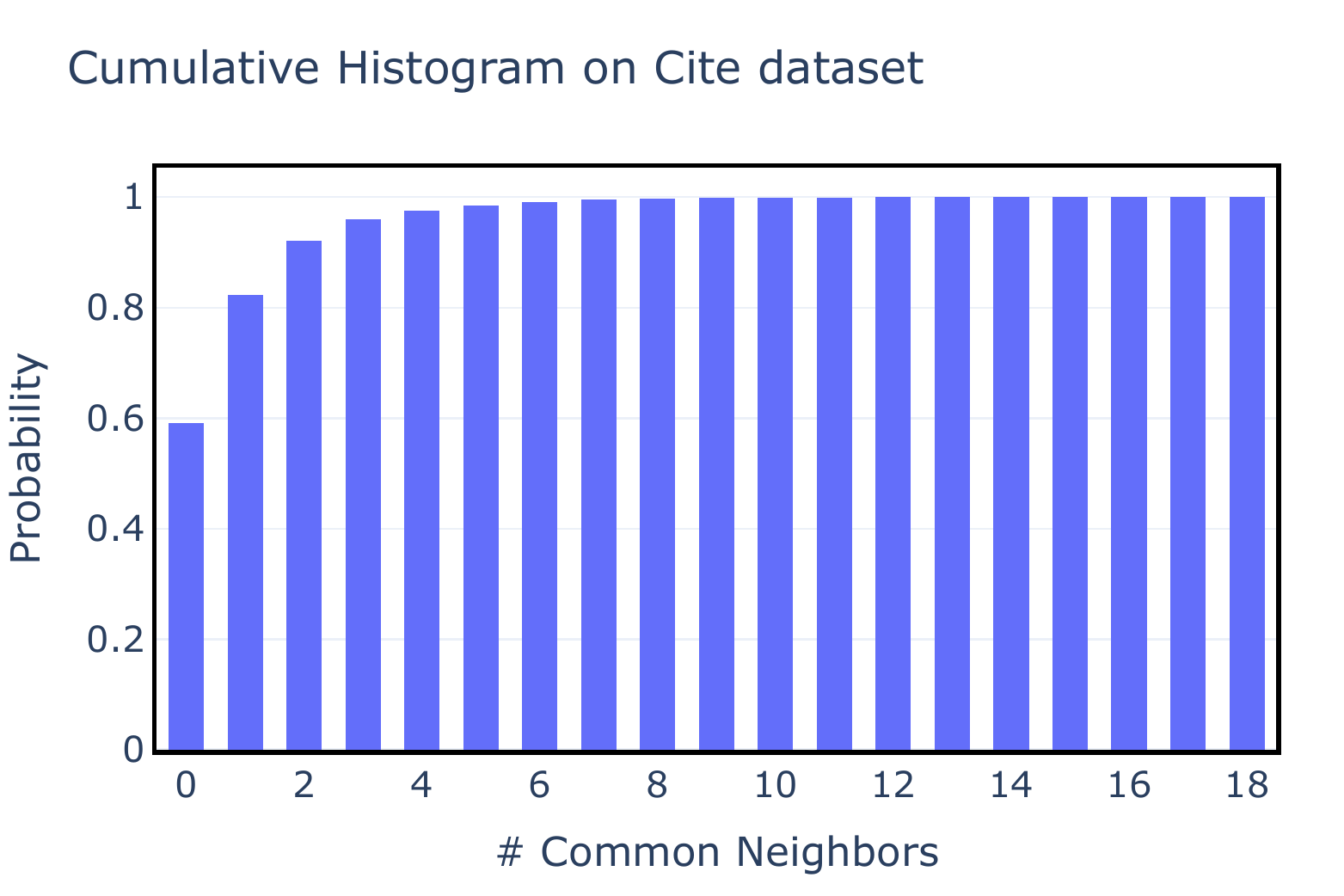}
    % \includegraphics[width=0.45\textwidth]{images/conflict_pubmed_struct.pdf}
    % \includegraphics[width=0.45\textwidth]{images/conflict_wiki_struct.pdf}
    % \includegraphics[width=0.45\textwidth]{images/conflict_uai_struct.pdf}
    % \includegraphics[width=0.45\textwidth]{images/conflict_coauthorcs_struct.pdf}
    % \caption{Histograms and cumulative histograms of common neighbors between connected node pairs. \textbf{Small} values mean connected node pairs differ in terms of graph structure. Here most values are 0s or close to 0.}
    \vspace{-6px}
    \caption{Cumulative histograms of common neighbors between connected node pairs. \textbf{Small} values mean connected node pairs differ in terms of graph structure. Here most values are 0s or close to 0.}
    \vspace{-2px}
    \label{fig:conflict_struct}
\end{figure}

% In Fig. \ref{fig:conflict_feature}, the cumulative histogram of the normalized Euclidean distance regarding node features and the distribution of such distance are depicted.
The difference between the node features of connected nodes is measured by normalized Euclidean distance, with Fig. \ref{fig:conflict_feature} depicting its cumulative histograms.
It is evident that the distances of most node pairs in both datasets are very close to the largest.
Specifically, the normalized distance of more than 70\% of connected nodes in both datasets is higher than 0.7.
% This indicates that most node features associated with each node are quite different from those of its neighbors. Thus, these connected nodes in the graph are highly irrelevant regarding node features.
This indicates that connected nodes have quite different features from their neighbors, and thus are highly irrelevant regarding node features.

The number of common neighbors is applied to measure the difference in the structure of connected nodes, with Fig. \ref{fig:conflict_struct} depicting its cumulative histograms.
The number of common neighbors is a widely used method to measure the similarity of two connected nodes \cite{DBLP:conf/colt/SarkarCM10,papadopoulos2015network}. It is observed that around 50\% of connected node pairs do NOT have even one common neighbor in Cora and Cite, and most connected node pairs (about 90\%) have very few (0, 1 or 2) common neighbors. 
Thus, the structure of connected node pairs is highly different, indicating these nodes are very irrelevant.
% Considering the total number of connected node pairs (i.e., total degrees) in these two datasets are 10,858 and 9104, the structure of more than 70\% of neighbors is highly irrelevant.

%In this section, we show the irrelevance between connected node pairs in Cora and Citeseer datasets. The irrelevance of node features is measured by the Euclidean distance between input features of two nodes, i.e. $\| \mathbf{X}_i - \mathbf{X}_j \|_2$ for node $i$ and node $j$.

% Given such observations, we believe that a large portion of neighbors are irrelevant to the central node.
Such observations indicate that a large portion of neighbors are irrelevant to the central node.
These neighbors are possibly less informative for feature aggregation.
Thus, they should be ignored by some appropriate attention mechanisms.
Motivated by the discovered phenomenon that most neighbors are highly irrelevant, we propose \mechanism~to endow graph neural networks with the capability of flexibly ignoring irrelevant neighbors that can be identified by diverse forms of node-node dissimilarity.
% The corresponding experimental results also demonstrate the proposed \model s indeed can learn to ignore those unimportant neighbors and their performances outperform all compared baselines on most testing datasets.

\section{More details on experiments}  \label{app:experiments}

\subsection{Data} \label{app:data_description}
\paragraph{Data description} 
The dataset statistics can be found in Table \ref{tab:dataset}. Cora, Cite, and Pubmed are citation networks, where nodes, edges, and node features respectively represent the documents, document-document citations, and the keywords of the documents. Wiki and Uai are two web networks, where nodes, edges and node features represent web pages, web-web hyperlinks and descriptive information on these web pages, respectively. 
CoauthorCS is a co-authorship graph, where nodes are authors, which are connected by an edge if they co-author a paper, node features represent keywords for each author’s papers, and class
labels indicate the most active fields of study for each author.
For Cora, Cite and Pubmed, we follow previous works \cite{DBLP:conf/iclr/KipfW17, DBLP:conf/iclr/VelickovicCCRLB18} to use 20 nodes from each label as the training set, 500 nodes as the validation set, and 1000 nodes as the test set.
For Wiki, Uai, and CoauthorCS, as there are no established splits for training and testing, we randomly generate five sets of splits.
In each of them, 20 nodes of each ground truth class are sampled for training, 500 nodes are sampled for validation, and 1000 nodes are done for testing.
For each of them, we randomly generate five sets of splits (i.e., training, validation, and testing splits) for the classification tasks, and use all nodes in each dataset for clustering tasks.
%we randomly generate the corresponding splits, where 20 nodes of each ground truth class are sampled for training, 500 nodes are sampled for validation, and 1000 nodes are done for testing. 

\paragraph{Pre-processing} From Table \ref{tab:dataset}, we find the dimension of input features are too high in Wiki, Uai and CoauthorCS, which may cause unstable training performances. 
Thus we apply a trainable linear layer without non-linearity activation to reduce the feature dimensions. After the linear layer, the dimension of the input to the GNNs are reduced to 512. This pre-processing step is applied for both our models and all the baselines.

\subsection{Detailed settings of all approaches}\label{detail-setting}
%\paragraph{Settings of graph neural networks}
\model s are compared with twelve strong graph neural networks, including MoNet \cite{DBLP:conf/cvpr/MontiBMRSB17}, GCN \cite{DBLP:conf/iclr/KipfW17}, GraphSage \cite{DBLP:conf/nips/HamiltonYL17}, JKNet \cite{DBLP:conf/icml/XuLTSKJ18}, APPNP \cite{DBLP:conf/iclr/KlicperaBG19}, ARMA \cite{bianchi2021graph}, GIN \cite{DBLP:conf/iclr/XuHLJ19}, Neural Sparse \cite{DBLP:conf/icml/ZhengZCSNYC020}, GAT \cite{DBLP:conf/iclr/VelickovicCCRLB18}, GATv2 \cite{brody2022how}, CAT \cite{he2021learning}, and HardGAT \cite{DBLP:conf/kdd/GaoJ19}.
%Besides, we also construct a variant of GAT (GAT-$k$-Lap), which uses the concatenation of original node features and the first $k$ eigenvectors of Laplacian graph \cite{DBLP:conf/wsdm/QiuDMLWT18}.
%GAT-$k$-Lap can be seen as a variant of GAT that is improved by graph structure.

To perform unbiased comparisons, the source codes released by the authors are used to implement all the mentioned baselines. In our experiments, all the baselines use a two-layer network structure to learn node representations, meaning that the output layer of each GNN is followed by one hidden layer.
As for the tuning of each baseline, we mainly follow the configurations presented in \cite{he2021learning,DBLP:conf/iclr/KipfW17,DBLP:conf/iclr/VelickovicCCRLB18}.
%The detailed configurations for attention based methods have been listed as follows.
% in Table \ref{exp-setting}, where ``lr'' and ``hidden'' respectively represent the learning rate of each epoch and the dimension of the hidden layer.
The configurations of the proposed \model s are generally same to those of GAT.
Specifically, 8 attention heads are used in hidden layers, while the number of hidden layer dimension (for one head) is 8 for Cora, Cite, and Pubmed, and 32 for Wiki, Uai, and CoauthorCS.
For the output layer, 1 attention head is used. 
``\verb|LeakyReLU|'' is used as the non-linearity in the attention mechanism with the negative slope as 0.2.
All attention-based GNNs are trained with learning rate as 0.005, weight decay as 0.0005, number of training epochs as 1000 and dropout ratio as 0.6. 
For \textit{Contractive apprehension span}, $\beta = 1.0$, while for \textit{Subtractive apprehension span}, $\beta = 0.5$. 
% The initialization of all GNNs is based on Glorot initialization, and all GNNs are trained to minimize the cross-entropy loss of the training nodes using Adam optimizer \cite{kingma2014adam}.
All GNNs are initialized with Glorot initialization \cite{DBLP:journals/jmlr/GlorotB10}, and all GNNs are trained to minimize the cross-entropy loss of the training nodes using Adam optimizer \cite{kingma2014adam}.
All the experiments are conducted on an NVIDIA RTX 3090 graphics card. The software environment of the experiments is CUDA 11.1, Python 3.8, and PyTorch 1.8.1.

\subsection{Visualization of attention scores}  \label{app:attn_coef}

We show the histograms of attention scores from the output layers of GAT, CAT and \model s in Fig. \ref{fig:cora_hist}-\ref{fig:coauthorcs_hist}.
It is not surprising that the attention coefficients learned by different variants of \model~are generally more concentrated.
%It is not surprising that the \model~ does not have much more small attention values.
We observe that \model~is able to learn ignoring on Cora, Cite, Wiki and Uai datasets.
\model~ignores fewer neighbors in citation networks, as Cora, Cite, and Pubmed are three very sparse graphs, which means there are few nodes having many neighbors, so \model~does not have enough training samples to learn ignoring. 
However, things are different on Wiki and Uai datasets, where there are much more very small attention values, which indicates \model~indeed learns to ignore irrelevant neighbors.
It is obvious that learning-to-ignore is indeed the reason that \model s achieve state-of-the-art performances on all the testing datasets.

\begin{figure}[htbp]
    \centering
    \subfigure{
        \includegraphics[height=0.15\textheight]{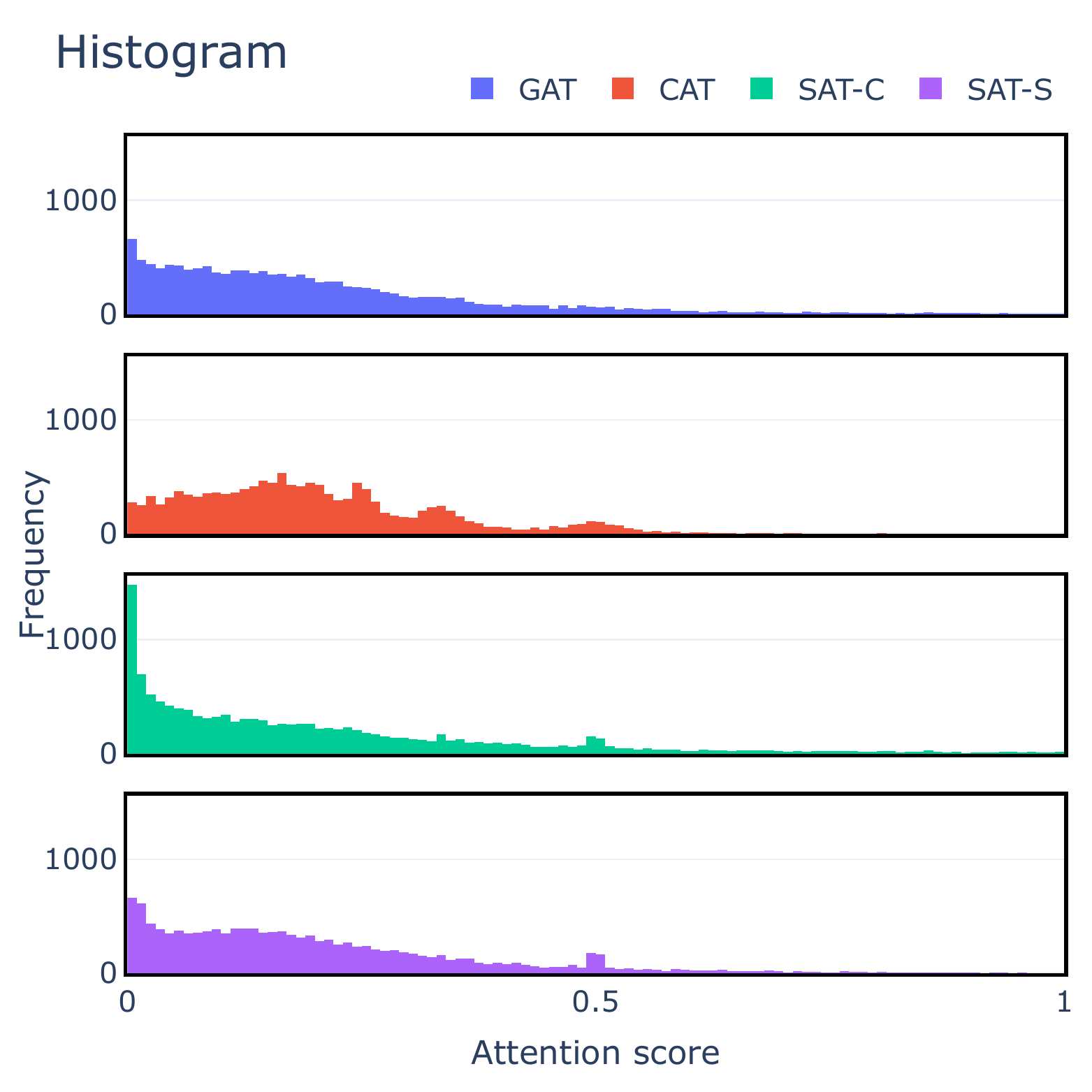}
    }
    \subfigure{
        \includegraphics[height=0.15\textheight]{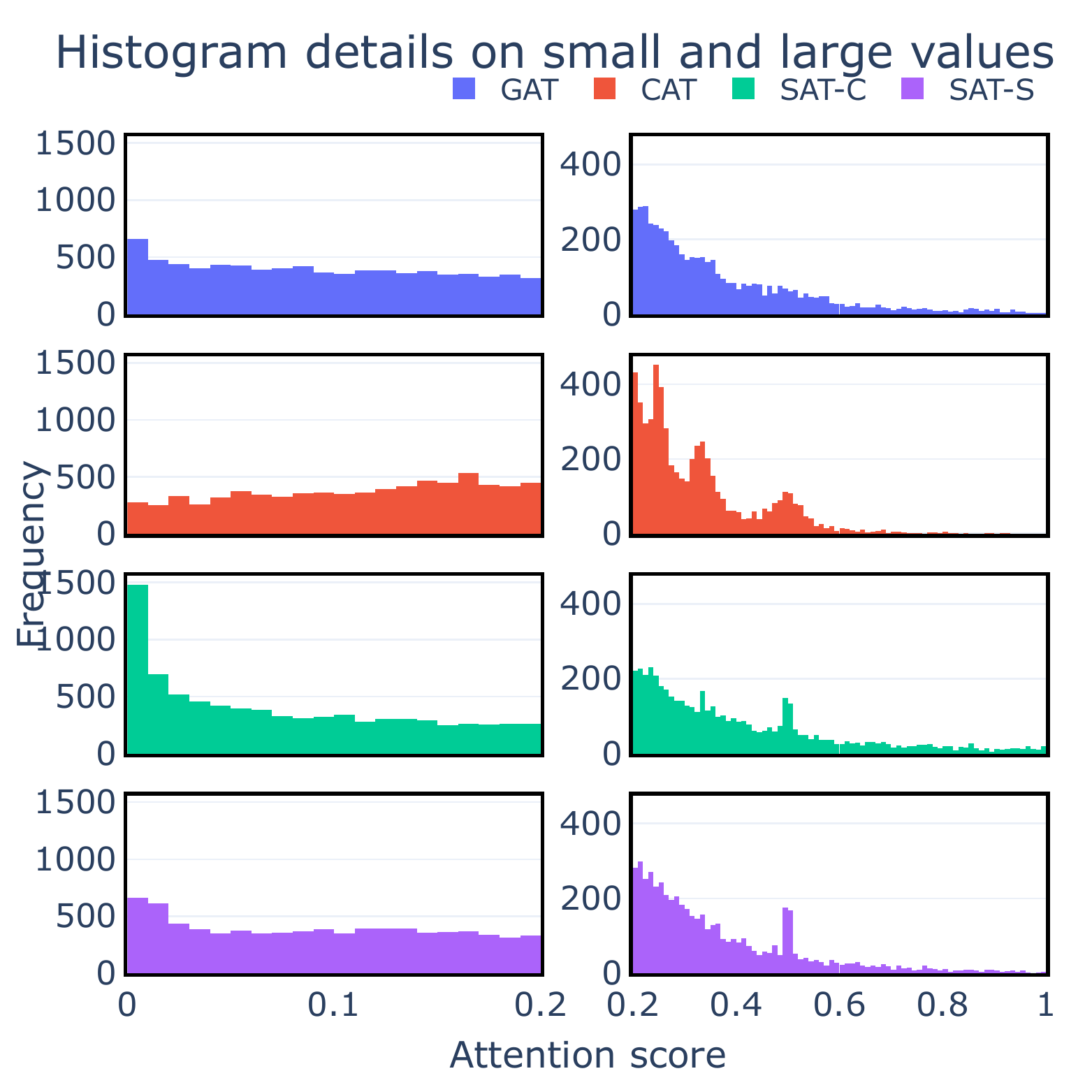}
    }
    \vspace{-8px}
    \caption{Attention scores from GAT, CAT and \model~on Cora}
    \vspace{-2px}
    \label{fig:cora_hist}
\end{figure}

\begin{figure}[htbp]
    \centering
    \subfigure{
        \includegraphics[height=0.15\textheight]{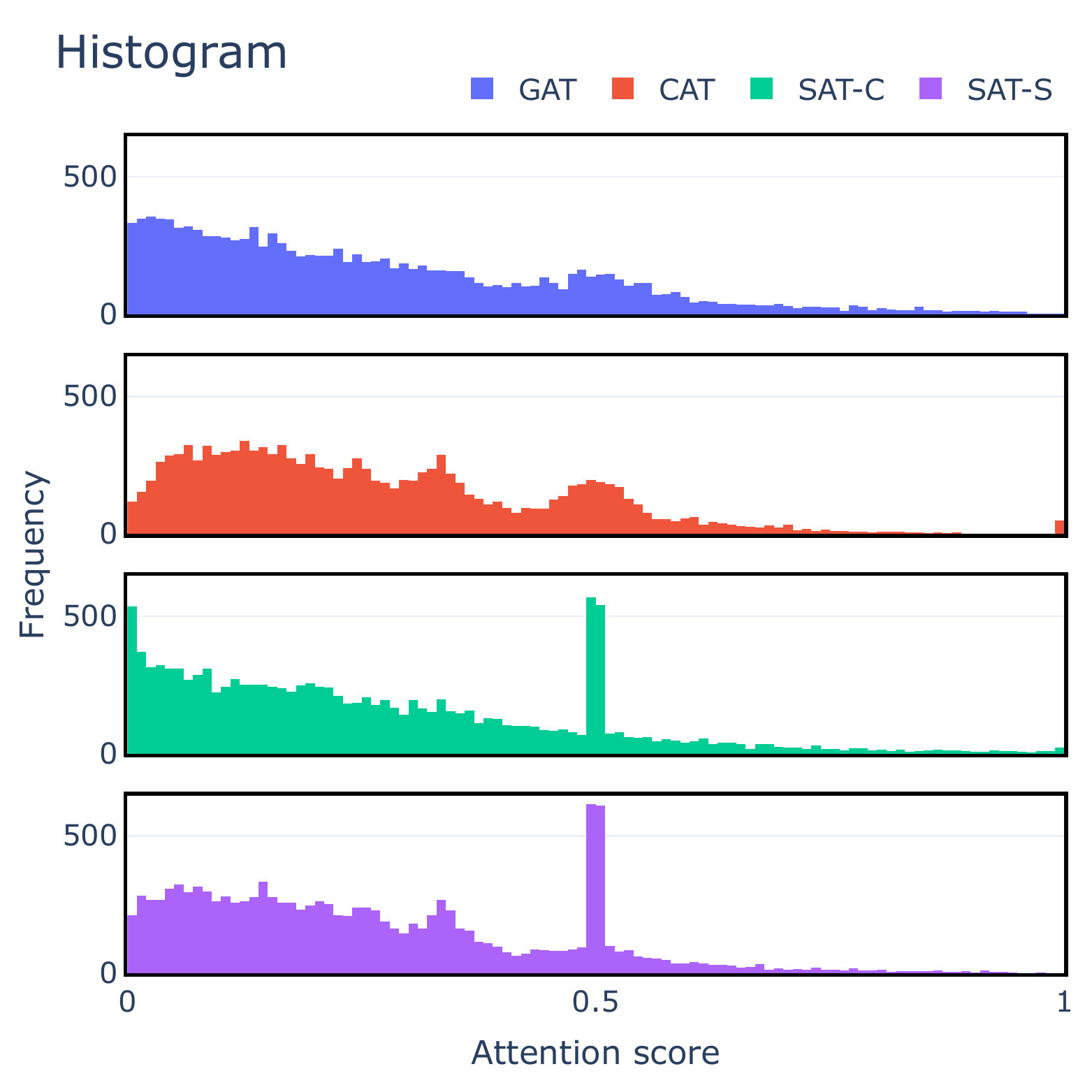}
    }
    \subfigure{
        \includegraphics[height=0.15\textheight]{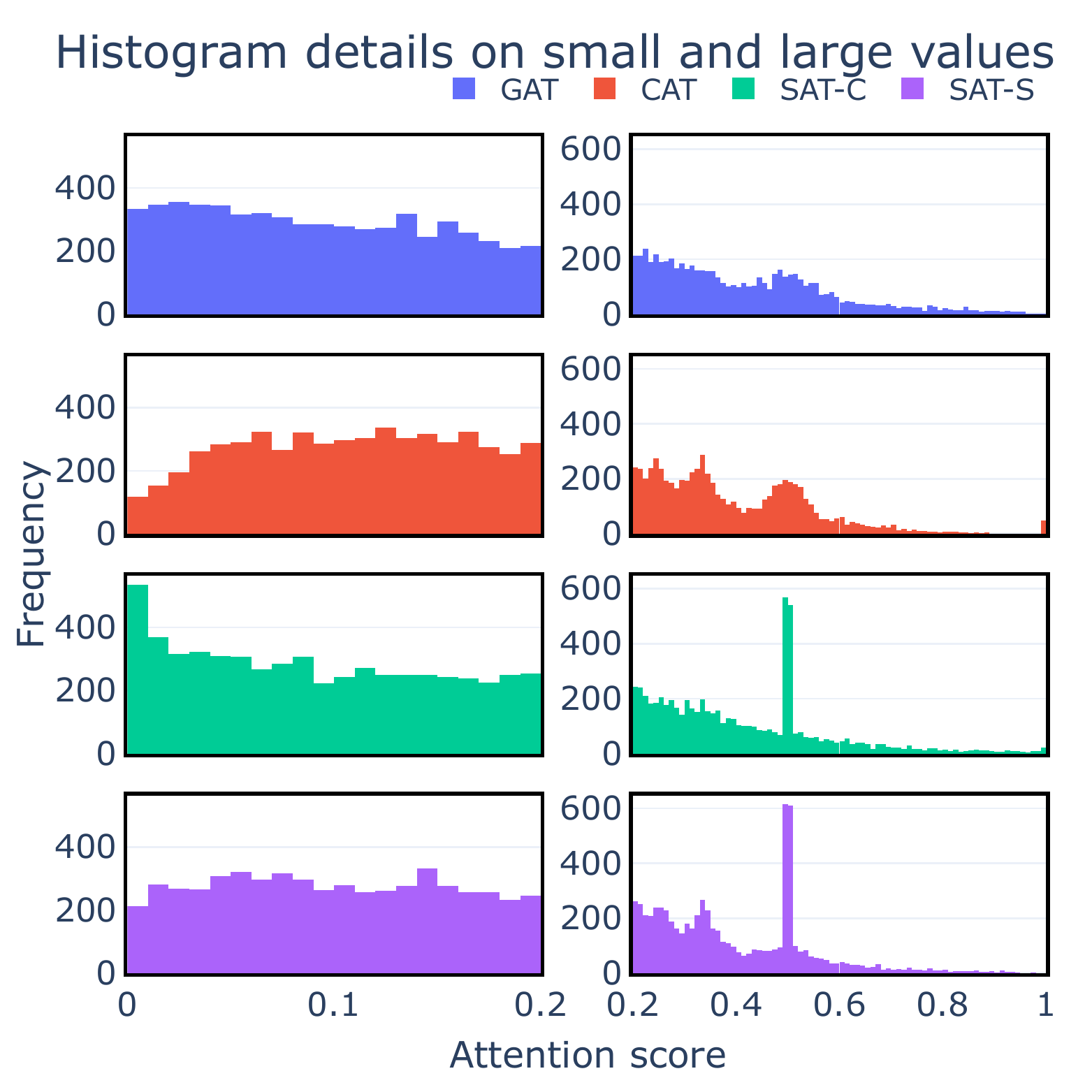}
    }
    \vspace{-10px}
    \caption{Attention scores from GAT, CAT and \model~on Cite}
    \vspace{-4px}
    \label{fig:cite_hist}
\end{figure}

\begin{figure}[htbp]
    \centering
    \subfigure{
        \includegraphics[height=0.15\textheight]{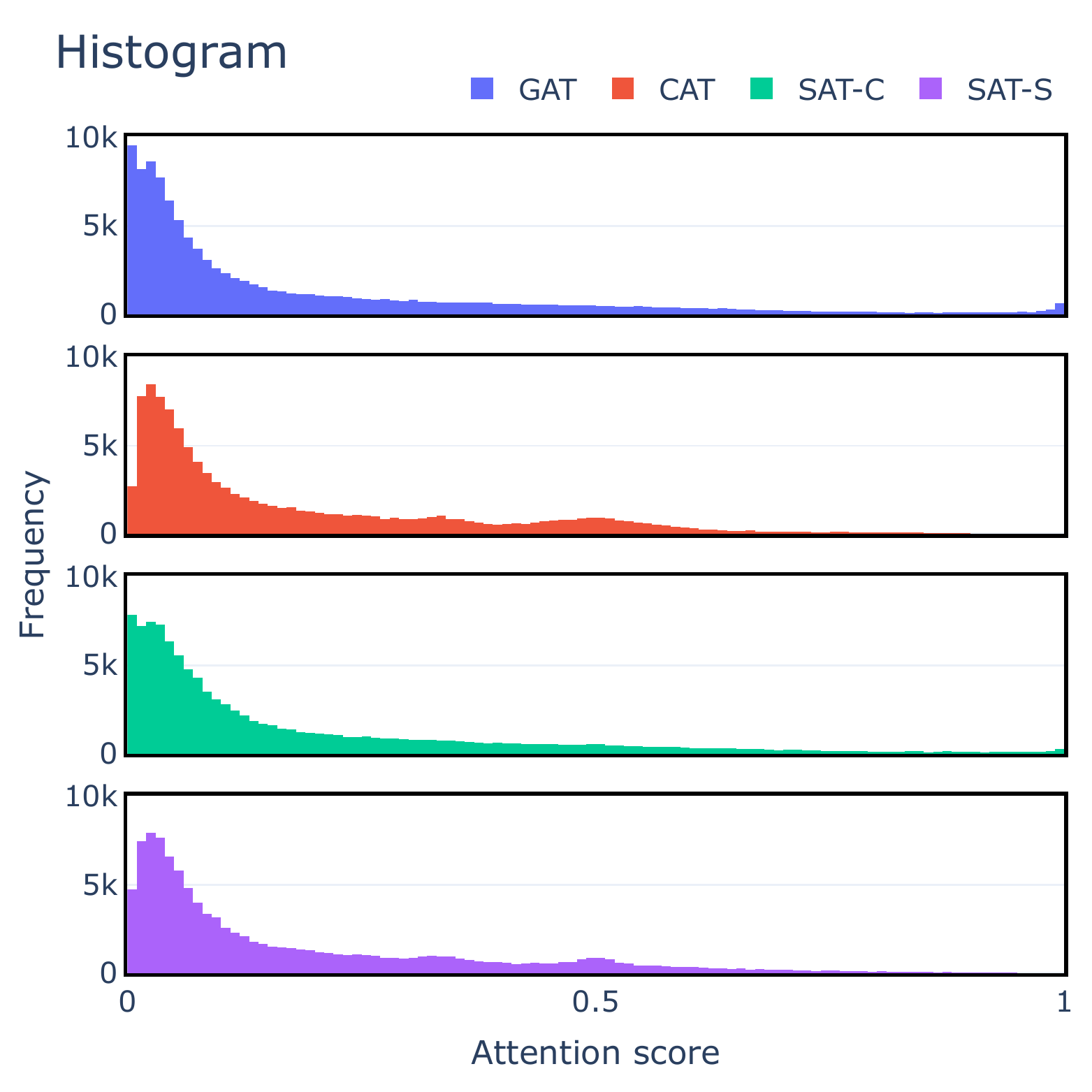}
    }
    \subfigure{
        \includegraphics[height=0.15\textheight]{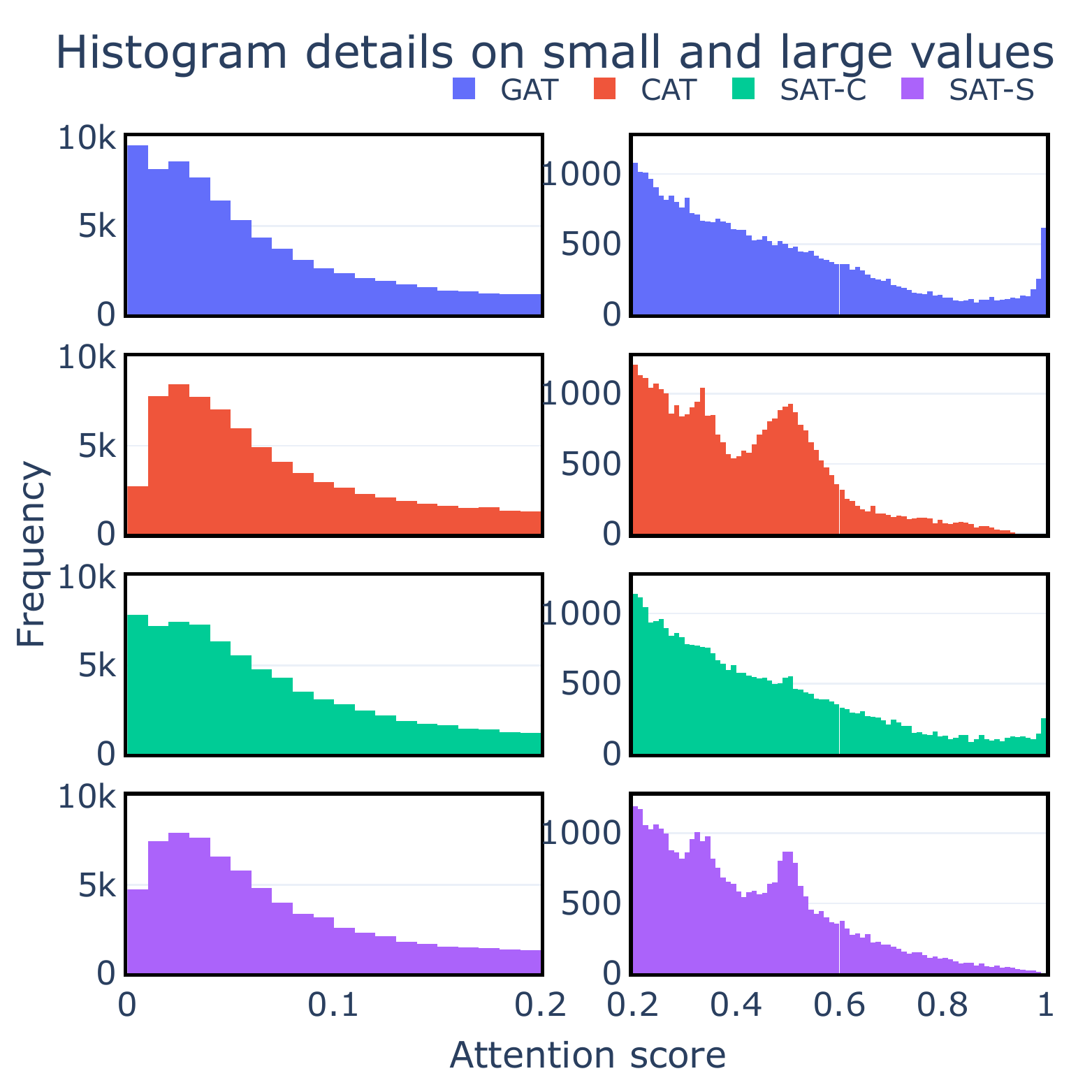}
    }
    \vspace{-8px}
    \caption{Attention scores from GAT, CAT and \model~on Pubmed}
    \vspace{-4px}
    \label{fig:pubmed_hist}
\end{figure}

\begin{figure}[htbp]
    \centering
    \subfigure{
        \includegraphics[height=0.15\textheight]{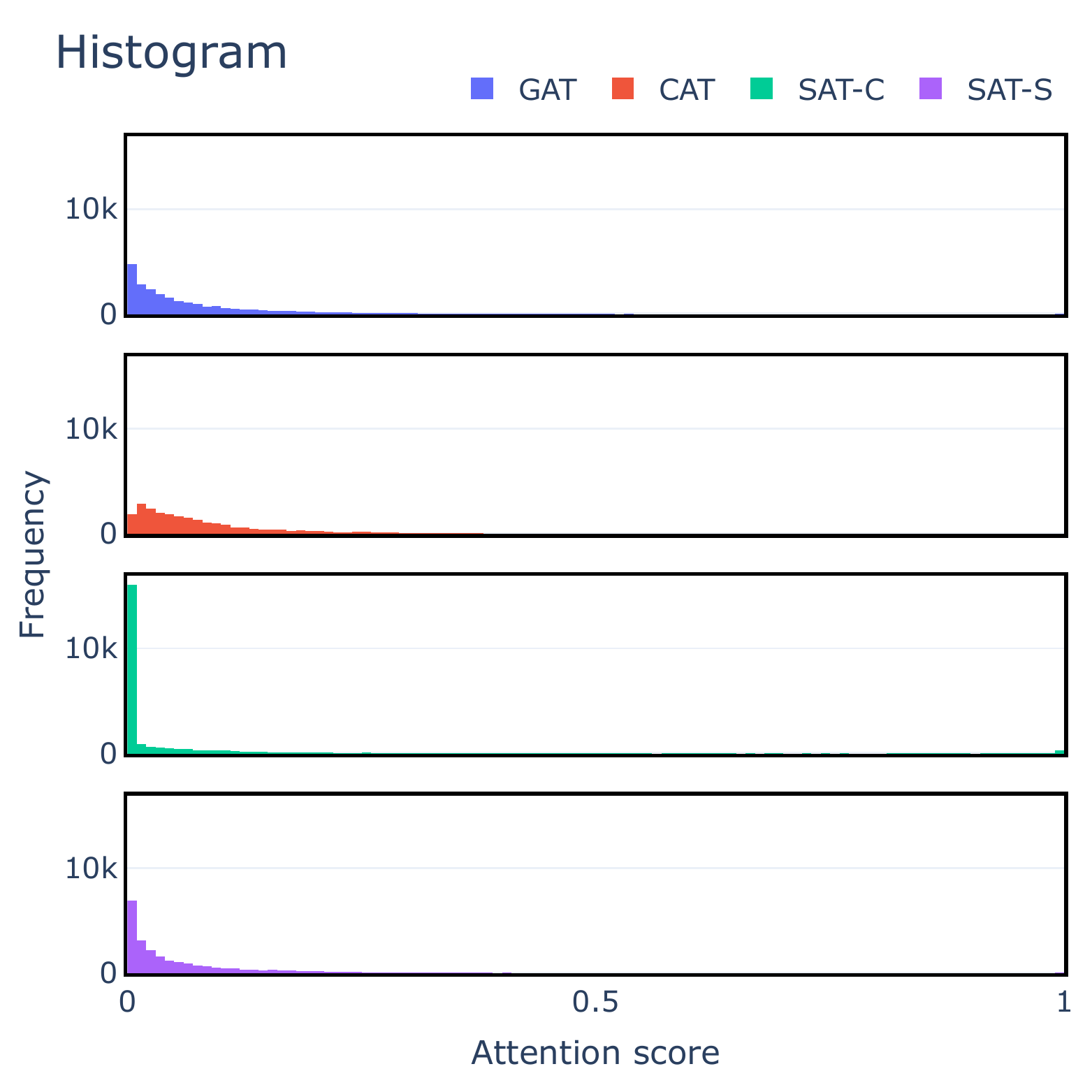}
    }
    \subfigure{
        \includegraphics[height=0.15\textheight]{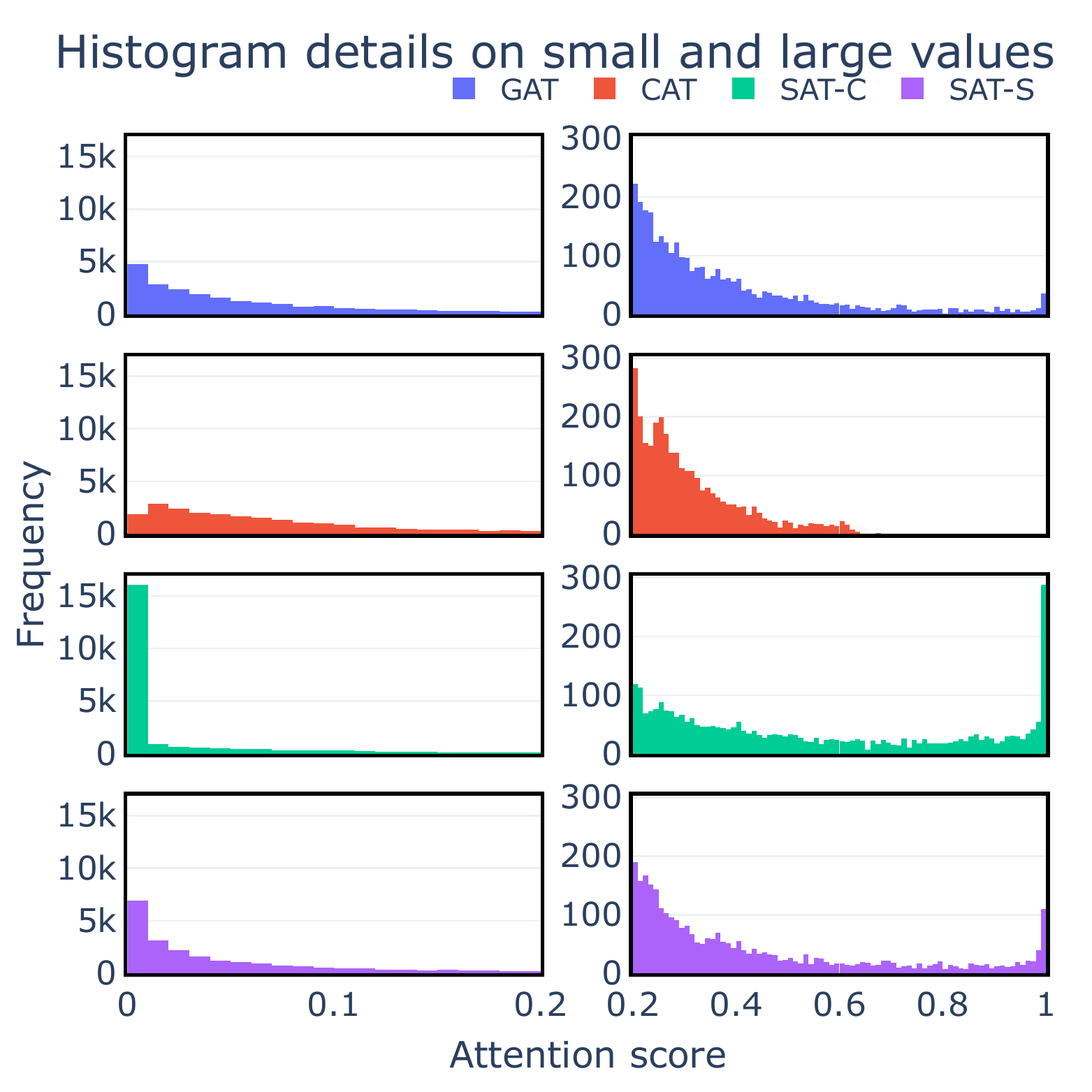}
    }
    \vspace{-10px}
    \caption{Attention scores from GAT, CAT and \model~on Wiki}
    \vspace{-4px}
    \label{fig:wiki_hist}
\end{figure}

\begin{figure}[htbp]
    \centering
    \subfigure{
        \includegraphics[height=0.15\textheight]{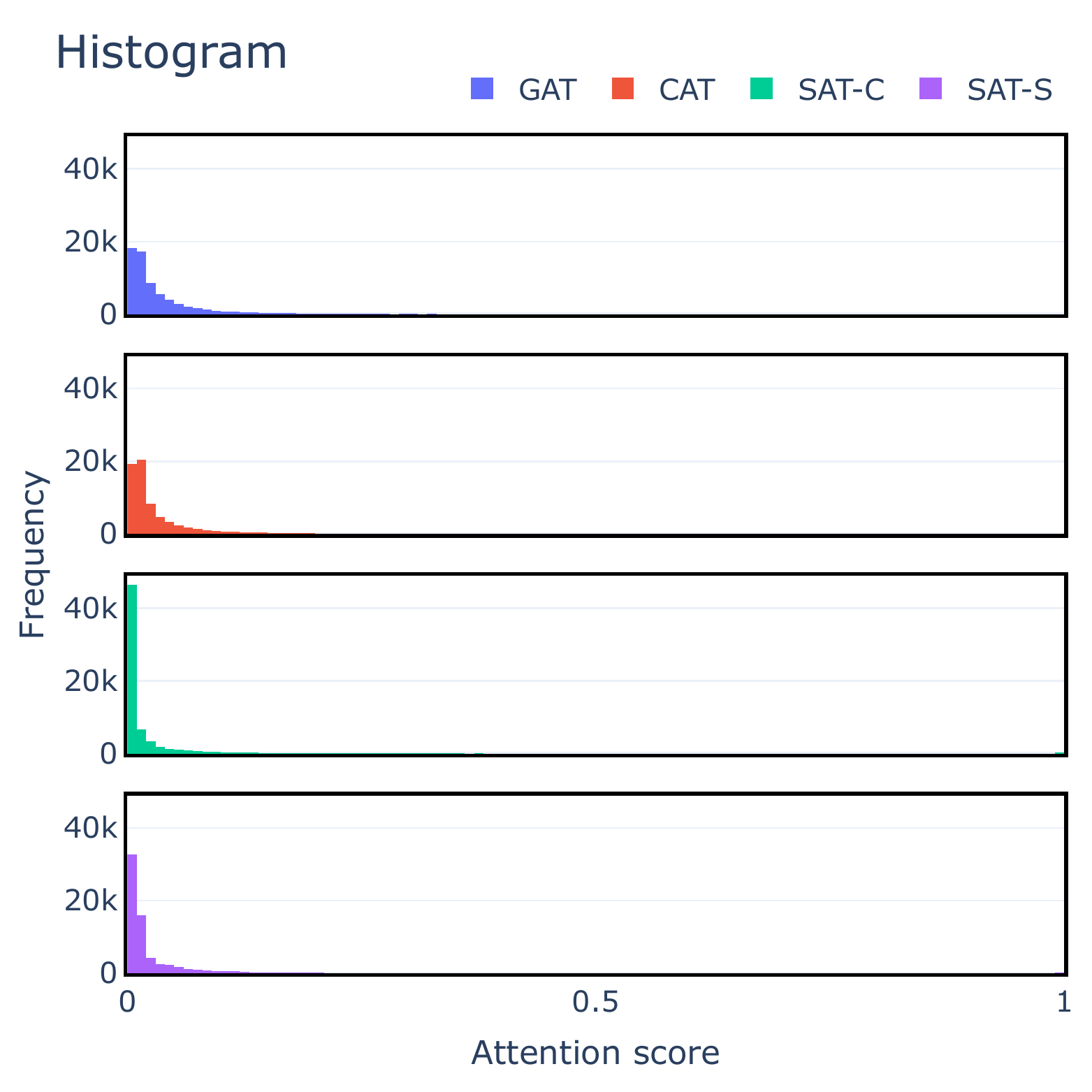}
    }
    \subfigure{
        \includegraphics[height=0.15\textheight]{images/uai_hist_detail_m.pdf}
    }
    \vspace{-10px}
    \caption{Attention scores from GAT, CAT and \model~on Uai}
    \vspace{-4px}
    \label{fig:uai_hist}
\end{figure}

\begin{figure}[htbp]
    \centering
    \subfigure{
        \includegraphics[height=0.15\textheight]{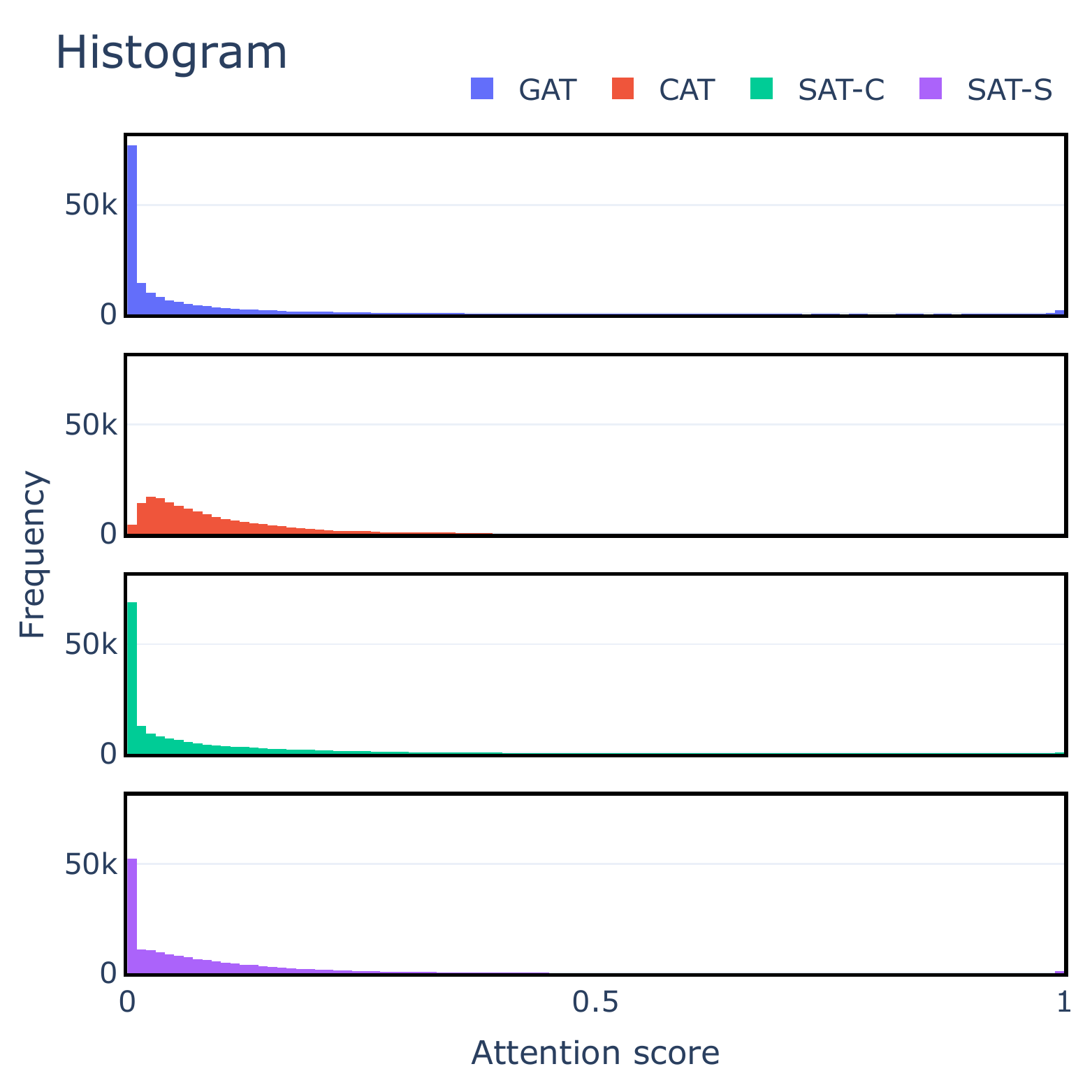}
    }
    \subfigure{
        \includegraphics[height=0.15\textheight]{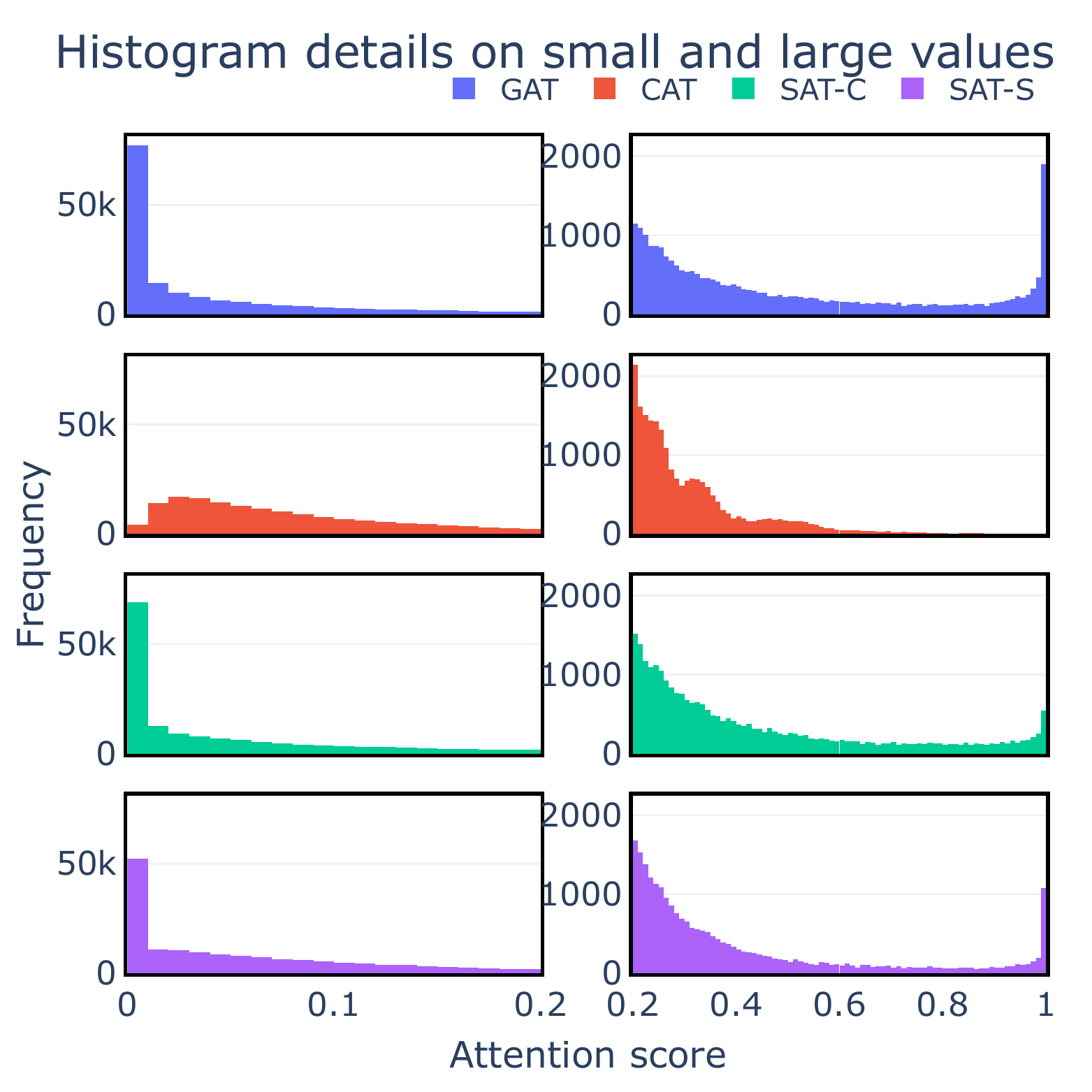}
    }
    \vspace{-10px}
    \caption{Attention scores from GAT, CAT and \model~on CoauthorCS}
    
    \label{fig:coauthorcs_hist}
\end{figure}

\newpage

\clearpage
\section*{Supplementary materials}
\section*{Proofs of Expressive Power}

\subsection*{Proof of Theorem \ref{theorem-c}}\label{proof-c}
\begin{proof}
The proof of Theorem \ref{theorem-c} can be divided into two parts, i.e., the proof of the sufficiency and necessity of the iff conditions \cite{he2021learning,DBLP:conf/iclr/XuHLJ19,zhang2020improving}.
The sufficiency of the iff conditions stated in Theorem \ref{theorem-c} is firstly proved.
Given a central node $c_i$, its aggregation function $h(c_i, X_i)$ can be written as:
%Given the definition of $h(c, X)$, we have:
\begin{equation}\label{hx-c}
    \begin{aligned}
    &h(c_i, X_i) = \sum_{x\in X_i} \alpha_{c_ix}g(x),\alpha_{c_ix} = \frac{f_{c_ix}\cdot d_{c_ix}}{ \sum_{x\in X_i} f_{c_ix}\cdot d_{c_ix}},\\
	&f_{c_ix} = \frac{\exp{ (m_{c_ix})}}{\sum_{x\in X_i}\exp{ (m_{c_ix})}}, d_{c_ix} = \exp{(-\beta \mathbf S_{c_ix})},
    \end{aligned}
\end{equation}
where $m_{c_ix}$ is the feature correlations between $c_i$ and the neighbor having vectorized features $x$.
Given Eq. (\ref{hx-c}), for two central nodes $c_1$ and $c_2$, $h(c_1, X_1)$ and $h(c_2, X_2)$ can be written as:
\begin{equation}\label{hx1hx2-c}
    \begin{aligned}
    &h(c_1, X_1) = \sum_{x\in X_1}\alpha_{c_1x}g(x)=\sum_{x\in X_1}[\frac{f_{c_1x}\cdot d_{c_1x}}{ \sum_{x\in X_1} f_{c_1x}\cdot d_{c_1x}}]\cdot g(x)\\
	&h(c_2, X_2) = \sum_{x\in X_2}\alpha_{c_2x}g(x)=\sum_{x\in X_2}[\frac{f_{c_2x}\cdot d_{c_2x}}{ \sum_{x\in X_2} f_{c_2x}\cdot d_{c_2x}}]\cdot g(x)
    \end{aligned}
\end{equation}
If $\mathsf{M}_1 = \mathsf{M}_2$, $h(c_2, X_2)$ can be written as follows:
\begin{equation}\label{hx2-rew1}
    \begin{aligned}
    &h(c_2, X_2) = \sum_{x\in \mathsf{M}_2}[\frac{f_{c_2x}\cdot \sum_{y=x, y\in X_2}d_{c_2y}}{ \sum_{x\in \mathsf{M}_2} f_{c_2x}\cdot\sum_{y=x, y\in X_2}d_{c_2y}}]\cdot g(x)\\
	&=\sum_{x\in \mathsf{M}_2}[\frac{\frac{\exp{ (m_{c_2x})}}{\sum_{x\in X_2}\exp{ (m_{c_2x})}}\cdot \sum_{y=x, y\in X_2}d_{c_2y}}{ \sum_{x\in \mathsf{M}_2} \frac{\exp{ (m_{c_2x})}}{\sum_{x\in X_2}\exp{ (m_{c_2x})}}\sum_{y=x, y\in X_2}d_{c_2y}}]\cdot g(x)\\
	&=\sum_{x\in \mathsf{M}_2}\frac{\exp{ (m_{c_2x})}\sum_{y=x, y\in X_2}\exp{ (-\beta\mathbf S_{c_2y})}}{\sum_{x\in \mathsf{M}_2}\exp{ (m_{c_2x})}\sum_{y=x, y\in X_2}\exp{ (-\beta\mathbf S_{c_2y})}}\cdot g(x).
    \end{aligned}
\end{equation}
Given $q \cdot \sum_{y=x, y\in X_1} \psi(-\beta\mathbf S_{c_1y}) = \sum_{y=x, y\in X_2} \psi(-\beta\mathbf S_{c_2y})$, and let $\psi(\cdot) \doteq \exp{(\cdot)}$, we have:
\begin{equation}\label{hx2-rew2}
	h(c_2, X_2) = \sum_{x\in \mathsf{M}_1}\frac{q\cdot \exp{ (m_{c_2x})}\sum_{ y=x, y\in X_1}d_{c_1y}}{\sum_{x\in \mathsf{M}_1}q\cdot\exp{ (m_{c_2x})}\sum_{y=x, y\in X_1}d_{c_1y}}\cdot g(x).
\end{equation}
$h(c_1, X_1)=h(c_2, X_2)$ can be derived if $c_1=c_2$.
%Given $c_1=c_2$, we derive $h(c_1, X_1)=h(c_2, X_2)$.

We next prove the necessity of the iff conditions that are stated in Theorem \ref{theorem-c}.
And it can be achieved by showing possible contradictions when the iff conditions are not satisfied.
If $h(c_1, X_1)=h(c_2, X_2)$, we have:
\begin{equation}\label{h1minush2-c}
	\begin{aligned}
		&h(c_1, X_1)-h(c_2, X_2) =\\
		&\sum_{x\in X_1}[\frac{f_{c_1x}\cdot d_{c_1x}}{ \sum_{x\in X_1} f_{c_1x}\cdot d_{c_1x}}]\cdot g(x)\\
		&- \sum_{x\in X_2}[\frac{f_{c_2x}\cdot d_{c_2x}}{ \sum_{x\in X_2} f_{c_2x}\cdot d_{c_2x}}]\cdot g(x)=0.
	\end{aligned}
\end{equation}
Firstly assuming $\mathsf{M}_1 \ne \mathsf{M}_2$, we thus have:
%We firstly assume $\mathsf{M}_1 \ne \mathsf{M}_2$, then:
\begin{equation}\label{s1nes2-c}
	\begin{aligned}
		&h(c_1, X_1)-h(c_2, X_2) = \\
		&\sum_{x\in \mathsf{M}_1 \cap \mathsf{M}_2}\![\frac{\exp{ (m_{c_1x})}\sum_{y=x, y\in X_1}\!d_{c_1y}}{\sum_{x\in \mathsf{M}_1}\!\exp{ (m_{c_1x})}\sum_{y=x, y\in X_1}\!d_{c_1y}} \\ 
		&- \frac{\exp{ (m_{c_2x})}\sum_{y=x, y\in X_2}\!d_{c_2y}}{\sum_{x\in \mathsf{M}_2}\!\exp{ (m_{c_2x})}\sum_{y=x, y\in X_2}\!d_{c_2y}}]\cdot g(x)\\ 
		&+\sum_{x\in \mathsf{M}_1 \setminus \mathsf{M}_2}\![\frac{\exp{ (m_{c_1x})}\sum_{y=x, y\in X_1}\!d_{c_1y}}{\sum_{x\in \mathsf{M}_1}\!\exp{ (m_{c_1x})}\sum_{y=x, y\in X_1}\!d_{c_1y}}]\cdot g(x)\\
		&-\sum_{x\in \mathsf{M}_2 \setminus \mathsf{M}_1}[\frac{\exp{ (m_{c_2x})}\sum_{y=x, y\in X_2}\!d_{c_2y}}{\sum_{x\in \mathsf{M}_2}\!\exp{ (m_{c_2x})}\sum_{y=x, y\in X_2}\!d_{c_2y}}]\cdot g(x)=0.
	\end{aligned}
\end{equation}
As Eq. (\ref{s1nes2-c}) holds for any possible $g(\cdot)$, we define a new function $g^\prime(\cdot)$:
\begin{equation}\label{gprime-c}
	\begin{aligned}
		&g(x) = g^\prime(x), \text {for } x\in \mathsf{M}_1\cap \mathsf{M}_2\\
		&g(x) = g^\prime(x) - 1, \text {for }x \in \mathsf{M}_1\setminus \mathsf{M}_2\\
		&g(x) = g^\prime(x) + 1, \text {for } x\in \mathsf{M}_2\setminus \mathsf{M}_1\\
	\end{aligned}
\end{equation}
It is known that Eq. (\ref{s1nes2-c}) holds for both $g(\cdot)$ and $g^\prime(\cdot)$.
Thus, we have:
\begin{equation}\label{s1nes3-c}
	\begin{aligned}
		&h(c_1, X_1)-h(c_2, X_2) = 0 =\\
		&\sum_{x\in \mathsf{M}_1 \cap \mathsf{M}_2}\![\frac{\exp{ (m_{c_1x})}\sum_{y=x, y\in X_1}\!d_{c_1y}}{\sum_{x\in \mathsf{M}_1}\!\exp{ (m_{c_1x})}\sum_{y=x, y\in X_1}\!d_{c_1y}}\\ 
		&- \frac{\exp{ (m_{c_2x})}\sum_{y=x, y\in X_2}\!d_{c_2y}}{\sum_{x\in \mathsf{M}_2}\!\exp{ (m_{c_2x})}\sum_{y=x, y\in X_2}\!d_{c_2y}}]\cdot g^\prime(x)\\
		&+\sum_{x\in \mathsf{M}_1 \setminus \mathsf{M}_2}\![\frac{\exp{ (m_{c_1x})}\sum_{y=x, y\in X_1}\!d_{c_1y}}{\sum_{x\in \mathsf{M}_1}\!\exp{ (m_{c_1x})}\sum_{y=x, y\in X_1}\!d_{c_1y}}]\cdot g^\prime(x)\\ 
		&-\sum_{x\in \mathsf{M}_2 \setminus \mathsf{M}_1}[\frac{\exp{ (m_{c_2x})}\sum_{y=x, y\in X_2}\!d_{c_2y}}{\sum_{x\in \mathsf{M}_2}\!\exp{ (m_{c_2x})}\sum_{y=x, y\in X_2}\!d_{c_2y}}]\cdot g^\prime(x)\\
	\end{aligned}
\end{equation}
As Eqs. (\ref{s1nes2-c}) and (\ref{s1nes3-c}) are equal to zero, we have:
\begin{equation}\label{eq-s1nes2}
	\begin{aligned}
		&\sum_{x\in \mathsf{M}_1 \setminus \mathsf{M}_2}\![\frac{\exp{ (m_{c_1x})}\sum_{y=x, y\in X_1}\!d_{c_1y}}{\sum_{x\in \mathsf{M}_1}\!\exp{ (m_{c_1x})}\sum_{y=x, y\in X_1}\!d_{c_1y}}]\\ & + \sum_{x\in \mathsf{M}_2 \setminus \mathsf{M}_1}[\frac{\exp{ (m_{c_2x})}\sum_{y=x, y\in X_2}\!d_{c_2y}}{\sum_{x\in \mathsf{M}_2}\!\exp{ (m_{c_2x})}\sum_{y=x, y\in X_2}\!d_{c_2y}}]=0.
	\end{aligned}
\end{equation}
It is seen that Eq. (\ref{eq-s1nes2}) does not hold as $\exp(\cdot)$ is positive.
$\mathsf{M}_1 \ne \mathsf{M}_2$ is therefore not true.
Now assuming $\mathsf{M}_1 = \mathsf{M}_2 = \mathsf{M}$ and excluding the irrational terms, Eq. (\ref{s1nes2-c}) can be rewritten as follows: 
\begin{equation}
	\begin{aligned}
		&\sum_{x\in \mathsf{M}_1 \cap \mathsf{M}_2}\![\frac{\exp{ (m_{c_1x})}\sum_{y=x, y\in X_1}\!d_{c_1y}}{\sum_{x\in \mathsf{M}_1}\!\exp{ (m_{c_1x})}\sum_{y=x, y\in X_1}\!d_{c_1y}}\\
		& - \frac{\exp{ (m_{c_2x})}\sum_{y=x, y\in X_2}\!d_{c_2y}}{\sum_{x\in \mathsf{M}_2}\!\exp{ (m_{c_2x})}\sum_{y=x, y\in X_2}\!d_{c_2y}}]\cdot g(x)=0.
	\end{aligned}
\end{equation}
To ensure the equation above to hold, each term in it has to be zero:
\begin{equation}\label{s1es2-c}
	\begin{aligned}
		&\frac{\exp{ (m_{c_1x})}\sum_{y=x, y\in X_1}d_{c_1y}}{\sum_{x\in \mathsf{M}_1}\exp{ (m_{c_1x})}\sum_{y=x, y\in X_1}d_{c_1y}}\\
		& - \frac{\exp{ (m_{c_2x})}\sum_{y=x, y\in X_2}d_{c_2y}}{\sum_{x\in \mathsf{M}_2}\exp{ (m_{c_2x})}\sum_{y=x, y\in X_2}d_{c_2y}}=0.
	\end{aligned}
\end{equation}
The above equation can be further rewritten as:
%The above equation can be rewritten as:
\begin{equation}\label{s1es22-c}
	\begin{aligned}
		&\frac{\sum_{y=x, y\in X_1}\!d_{c_1y}}{\sum_{y=x, y\in X_2}\!d_{c_2y}}=\\
		&\frac{\exp{ (m_{c_2x})}\sum_{x\in \mathsf{M}_1}\!\exp{ (m_{c_1x})}\sum_{y=x, y\in X_1}\!d_{c_1y}}{\exp{ (m_{c_1x})}\sum_{x\in \mathsf{M}_2}\!\exp{ (m_{c_2x})}\!\sum_{y=x, y\in X_2}\!d_{c_2y}}.
	\end{aligned}
\end{equation}
Assuming $\mathsf{M} = \lbrace s, s_0\rbrace$, $c_1 = s_0$, $c_2 = s$, the feature correlations between the central node and its neighbors are $m_{c_1x} = 1$ for $x \in \mathsf{M}$, $m_{c_2s} = 1$, and $m_{c_2s_0} = 2$.
When $x = s$, we have:
\begin{equation}
	\begin{aligned}
		&\frac{\sum_{s\in X_1}d_{c_1s}}{\sum_{s\in X_2}d_{c_2s}}=\frac{e [e\sum_{s\in X_1}d_{c_1s} + e\sum_{s_0\in X_1}d_{c_1s_0}]}{e[e\sum_{s\in X_2}d_{c_2s} + e^2\sum_{s_0\in X_2}d_{c_2s_0}]}.
	\end{aligned}
\end{equation}
$d_{cx} = \exp{ (-\beta\mathbf S_{cx})}$ can be any positive value as the computation of feature correlation and the learning of node-node dissimilarity ($\mathbf S$) are mutually independent.
Considering $d_{cx} = \exp{ (-\beta\mathbf S_{cx})}=a > 0$,
we have $\frac{\mu_1(s)}{\mu_2(s)}=\frac{|X_1|}{|X_2|-n+ne}$.
This equation does not hold because LHS (a rational number) does not equal RHS (an irrational number).  
Thus, $c_1 \ne c_2$ is not true.
Let $c_1 = c_2 = c$, Eq. (\ref{s1es22-c}) can be rewritten as follows:
\begin{equation}\label{c1ec2-c}
	\begin{aligned}
		&\frac{\sum_{y=x, y\in X_1}\!d_{cy}}{\sum_{y=x, y\in X_2}\!d_{cy}}\!=\!\frac{\sum_{x\in \mathsf{M}_1}\!\exp{ (m_{cx})}\!\sum_{y=x, y\in X_1}\!d_{cy}}{\sum_{x\in \mathsf{M}_2}\!\exp{ (m_{cx})}\!\sum_{y=x, y\in X_2}\!d_{cy}}\!=\!const \!>\! 0.
	\end{aligned}
\end{equation}
Setting $\frac{1}{q} = const$ and $d_{cy} = \exp{ (-\beta\mathbf S_{cy})} = \psi(-\beta\mathbf S_{cy})$, we finally have $q\sum_{y=x, y\in X_1}\!\psi(-\beta\mathbf S_{cy})=\sum_{y=x, y\in X_2}\!\psi(-\beta\mathbf S_{cy})$.

\end{proof}

\subsection*{Proof of Theorem \ref{theorem-s}}
\begin{proof}
Theorem \ref{theorem-s} can also be proved by considering the sufficiency and necessity of the iff conditions stated.
The sufficiency of the iff conditions in Theorem \ref{theorem-s} is firstly proved.
Given a central node $c_i$, its aggregation function $h(c_i, X_i)$ can be written as:
%Given the definition of $h(c, X)$, we have:
\begin{equation}\label{hx-s}
    \begin{aligned}
    &h(c_i, X_i) = \sum_{x\in X_i} \alpha_{c_ix}g(x),\alpha_{c_ix} = \frac{f_{c_ix}\cdot d_{c_ix}}{ \sum_{x\in X_i} f_{c_ix}\cdot d_{c_ix}},\\
	&f_{c_ix} = \frac{\exp{ (m_{c_ix})}}{\sum_{x\in X_i}\exp{ (m_{c_ix})}}, d_{c_ix} =1-\beta \frac{\exp{(\mathbf S_{c_ix})}}{\sum_{x \in X_i}\exp{(\mathbf S_{c_ix})}},
    \end{aligned}
\end{equation}
Given $c_1 = c_2$, $\mathsf{M}_1 = \mathsf{M}_2 = \mathsf{M}$, $q[\sum_{y=x, y\in X_1}\sum_{x\in X_1}\psi(\mathbf S_{c_1x})-\sum_{y=x, y\in X_1}$ $ \beta\psi(\mathbf S_{c_1y}) ]= \sum_{y=x, y\in X_2}\sum_{x\in X_2}\psi(\mathbf S_{c_2x}) - \sum_{y=x, y\in X_2} \beta \\ \psi(\mathbf S_{c_2y})$, and $\psi(\cdot) \doteq \exp{(\cdot)}$, we have:
\begin{equation}\label{hx2-rew2-s}
\begin{aligned}
&h(c_2, X_2) = \sum_{x \in X_2}\frac{f_{c_2x}\cdot d_{c_2x}}{ \sum_{x\in X_2} f_{c_2x}\cdot d_{c_2x}} g(x)\\
&\!=\!\sum_{x \in \mathsf{M}_2}\!\frac{\exp{(m_{c_2x})}\!\sum_{y=x, y\in X_2}\![1-\beta \frac{\exp{(\mathbf S_{c_2y})}}{\sum_{x \in X_2}\!\exp{(\mathbf S_{c_2x})}}]}{\sum_{x \in \mathsf{M}_2}\exp{(m_{c_2x})}\!\sum_{y=x, y\in X_2}\![1\!-\!\beta \frac{\exp{(\mathbf S_{c_2y})}}{\sum_{x \in X_2}\exp{(\mathbf S_{c_2x})}}]}\!g(x)\\
&\!=\!\sum_{x\in \mathsf{M}_1}\frac{q\cdot \exp{ (m_{c_1x})}\sum_{ y=x, y\in X_1}d_{c_1y}}{\sum_{x\in \mathsf{M}_1}q\cdot\exp{ (m_{c_1x})}\sum_{y=x, y\in X_1}d_{c_1y}}g(x) \!= \!h(c_1, X_1).
\end{aligned}
\end{equation}
Next, we prove the necessity of the iff conditions stated in Theorem \ref{theorem-s}.
Given $h(c_1, X_1) = h(c_2, X_2)$, we have:
\begin{equation}\label{h1minush2-s}
	\begin{aligned}
		&h(c_1, X_1)-h(c_2, X_2) \!=\! \sum_{x\in X_1}[\frac{f_{c_1x}\cdot d_{c_1x}}{ \sum_{x\in X_1} f_{c_1x}\cdot d_{c_1x}}]\cdot g(x)\\ 
		& - \sum_{x\in X_2}[\frac{f_{c_2x}\cdot d_{c_2x}}{ \sum_{x\in X_2} f_{c_2x}\cdot d_{c_2x}}]\cdot g(x)=0.
	\end{aligned}
\end{equation}
Assuming $\mathsf{M}_1 \ne \mathsf{M}_2$, we have:
\begin{equation}\label{s1nes2-s}
	\begin{aligned}
		&h(c_1, X_1)-h(c_2, X_2) = \\
		&\sum_{x\in \mathsf{M}_1 \cap \mathsf{M}_2}\![\frac{\exp{ (m_{c_1x})}\sum_{y=x, y\in X_1}\!d_{c_1y}}{\sum_{x\in \mathsf{M}_1}\!\exp{ (m_{c_1x})}\sum_{y=x, y\in X_1}\!d_{c_1y}}\\ 
		&- \frac{\exp{ (m_{c_2x})}\sum_{y=x, y\in X_2}\!d_{c_2y}}{\sum_{x\in \mathsf{M}_2}\!\exp{ (m_{c_2x})}\sum_{y=x, y\in X_2}\!d_{c_2y}}]\cdot g(x)\\
		&+\sum_{x\in \mathsf{M}_1 \setminus \mathsf{M}_2}\![\frac{\exp{ (m_{c_1x})}\sum_{y=x, y\in X_1}\!d_{c_1y}}{\sum_{x\in \mathsf{M}_1}\!\exp{ (m_{c_1x})}\sum_{y=x, y\in X_1}\!d_{c_1y}}]\cdot g(x)\\ 
		&-\sum_{x\in \mathsf{M}_2 \setminus \mathsf{M}_1}[\frac{\exp{ (m_{c_2x})}\sum_{y=x, y\in X_2}\!d_{c_2y}}{\sum_{x\in \mathsf{M}_2}\!\exp{ (m_{c_2x})}\sum_{y=x, y\in X_2}\!d_{c_2y}}]\cdot g(x) = 0.
	\end{aligned}
\end{equation}
Again we may define $g^\prime(\cdot)$ as Eq. (\ref{gprime-c}) shows.
It is known that $h(c_1, X_1)-h(c_2, X_2)=0$ holds for both $g(\cdot)$ and $g^\prime(\cdot)$.
Thus, we have:
\begin{equation}\label{s1nes3-s}
	\begin{aligned}
		&h(c_1, X_1)-h(c_2, X_2) = \\
		&\sum_{x\in \mathsf{M}_1 \cap \mathsf{M}_2}\![\frac{\exp{ (m_{c_1x})}\sum_{y=x, y\in X_1}\!d_{c_1y}}{\sum_{x\in \mathsf{M}_1}\!\exp{ (m_{c_1x})}\sum_{y=x, y\in X_1}\!d_{c_1y}}\\ 
		&- \frac{\exp{ (m_{c_2x})}\sum_{y=x, y\in X_2}\!d_{c_2y}}{\sum_{x\in \mathsf{M}_2}\!\exp{ (m_{c_2x})}\sum_{y=x, y\in X_2}\!d_{c_2y}}]\cdot g^\prime(x)\\
		&+\sum_{x\in \mathsf{M}_1 \setminus \mathsf{M}_2}\![\frac{\exp{ (m_{c_1x})}\sum_{y=x, y\in X_1}\!d_{c_1y}}{\sum_{x\in \mathsf{M}_1}\!\exp{ (m_{c_1x})}\sum_{y=x, y\in X_1}\!d_{c_1y}}]\cdot g^\prime(x)\\ 
		&-\sum_{x\in \mathsf{M}_2 \setminus \mathsf{M}_1}[\frac{\exp{ (m_{c_2x})}\sum_{y=x, y\in X_2}\!d_{c_2y}}{\sum_{x\in \mathsf{M}_2}\!\exp{ (m_{c_2x})}\sum_{y=x, y\in X_2}\!d_{c_2y}}]\cdot g^\prime(x) = 0.\\
	\end{aligned}
\end{equation}
As both Eqs. (\ref{s1nes2-s}) and (\ref{s1nes3-s}) equal zero, we have:
\begin{equation}
	\begin{aligned}
		&\sum_{x\in \mathsf{M}_1 \setminus \mathsf{M}_2}\![\frac{\exp{ (m_{c_1x})}\sum_{y=x, y\in X_1}\!d_{c_1y}}{\sum_{x\in \mathsf{M}_1}\!\exp{ (m_{c_1x})}\sum_{y=x, y\in X_1}\!d_{c_1y}}]\\
		& + \sum_{x\in \mathsf{M}_2 \setminus \mathsf{M}_1}[\frac{\exp{ (m_{c_2x})}\sum_{y=x, y\in X_2}\!d_{c_2y}}{\sum_{x\in \mathsf{M}_2}\!\exp{ (m_{c_2x})}\sum_{y=x, y\in X_2}\!d_{c_2y}}]=0.
	\end{aligned}
\end{equation}
Like the analysis on proving Theorem \ref{theorem-c}, $\mathsf{M}_1 \ne \mathsf{M}_2$ is false.
%Like the analysis in the main manuscript, $\mathsf{M}_1 \ne \mathsf{M}_2$ is not true.
Now we are able to assume $\mathsf{M}_1 = \mathsf{M}_2 = \mathsf{M}$.
The terms about $x \in \mathsf{M}_1 \setminus \mathsf{M}_2$ and $x \in \mathsf{M}_2 \setminus \mathsf{M}_1$ in Eq. (\ref{s1nes2-s}) can therefore be eliminated and it is known that the following equation must hold to ensure $h(c_1, X_1)-h(c_2, X_2)=0$:
\begin{equation}\label{s1es22-s}
	\begin{aligned}
		&\frac{\sum_{y=x, y\in X_1}\!d_{c_1y}}{\sum_{y=x, y\in X_2}\!d_{c_2y}}=\\
		&\frac{\exp{ (m_{c_2x})}\sum_{x\in \mathsf{M}_1}\!\exp{ (m_{c_1x})}\sum_{y=x, y\in X_1}\!d_{c_1y}}{\exp{ (m_{c_1x})}\sum_{x\in \mathsf{M}_2}\!\exp{ (m_{c_2x})}\!\sum_{y=x, y\in X_2}\!d_{c_2y}}.
	\end{aligned}
\end{equation}
Assuming $\mathsf{M} = \lbrace s, s_0\rbrace$, $c_1 = s_0$, $c_2 = s$, the feature correlations between the central node and its neighbors are $m_{c_1x} = 1$ for $x \in \mathsf{M}$, $m_{c_2s} = 1$, and $m_{c_2s_0} = 2$.
%Let $S = \lbrace s, s_0\rbrace$, $c_1 = s_0$, $c_2 = s$, the feature similarity between central node and others be $m_{c_1x} = 1$ for $x \in S$, $m_{c_2s} = 1$, and $m_{c_2s_0} = 2$. 
When $x = s$, we have:
%Considering $x = s$, we have:
\begin{equation}
	\begin{aligned}
		&\frac{\sum_{s\in X_1}d_{c_1s}}{\sum_{s\in X_2}d_{c_2s}}=\frac{e [e\sum_{s\in X_1}d_{c_1s} + e\sum_{s_0\in X_1}d_{c_1s_0}]}{e[e\sum_{s\in X_2}d_{c_2s} + e^2\sum_{s_0\in X_2}d_{c_2s_0}]}.
	\end{aligned}
\end{equation}
$d_{cx}=1-\beta \frac{\exp{(\mathbf S_{cx})}}{\sum_{x \in X}\exp{(\mathbf S_{cx})}}$ can be any positive value as the computation of feature correlation and the learning of $\mathbf S$ are mutually independent.
Considering $d_{cx} = a > 0$,
We have $\frac{\mu_1(s)}{\mu_2(s)}=\frac{|X_1|}{|X_2|-n+ne}$. 
This equation does not hold because LHS (a rational number) does not equal RHS (an irrational number).
Thus, $c_1 \ne c_2$ is false.
Since $c_1 = c_2 = c$, Eq. (\ref{s1es22-s}) can be rewritten as:
\begin{equation}\label{c1ec2-c}
	\begin{aligned}
		&\frac{\sum_{y=x, y\in X_1}\![\sum_{x \in X_1}\exp{(\mathbf S_{cx})} - \beta \exp{(\mathbf S_{cy})}]}{\sum_{y=x, y\in X_2}[\sum_{x \in X_2}\exp{(\mathbf S_{cx})} - \beta \exp{(\mathbf S_{cy})}]} =\\
		&\!\frac{\sum_{x\in \mathsf{M}_1}\!\exp{ (m_{cx})}\!\sum_{y=x, y\in X_1}[\sum_{x \in X_1}\exp{(\mathbf S_{cx})} - \beta \exp{(\mathbf S_{cy})}]}{\sum_{x\in \mathsf{M}_2}\!\exp{ (m_{cx})}\!\sum_{y=x, y\in X_2}[\sum_{x \in X_2}\exp{(\mathbf S_{cx})} - \beta \exp{(\mathbf S_{cy})}]}\\
		&=const > 0.
	\end{aligned}
\end{equation}
Letting $const = \frac{1}{q}$ and $\exp{(\cdot)} = \psi(\cdot)$,
we finally have $q\sum_{y=x, y\in X_1}$ $[\sum_{x\in X_1}\psi(\mathbf S_{cx}) - \beta\psi(\mathbf S_{cy}) ]= \sum_{y=x, y\in X_2}[\sum_{x\in X_2}\psi(\mathbf S_{cx}) - \beta\psi(\mathbf S_{cy})]$.
\end{proof}

\subsection*{Proof of Corollary \ref{coro-att}}
\begin{proof}
We may complete this proof by following the procedure presented in \cite{DBLP:conf/iclr/XuHLJ19,he2021learning}.
According to Theorem \ref{theorem-c}, we assume $X_1 = (\mathsf{M}, \mu_1)$, $X_2 = (\mathsf{M}, \mu_2)$, $c \in \mathsf{M}$, and $q\cdot\sum_{y=x, y\in X_1} \psi(-\beta\mathbf S_{c_1y}) = \sum_{y=x, y\in X_2} \psi(-\beta\mathbf S_{c_2y})$, for $q > 0$.
When $\mathcal T$ uses the attention scores solely according to Eq. (\ref{c-strategy}) to aggregate node features, we have $\sum_{x\in X_1} \alpha_{cx} g(x) = \sum_{x\in X_2} \alpha_{cx} g(x)$.
This means $\mathcal T$ fails to discriminate the structures satisfying the conditions stated in Theorem \ref{theorem-c}.
When $\mathcal T$ uses Eq. (\ref{att-aggregation}) where the attention coefficients are obtained by the \textit{Contractive apprehension span} (Eq. (\ref{c-strategy})) to aggregate node features, we have $\sum_{x\in X_1} \alpha_{cx} g(x) - \sum_{x\in X_2} \alpha_{cx} g(x) = \epsilon (\frac{1}{|X_1|}-\frac{1}{|X_2|})\alpha_{cc} g(c)$, where $|X_1| = |\mathcal N_1|$, and $|X_2| = |\mathcal N_2|$.
Since $|X_1| \ne |X_2|$, $\sum_{x\in X_1} \alpha_{cx} g(x) - \sum_{x\in X_2} \alpha_{cx} g(x) \ne 0$, which means $\mathcal T$ based on Eqs. (\ref{c-strategy}) and (\ref{att-aggregation}) is able to discriminate all the structures that $\mathcal T$ solely based on Eq. (\ref{c-strategy}) fails to distinguish.
Following the similar procedure, when the \mechanism~layer (Eq. (\ref{att-aggregation})) utilizes the \textit{Subtractive apprehension span} (Eq. (\ref{s-strategy})), we are able to prove that the corresponding aggregation function also can distinguish those distinct structures that the aggregation function only using \textit{Subractive apprehension span} fails to discriminate.
\end{proof}

\end{document}